\documentclass[a4paper,11pt]{article}
\usepackage{graphicx} 
\usepackage[a4paper,left=2.5cm,right=2.5cm,top=3cm,bottom=3cm]{geometry}
\usepackage{natbib}
\usepackage{amsfonts}
\usepackage{amsmath}
\usepackage{amsthm}
\usepackage{import}
\usepackage{enumerate}
\usepackage[dvipsnames]{xcolor}
\usepackage{authblk, bm, bbm, soul}
\usepackage{amssymb}
\usepackage{enumitem}
\usepackage{multirow}
\usepackage{capt-of}
\usepackage{booktabs}
\usepackage{hyperref}
\usepackage{cleveref}
\usepackage{natbib}
\hypersetup{colorlinks=true, citecolor=blue, linkcolor = blue, urlcolor = blue}
\usepackage{stackrel}
\usepackage{algorithm}
\usepackage{multirow}
\usepackage{dsfont}
\usepackage{bbm}
\usepackage{algpseudocode}
\usepackage[toc,page]{appendix}

\newtheorem{theorem}{Theorem}[]
\newtheorem{definition}{Definition}[]

\newtheorem{assumption}{Assumption}
\newtheorem{proposition}[theorem]{Proposition}
\newtheorem{corollary}[theorem]{Corollary}
\newtheorem{lemma}[theorem]{Lemma}
\theoremstyle{remark}
\newtheorem{remark}{Remark}

\DeclareMathOperator*{\var}{Var}
\DeclareMathOperator*{\cov}{Cov}

\DeclareMathOperator*{\opt}{opt}
\DeclareMathOperator*{\D}{\mathrm{D}}

\DeclareMathOperator*{\DD}{\mathrm{DD}}

\DeclareMathOperator*{\argmin}{arg\,min}

\DeclareMathOperator*{\arginf}{arg\,inf}
\DeclareMathOperator*{\argmsup}{arg\,sup}
\global\long\def\indc{\mathds{1}}

\allowdisplaybreaks

\title{Federated fairness-aware classification under differential privacy}
\author{Gengyu Xue}
\author{Yi Yu}

\affil{Department of Statistics, University of Warwick}
\date{\today}

\begin{document}
\maketitle
\begin{abstract}
    Privacy and algorithmic fairness have become two central issues in modern machine learning. Although each has separately emerged as a rapidly growing research area, their joint effect remains comparatively under-explored. In this paper, we systematically study the joint impact of differential privacy and fairness on classification in a federated setting, where data are distributed across multiple servers.  Targeting demographic disparity constrained classification under federated differential privacy, we propose a two-step algorithm, namely FDP-Fair. In the special case where there is only one server, we further propose a simple yet powerful algorithm, namely CDP-Fair, serving as a computationally-lightweight alternative. Under mild structural assumptions, theoretical guarantees on privacy, fairness and excess risk control are established. In particular, we disentangle the source of the private fairness-aware excess risk into a) intrinsic cost of classification, b) cost of private classification, c) non-private cost of fairness and d) private cost of fairness. Our theoretical findings are complemented by extensive numerical experiments on both synthetic and real datasets, highlighting the practicality of our designed algorithms.
\end{abstract}

\section{Introduction} \label{sec_intro}

Recent advances in modern technology have made it increasingly easier to collect, store, and analyse large volumes of data. While these developments bring substantial benefits, the growing availability of data has also raised significant public concerns. Among these, privacy concerns due to use of sensitive data and ethical challenges caused by algorithmic unfairness have emerged as two central issues, attracting increasing attention from both researchers and the public.

To address privacy concerns, differential privacy (DP; \citealp{dwork2006calibrating}) has emerged as one of the most principled approaches with wide applications in Google \citep[e.g.][]{song2021evading}, Meta \citep[e.g.][]{yousefpour2021opacus} and the US Census Bureau \citep[e.g.][]{uscensus}, to name but a few. From a statistical perspective, substantial analyses have also been carried out in various settings to quantify the cost of privacy constraints on statistical accuracy across a range of problems, such as mean estimation \citep[e.g.][]{karwa2017finite}, density estimation \citep[e.g.][]{butucealocal2020}, classification \citep[e.g.][]{auddy2025minimax} and non-parametric regression \citep[e.g.][]{cai2024optimal}.

Alongside privacy, ensuring fairness in algorithmic decision-making has emerged as another key challenge. Empirical studies have shown that algorithms may inherit biases from data, leading to ethical concerns \citep[e.g.][]{angwin2022machine}. In response, leading companies such as LinkedIn \citep[e.g.][]{quinonero2023disentangling} and Meta \citep[e.g.][]{bakalar2021fairness} have proposed practical frameworks to mitigate fairness issues and protect vulnerable groups in practice. From a methodological perspective, existing approaches to algorithmic fairness are often broadly categorised into pre-processing \citep[e.g.][]{calmon2017optimized, johndrow2019algorithm}, in-processing \citep[e.g.][]{celis2019classification, cho2020fair}, and post-processing \citep[e.g.][]{zeng2024bayes, hou2024finite} strategies, depending on the stage at which fairness constraints are incorporated. We refer readers to \cite{pessach2022review} for a comprehensive review of recent developments.

Despite being well studied separately, DP and fairness have received relatively less attention when considered jointly. Prior empirical work indicates that differential privacy can deteriorate fairness in various settings \citep[e.g.][]{ bagdasaryan2019differential, ganev2022robin}. From a theoretical perspective, \citet{mangold2023differential} further study the impact of DP on fairness when fairness measures are pointwise Lipschitz continuous. By incorporating fairness directly into the objective function, \citet{zhou2024differentially} focus on private worst-group risk minimisation under the min-max fairness framework introduced by \citet{martinez2020minimax}. In the context of classification, \citet{jagielski2019differentially} consider multi-class private fair classification under label-central DP when only sensitive features are protected. \citet{ghoukasian2024differentially} design a post-processing method for private group-wise classifiers, targeting demographic disparity. More recently, \citet{say2025fairness} consider multi-class classification under DP by extending the framework in \citet{denis2024fairness} using output perturbation and noisy gradient descent. However, to the best of our knowledge, most prior works lack rigorous theoretical guarantees on both fairness and excess risk control, and are typically developed under a single-server data setting, where all data are assumed to be centrally stored and processed by a single trusted entity. Algorithms with theoretical guarantees for distributed fairness-aware classification under privacy constraints are yet to be explored.

\subsection{List of contributions} \label{sec_list_contribution}
In this paper, we systematically study the problem of classification under demographic disparity (\Cref{def_disparity_measure}) in a distributed learning setup with multiple servers, adhering to federated DP (\Cref{def_fdp}) constraints. This setting is graphically illustrated in \Cref{fig_fdp_fair} and has wide-ranging applications, for instance, analysing data from patients across different hospitals \citep[e.g.][]{li2020multi} and cross-organisation collaborations \citep[e.g.][]{heyndrickx2023melloddy}, among many others. 

We summarise the main contributions of the paper as follows.
\begin{itemize}
    \item To the best of our knowledge, this is the first work to systematically study private distributed fairness-aware classification. We propose a two-step post-processing algorithm, detailed in \Cref{alg_fair_fdp}. In the first step, due to the coexistence of scalar and functional outputs, we employ various Gaussian mechanisms, adding univariate Gaussian noise to scalars and Gaussian processes to functional outputs. In the second step, to determine the adjusted decision threshold under federated DP, we introduce a novel private tree-based algorithm that avoids repeated privacy composition arising from summation estimations in the optimisation problem. 
    
    \item In the special case when there is only one server, we further show that a new and simpler algorithm, \Cref{alg_fair_cdp}, is sufficient for fair classification under central DP, in contrast to the more complex tree-based procedure in \Cref{alg_bin_tree}. Notably, \Cref{alg_fair_cdp} and its theoretical guarantees provide a novel private mechanism for optimising a functional object without incurring functional privacy, and is of independent interest.

    \item We further establish finite-sample theoretical guarantees for the proposed algorithms in terms of privacy, fairness and excess risk control. In particular, we show that both Algorithms~\ref{alg_fair_fdp} and~\ref{alg_fair_cdp} satisfy fairness constraints with high probability while maintaining satisfactory classification performance under DP, with the cost of privacy and fairness explicitly quantified. We separate the source of excess risk into decoupled terms, corresponding to the intrinsic hardness of classification, the privacy and fairness constraints. These results provide a rigorous theoretical foundation for the decision makers to understand the trade-offs between privacy, fairness and excess risk.

    \item Systematic numerical experiments on both simulated and real datasets are carried out in \Cref{sec_numerical}. These results further support our theoretical findings in \Cref{sec_theoretical} and highlight the practicality.
    
\end{itemize}

\noindent\textbf{Notation.} For a positive integer $a$, denote $[a] = \{1, \ldots, a\}$.  Let $\lceil a \rceil$ be the smallest integer greater than or equal to $a$ and $\lfloor a \rfloor$ be the greatest integer less than or equal to $a$. For $a,b \in \mathbb{R}$, let $a \vee b = \max\{a,b\}$ and $a\wedge b = \min\{a,b\}$. For $v \in \mathbb{R}^p$, let $\|v\|_1$, $\|v\|_2$ and $\|v\|_{\infty}$ be $\ell_1$-, $\ell_2$- and $\ell_{\infty}$-norms. For a sequence of positive numbers $\{a_n\}$ and a sequence of random variables $\{X_n\}$, denote $X_n = O_{\mathrm{p}}(a_n)$ if $\lim_{M \rightarrow \infty}\lim\sup_n \mathbb{P}(|X_n|\geq M a_n)=0$. We write $X_n = \widetilde{O}_{\mathrm{p}}(a_n)$ if there exists $k>0$ such that $X_n = O_{\mathrm{p}}(a_n (\log n)^k)$. For two sequences of positive numbers $\{a_n\}$ and $\{b_n\}$, denote $a_n \lesssim b_n$, $a_n \gtrsim b_n$, and $a_n \asymp b_n$, if there exists some constants $c,C > 0$ such that $a_n/b_n \leq C$, $b_n/a_n \leq C$ and $c \leq a_n/b_n \leq C$. Write $a_n\lesssim_{\log} b_n$, $a_n \asymp_{\log} b_n$ and $a_n =_{\log} b_n$, if $a_n\lesssim b_n$, $a_n \asymp b_n$ and $a_n = b_n$ up to poly-logarithmic factors.

For $d \in \mathbb{N}_+$ and any vector $s = (s_1, \ldots, s_d)^{\top} \in \mathbb{N}^d$, define $\vert s\vert = \sum_{i = 1}^d s_i$, $s!=s_1!\cdots s_d!$ and the associated partial differential operator $D^s= \frac{\partial^{\vert s\vert}}{\partial x_1^{s_1}\cdots  \partial x_{d}^{s_d}}$. For any function $f:\, \mathbb{R}^d \to \mathbb R $ that is 
$ \lfloor \beta \rfloor $-times continuously differentiable at point $x_0$, denote $f_{x_0}^\beta $ the Taylor polynomial of degree $ \lfloor\beta\rfloor $ at $x_0$, defined as $f_{x_0}^\beta(x) = \sum_{|s| \leq \lfloor\beta\rfloor } \frac{(x-x_0)^s}{s!} D^s f(x_0)$. For a constant $L>0$, let $\mathcal{H}^{\beta}(L)$ be the class of H\"{o}lder smooth functions $f:\, \mathbb{R}^p \to \mathbb R $ such that $f$ is $\lfloor\beta\rfloor$-times differentiable for all  $  x  \in \mathbb{R}^d$  and satisfies $ |f(x) - f_{x_0}^\beta(x)  | \le L| x-x_0|^ \beta$, for all $x, x_0\in \mathbb{R}^d$. 

\section{Problem formulation} \label{sec_formulation}
In this section, we formally present the model setup and federated DP in \Cref{sec_fdp}, then introduce fairness constraints and present the framework of distributed fairness-aware classification in \Cref{sec_fair_fdp_classification}.

\subsection{Setup and federated differential privacy} \label{sec_fdp}

Suppose we have $N_{\text{total}}$ independent and identically distributed samples $\mathcal{D} = \{(X_i, A_i, Y_i), i \in [N_{\text{total}}]\}$, where $X_i \in [0,1]^d$ is the standard $d$-dimensional feature, $A_i \in \{0,1\}$ is the sensitive feature (e.g.~race and gender) and $Y_i \in \{0,1\}$ is the binary label. We further assume that these data are distributed across $S$ servers, with server $s \in [S]$ holding a dataset $\mathcal{D}_s = \{(X^s_i,A^s_i,Y^s_i)\}_{i=1}^{N_s}$ of size $N_s$, where $\sum_{s=1}^S N_s = N_{\text{total}}$.

\begin{figure}[!htbp]
    \centering
    \includegraphics[width=0.75\linewidth]{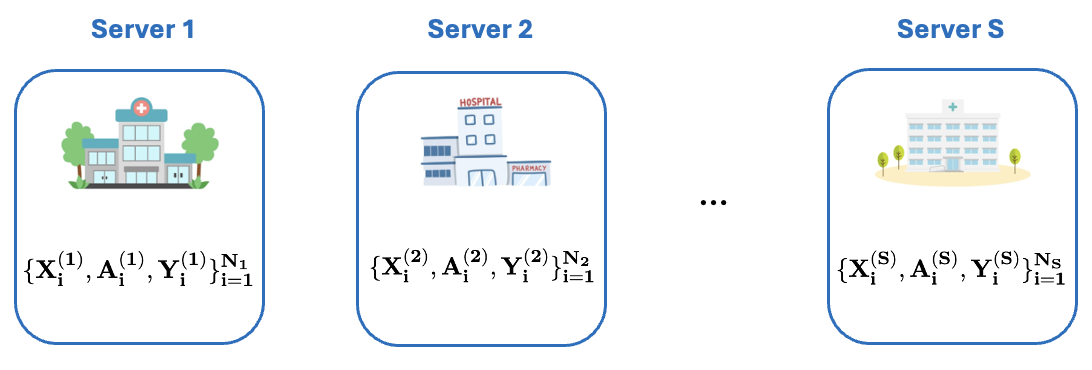}
    \caption{An illustration of the framework we consider in the problem of fairness-aware classification in a distributed setting.}
    \label{fig_fdp_fair}
\end{figure}

With the setup above, we now introduce the federated DP constraints. At a high level, it is a variant of DP tailored to distributed settings, ensuring only private information is shared across servers.  To formally define federated DP, we begin with the definition of central DP \citep{dwork2006calibrating}, which is arguably the most fundamental notion of DP and is imposed within each server under our setting. Formally speaking, given a dataset $D$, a privacy mechanism $Q(\cdot|D)$ is a conditional distribution of the private information given data $D$. Let $Z \in \mathcal{Z}$ denote the privatised data and $\sigma(\mathcal{Z})$ denote the sigma-algebra on $\mathcal{Z}$. We require the privacy mechanism to satisfy the~following. 
\begin{definition}[Central differential privacy, CDP] \label{def_cdp}
    For $\epsilon > 0$ and $\delta \geq 0$, the privacy mechanism~$Q$ is said to satisfy $(\epsilon,\delta)$-CDP if
    \begin{align*}
        Q(Z \in A|D) \leq e^{\epsilon}Q(Z \in A|D')+ \delta,
    \end{align*}
    for all $A \in \sigma(\mathcal{Z})$ and $D$ and $D'$ that differ by at most one data entry, denoted by $D\sim D'$.
\end{definition}

Motivated by federated learning, Federated DP is introduced and has been applied in various settings under various names, e.g.~\citet{li2024federated}, \citet{cai2024optimal}, \citet{xue2024optimal}, \citet{auddy2025minimax}, and many more. In this paper, we focus on the non-interactive setting, in which no communication or interaction between servers is allowed.  Let $Z^s \in \mathcal{Z}$ denote the privatised information output by server $s$ and $\sigma(\mathcal{Z})$ denote the sigma-algebra on $\mathcal{Z}$. 
\begin{definition} [Federated differential privacy, FDP] 
\label{def_fdp}
     For $S \in \mathbb{Z}_+$, denote $(\bm{\epsilon}, \bm{\delta}) = \{(\epsilon_s, \delta_s)\}_{s=1}^S$ the collection of privacy parameters where $\epsilon_s >0$ and $\delta_s \geq 0$, $s\in [S]$. We say that a privacy mechanism $Q$ satisfies $(\bm{\epsilon}, \bm{\delta})$-FDP, if for any $s\in [S]$, the transcript $Z^s \in \mathcal{Z}$ shared to the central server satisfies an $(\epsilon_s,\delta_s)$-CDP constraint, i.e.
     \begin{align*}
         Q(Z^s \in A|\mathcal{D}_s) \leq e^{\epsilon_s}Q(Z^s \in A|\mathcal{D}'_s)+ \delta_s,
     \end{align*}
     for all $A \in \sigma(\mathcal{Z})$ and $\mathcal{D}_s$ and $\mathcal{D}_s'$ that differ by at most one data entry.
\end{definition}

In both Definitions \ref{def_cdp} and \ref{def_fdp}, the strength of the privacy is controlled by $(\epsilon,\delta)$. The parameter $\epsilon$ controls the strength of the privacy constraint, with smaller values of $\epsilon$ corresponding to stronger privacy. The parameter $\delta \geq 0 $ controls the level of privacy leakage.

\subsection{Fairness-aware classification} \label{sec_fair_fdp_classification}

Having introduced the FDP constraints, we next present the fairness constraints, i.e.~demographic disparity (\Cref{def_disparity_measure}), and formulate the fairness-aware classification problem in this subsection.  

\begin{definition}[Randomised classifier]\label{def_random_classifier}
    For any $x \in [0,1]^d$ and $a \in \{0,1\}$, a randomised classifier $f\in \mathcal F$ is a measurable function such that $f(x, a) = \mathbb{P}(\widehat Y_f =1 | X=x, A=a)$, where $\widehat{Y}_f = \widehat{Y}_f(x,a)$ is the predicted label, i.e.~$\widehat Y_f |\{X=x, A=a\} \sim \mathrm{Bernoulli}\;(f(x,a))$.
\end{definition}

\begin{definition}[Demographic disparity] \label{def_disparity_measure}
    For a given classifier $f \in \mathcal F$, the demographic disparity ($\DD$) is defined as
	\begin{align*} 
            \DD(f) & = \mathbb{P}\{\widehat Y_f(X, 1)=1| A=1\} - \mathbb{P}\{\widehat Y_f(X,0)=1 | A=0\}.
	\end{align*}
\end{definition}

Let $\mathcal{F}$ be the class of measurable functions $f : [0,1]^d \times \{0,1\} \rightarrow [0,1]$. For a prespecified $\alpha \geq 0$, we define the $\alpha$-fair Bayes optimal classifier as the randomised classifier (\Cref{def_random_classifier}), $f^*_{\DD,\alpha}$, which minimises the misclassification error $R(f^*_{\DD,\alpha}) = \mathbb{P}(\widehat{Y}_{f^*_{\DD,\alpha}}(X,A) \neq Y)$ while satisfying the $\alpha$-demographic disparity , i.e.
\begin{align} \label{def_bayes_fair_classifier}
    f^*_{\DD,\alpha} \in \argmin_{f \in \mathcal{F}} \{R(f): |\DD(f)| \leq \alpha\}.
\end{align}

The Bayes fairness-aware classifier defined in \eqref{def_bayes_fair_classifier} has been previously studied in \citet{zeng2024bayes}, where it is shown that the optimal fairness-aware classifier can be obtained by shifting the Bayes decision rule through an adjusted threshold. Specifically,
\begin{align} \label{eq_bayes_fair}
    f^*_{\DD,\alpha}(x,a) = \begin{cases} \vspace{0.5em}
            1, &\eta_a(x) \geq \frac{1}{2}+\frac{\tau^*_{\DD,\alpha}(2a-1)}{2\pi_a},\\ 
            0, &\eta_a(x) < \frac{1}{2}+\frac{\tau^*_{\DD,\alpha}(2a-1)}{2\pi_a},
        \end{cases}
\end{align}
where
\begin{equation} \label{eq_tau_star}
    \tau_{\mathrm{DD}, \alpha}^* = \arg\min_{\tau \in \mathbb{R}} \big\{ |\tau|: \, |\mathrm{DD}(\tau)| \leq \alpha\big\}
\end{equation}
is the magnitude of the adjusted threshold, and for any $a \in \{0,1\}$ and $x \in [0,1]^d$, $\eta_a(x) = \mathbb{P}(Y=1|A=a, X=x)$, $\pi_{a,y} = \mathbb{P}(A=a, Y=y)$ and $\pi_a = \mathbb{P}(A=a)$. 

Our goal in this paper is to construct an estimator for $f^*_{\DD,\alpha}$ under FDP defined in \Cref{def_fdp}.

\begin{remark}
    In this paper, we focus on classification with demographic disparity, which controls the difference in the probability of predicting the positive label across sensitive groups. We would like to remark that, with minor modifications, our algorithm can be generalised to accommodate other bilinear disparity measures (\Cref{def_bilinear_disparity}, Definition~3.3 in \citealp{zeng2024bayes}), such as disparity of opportunity (DO) and predictive disparity (PD).
\end{remark}

\section{Federated differentially private fair classifier under demographic disparity}  \label{sec_alg}

In this section, we present our two-step plug-in FDP-Fair classifier in \Cref{alg_fair_fdp}.

Building on the Bayes optimal fairness-aware classifier in \eqref{eq_bayes_fair}, the construction of \Cref{alg_fair_fdp} relies on estimating three key quantities: regression functions $\eta_a$, class probabilities $\pi_a$, and the fairness threshold parameter $\tau^*_{\DD,\alpha}$. The algorithm is organised in two steps. \textbf{S1} produces estimators of $\eta_a$ and $\pi_a$. In \textbf{S2}, we estimate $\tau^*_{\DD,\alpha}$ through a private threshold search procedure building upon binary tree constructions in \Cref{alg_threshold_search_fdp}, which itself relies on the auxiliary procedures given in Algorithms \ref{alg_bin_tree} and \ref{alg_non_increasing}.

For notational clarity, in \Cref{alg_fair_fdp}, we decompose the data for server $s \in [S]$ as $\mathcal{D}^s = \mathcal{D}^s_{0,1} \cup \mathcal{D}^s_{0,0} \cup \mathcal{D}^s_{1,1} \cup \mathcal{D}^s_{1,0}$, where for $a, y \in \{0,1\}$, $\mathcal{D}^s_{a,y} = \{(X_i^s, A_i^s = a, Y_i^s = y)\}$. We further denote the $i$-th feature in $\mathcal{D}^s_{a,y}$ by $X^{s,i}_{a,y}$ for $i \in [|\mathcal{D}^s_{a,y}|]$. For any kernel function $K : \mathbb{R}^d \to \mathbb{R}^+$, let the scaled kernel function be $K_h(x) = h^{-d}K(x/h)$, for any $x \in \mathbb{R}^d$ and bandwidth $0 < h < 1$.

The key components of \Cref{alg_fair_fdp} are outlined below.

\smallskip

\noindent \textbf{Estimation of class probability $\pi_a$ and regression function $\eta_a$.} In \textbf{S1.}~of \Cref{alg_fair_fdp}, site-wise estimators of $\pi_a^s$ and $\eta_a^s$ are constructed from the training data using empirical and kernel density estimators, respectively. To preserve FDP, we utilise Gaussian mechanisms for both scalars used in estimating $\pi_a$'s and functions in estimating $\eta_a$'s. The resulting privatised site-wise estimators are then aggregated at the central server through weighted sums, yielding global estimators while maintaining FDP constraints.

\medskip

\noindent \textbf{Binary tree construction.} To estimate the private threshold, $\widetilde{\tau}_{\DD,\alpha}$, we approximate the original continuous optimisation problem in \eqref{eq_tau_star} using discretisation. For each discretised candidate, we further reduce it to evaluating tail counts utilising binary trees, where counts in dyadic intervals are organised hierarchically, enabling efficient privacy budget allocation across a sufficiently fine evaluation grids.

Specifically, by considering the random variables $Z_{s,a,y} = 2(2a-1)\widetilde{\pi}_a\{\widetilde{\eta}_a(X^s_{a,y})-1/2\}$, for each evaluated $\tau \in \mathbb{R}$, we can rewrite the empirical plug-in estimator of DD,
\begin{align} \notag
    \widehat{\mathrm{DD}}(\tau) = \;& \sum_{s=1}^S\mu_s  \widehat{\mathrm{DD}}_s(\tau)\\ \notag
    = \;& \sum_{s=1}^S\mu_s \Bigg[\frac{1}{\breve{n}_{s,1}}\sum_{y\in\{0,1\}}\sum_{i=1}^{\breve{n}_{s,1,y}}\indc\Big\{2\widetilde{\pi}_1\{\widetilde{\eta}_1(\breve{X}^{s,i}_{1,y})-1/2\}\geq \tau \Big\} \\  \label{eq_DD_tree}
    & \hspace{3.5em}-\frac{1}{\breve{n}_{s,0}}\sum_{y\in\{0,1\}}\sum_{i=1}^{\breve{n}_{s,0,y}}\indc\Big\{-2\widetilde{\pi}_0\{\widetilde{\eta}_0(\breve{X}^{s,i}_{0,y})-1/2\} \leq \tau\Big\}\Bigg],
\end{align}
as
\[\widehat{\DD}(\tau) = \sum_{s=1}^S \mu_s \times \Big \{\frac{\sum_y \mbox{count of } Z_{s,1,y} \mbox{ larger than } \tau }{\sum_y \mbox{Total count of }Z_{s,1,y}} -\frac{\sum_y \mbox{count of } Z_{s,0,y} \mbox{ less than } \tau }{\sum_y \mbox{Total count of }Z_{s,0,y}}\Big\}.\]

Given the fact that $|\tau^*_{\DD,\alpha}| \leq \min\{\pi_0, \pi_1\}$ in \Cref{l_range_tau}, we partition the interval~$[-1,1]$ into equal-sized subintervals and arrange them in a multi-layer binary tree. For each server $s \in [S]$ and sensitive group $a \in \{0,1\}$, the count of $Z_{s,a,y}$ falling into binary subintervals across all layers is computed and privatised using Gaussian mechanisms. The resulting site-wise binary trees are then aggregated at the central server via weighted averaging, where the weights are chosen according to the variances of the corresponding site-level estimators, similar to \textbf{S1}.
\begin{remark} \label{remark_tree}
    With the binary tree construction in \Cref{alg_bin_tree}, changing a single data point affects at most one node per tree-level, resulting in constant-order $\ell_2$-sensitivity used in the Gaussian mechanism. Moreover, since the tree depth is logarithmic by selection in \Cref{t_fdp_risk}, only a \emph{logarithmic} number of privacy compositions is required. This leads to a substantial improvement in accuracy compared to standard discretisation approaches in private nonparametric estimation, where a \emph{polynomial} number of compositions is needed. We defer more discussions on tree depth selection until after presenting theoretical results in \Cref{t_fdp_risk} in \Cref{sec_theoretical}.
\end{remark}

\medskip
\noindent  \textbf{Monotonicity correction.} The aggregation of Gaussian noise across different leaves in the binary tree means that the evaluated $\widetilde{\DD}$ across grids is no longer guaranteed to be monotonically non-increasing, which is a crucial property of the true DD function underlying the derivation of $\tau^*_{\DD,\alpha}$ in \eqref{eq_tau_star}. To preserve this structure, we introduce \Cref{alg_non_increasing}, which fits, with high probability, a non-increasing sequence of $\{\widetilde{\DD}(\tau_j)\}$ while controlling the induced bias.

\begin{algorithm}[!htbp]
    \caption{\textsc{FDP-Fair} classifier under demographic disparity. \label{alg_fair_fdp}}
    \begin{algorithmic}
        \Require Data $\mathcal{D}$, disparity level $\alpha$, Kernel function $K$, bandwidth parameters $h$, weights $\{\nu_s\}_{s\in [S]}$ and $\{\mu_s\}_{s\in [S]}$, constant $C_\rho, C_{\omega} >0$, smallest interval length $\theta$, tolerance $\eta$.

        \State For each $s\in [S]$, let $n_s = \lceil N_s/2 \rceil$ and $\breve{n}_s = N_s -n_s$. Split the data in each site into training data $\{(X^s_i,A^s_i,Y^s_i)\}_{i=1}^{n_s}$ and calibration data $\{(X^s_i, A^s_i, Y^s_i)\}_{i=n_s+1}^{N_s} = \{\breve{X}^s_i, \breve{A}^s_i, \breve{Y}^s_i\}_{i=1}^{\breve{n}_s}$. We further denote $n_{s,a,y} = |\{(X^s_i,A^s_i=a,Y^s_i=y)\}_{i=1}^{n_s}|$ and $n_{s,a} = |\{(X^s_i,A^s_i=a,Y^s_i)\}_{i=1}^{n_s}|$.

        \State \textbf{S1.} Estimating regression function $\eta_a$ and class probabilities $\pi_{a}$ using training data adhering to federated differential privacy constraints.

        \State \textbf{S1.1} For $s \in [S]$, $a \in \{0,1\}$, generate $w^s_a \stackrel{\mathrm{i.i.d.}}{\sim} N(0, \sigma_{s}^2)$, where $\sigma_{s} = 4\sqrt{2\log(5/\delta_s)}/(n_s\epsilon_s)$. Calculate
        \[\widetilde{\pi}_a = \sum_{s=1}^S \nu_s \widetilde{\pi}_a^s, \; \text{where}\;\; \widetilde{\pi}_a^s= \frac{n_{s,a,0}+n_{s,a,1}}{n_s} + w_{a}^s.\]

        \State \textbf{S1.2} Let $\{W^s_{k,a}(\cdot)\}_{k,a\in\{0,1\},s\in[S]}\}$  be independent mean zero Gaussian processes with covariance kernels $\cov(W^s_{k,a}(\ell),W^s_{k,a}(t)) = K\{(\ell-t)/h\}, \ell,t \in [0,1]^d$.
        
        For any $x\in [0,1]^d$ and $s \in [S]$, denote
        \begin{align*}
            \widetilde{p}^s_{X|A=a}(x|a) =\frac{1}{n_{s,a}}\sum_{y \in \{0,1\}}\sum_{i=1}^{n_{s,a,y}}K_{h}(X_{a,y}^{s,i} -x)+ \frac{8\sqrt{2C_K\log(8/\delta_s)}}{n_{s,a}\epsilon_s h^d}W^s_{1,a}(x), 
        \end{align*}
        and
        \begin{align*}
           \widetilde{p}^s_{X,Y|A=a}(x,y|a) = \frac{1}{n_{s,a}}\sum_{i=1}^{n_{s,a,1}} K_{h}(X_{a,1}^{s,i} -x)+ \frac{8\sqrt{2C_K\log(8/\delta_s)}}{n_{s,a}\epsilon_s h^d}W^s_{2,a}(x),
        \end{align*}
        Denote $\widetilde{p}_{X|A=a}(x|a) = \sum_{s=1}^S \nu_s\widetilde{p}^s_{X|A=a}(x|a)$ and $\widetilde{\eta}_a(x) =\frac{1}{\widetilde{p}_{X|A=a}(x|a)}\sum_{s=1}^S \nu_s \widetilde{p}^s_{X,Y|A=a}(x,1|a)$.

        \State \textbf{S2.} Estimating the optimal threshold using calibration data. 

        \State Set $\widetilde{\tau}_{\DD,\alpha} = \textsc{FDP.Threshold.Search}(\{\{\breve{X}_i, \breve{A}_i, \breve{Y}_i\}_{i=1}^{\breve{n}_s}\}_{s=1}^S, \{\mu_s\}_{s\in [S]}, \{\widetilde{\eta}_a,\widetilde{\pi}_a\}_{a\in\{0,1\}}, C_\rho, C_{\omega}, \theta, \eta)$. \Comment{See \Cref{alg_threshold_search_fdp}.}

        \Ensure $\widetilde{f}_{\DD,\alpha}(x,a)$ with 
        \begin{align*}
            \widetilde{f}_{\DD,\alpha}(x,a) = \begin{cases} \vspace{0.5em}
            1, &\widetilde{\eta}_a(x) \geq \frac{1}{2}+\frac{\widetilde{\tau}_{\DD,\alpha}(2a-1)}{2\widetilde{\pi}_a},\\ 
            0, &\widetilde{\eta}_a(x) < \frac{1}{2}+\frac{\widetilde{\tau}_{\DD,\alpha}(2a-1)}{2\widetilde{\pi_a}}.
        \end{cases}
        \end{align*}
    \end{algorithmic}
\end{algorithm}

\begin{algorithm}[!htbp]
\caption{\textsc{FDP.Threshold.Search}$(\breve{\mathcal{D}},\{\mu_s\}_{s\in [S]},  \{\widetilde{\eta}_a, \widetilde{\pi}_a\}_{a\in\{0,1\}}, C_\rho, C_{\omega}, \theta, \eta)$.} \label{alg_threshold_search_fdp}
    \begin{algorithmic}
        \Require Data $\breve{\mathcal{D}}$, weight $\{\mu_s\}_{s\in [S]}$ , estimated regression functions and class probabilities $\{\widetilde{\eta}_a, \widetilde{\pi}_a\}_{a\in\{0,1\}}$ with bandwidth $h$, constant $C_\rho, C_{\omega} >0$, minimum bin length $\theta$, tolerance $\eta$.

        \State Set the number of layers at $M := \log_2(\theta^{-1})+1$ and tolerance at $\rho^* = C_{\rho}\rho$, where $\rho$ is given in \eqref{t_fdp_risk_eq1}.

        \State \textbf{S1 Binary tree construction: } 
        \State Run \textsc{FDP.Binary.Tree}$(\breve{\mathcal{D}},  \{\widetilde{\eta}_a, \widetilde{\pi}_a\}_{a\in\{0,1\}}, M)$ detailed in \Cref{alg_bin_tree}.

        \State \textbf{S2 Estimation of disparity value in central server: }
        \State Construct evaluation grids $\mathcal{G}: = \{\tau_1, \ldots, \tau_{2^M+1}\}$, where $\tau_j =-1+ (j-1)2^{1-M} $ for $j \in [2^M+1]$. 
        
        \State With the binary tree given in \textbf{S1}, calculate 
        \[N_{s,1} = N_{s, 1, 1,1} + N_{s, 1, 1,2},\;\; N_{s,0} = N_{s, 0, 1,1} + N_{s, 0, 1,2},\]
        and for any $j \in [2^{M}]$
        \begin{align*}
            \text{Tail}_{s,a}(\tau_j) =  \sum_{\ell=1}^M \sum_{k=1}^{2^\ell}  N_{s,a,\ell,k}\indc\Big\{(k-1)2^{M-\ell}+1 \geq j \text{ and } (\lceil k/2\rceil-1)2^{M-\ell+1}+1 <j \Big\}.
        \end{align*}
        Then for any $\tau \in \mathcal{G}$, set $\widetilde{\DD}(\tau) = \sum_{s=1}^S  \mu_s \widetilde{\DD}^s(\tau)$, where
        \begin{align*}
            \widetilde{\DD}^s(\tau)= \frac{\text{Tail}_{s,1}(\tau)}{N_{s,1}} - \frac{N_{s,0}-\text{Tail}_{s,0}(\tau)}{N_{s,0}}.
        \end{align*}
        \State Set $\{\widetilde{\DD}_{\downarrow}(\tau_i)\}_{i \in [2^M+1]} = \textsc{Non.increasing}(\{\widetilde{\DD}(\tau_i)\}_{i \in [2^M+1]}, \{\mu_s\}_{s\in[S]}, C_{\omega}, \eta)$. \Comment{See \Cref{alg_non_increasing}}

        \State \textbf{S3 Threshold selection:}
        
        \If{$|\widetilde{\DD}_{\downarrow}(0)| \leq \alpha$}
        \State Set $\widetilde{\tau}_{\DD,\alpha} =0$.
        \Else 
        \State Set $\widetilde{\tau}_{\DD,\alpha} = \argmin_{\tau \in \mathcal{G}}\{|\tau|: |\widetilde{\DD}_{\downarrow}(\tau)| \in [\alpha-\rho^*, \alpha+\rho^*]\}$.
        \EndIf

        \Ensure $\widetilde{\tau}_{\DD,\alpha}$

    \end{algorithmic}
\end{algorithm}

\begin{algorithm}
    \caption{\textsc{FDP.Binary.Tree}$(\breve{\mathcal{D}},  \{\widetilde{\eta}_a, \widetilde{\pi}_a\}_{a\in\{0,1\}}, M)$.} 
    \label{alg_bin_tree}
    \begin{algorithmic}
        \Require Data $\breve{\mathcal{D}}$, estimated regression functions and class probabilities  $\{\widetilde{\eta}_a,\widetilde{\pi}_a\}_{a\in\{0,1\}}$, number of layers~$M$.

        \For{$s = 1, \ldots, S; a = 0, 1; \ell =1, \ldots, M\text{ and } k = 1, \ldots, 2^\ell$}
        \State Generate independently
        \[\breve{w}_{s,a,\ell,k} \sim N\Big(0, \frac{4\log(1/\delta_s)/\epsilon_s+2}{\epsilon_s/M} \Big).\]
        \EndFor

        \For{$s = 1, \ldots, S$}
        \For{$a = 0, 1 $}
        \State For any $y \in\{0,1\}$ and $i \in [\breve{n}_{s,a,y}]$, calculate
        \[Z_{s,i,a,y} := 2(2a-1)\widetilde{\pi}_a\{\widetilde{\eta}_a(\breve{X}^{s,i}_{a,y})-1/2\}.\]
        \For{$ j = 1, \ldots, 2^{M}$}
        \State Let
        \[\text{count}_{s,a,M, j} = \sum_{y\in\{0,1\}} \sum_{i=1}^{\breve{n}_{s,a,y}} \indc\Big\{-1+ (j-1)2^{1-M} \leq Z_{s,i,a,y} < -1+ j2^{1-M}\Big\}.\]
        \For{$\ell = M-1, \ldots, 1$}
        \For{$k=1, \ldots, 2^{\ell}$}
        \State Set 
        \[\text{count}_{s, a, \ell,k} = \text{count}_{s,a,l+1, 2k-1} +\text{count}_{s,a,l+1, 2k}.\]
        \EndFor
        \EndFor
        \EndFor

        \For{$\ell =1, \ldots, M$}
        \For{$k = 1, \ldots, 2^\ell$}

        \State Calculate $N_{s, a, \ell,k} = \text{count}_{s, a, \ell,k}+\breve{w}_{s,a,\ell,k}.$
        
        \EndFor
        \EndFor

        \EndFor

        \EndFor
        \Ensure Collection of private trees $\{\{N_{s, a, \ell,k}\}_{ \ell=1,k =1}^{M,2^\ell}\}_{a \in\{0,1\}, s \in[S]}$.
    \end{algorithmic}
\end{algorithm}

\begin{algorithm}[!htbp]  
    \caption{\textsc{Non.increasing}($\{g_i\}_{i \in [2^M+1]}, \{\mu_s\}_{s\in[S]}, C_{\omega}$, $M$, $\eta$)}
    \label{alg_non_increasing}
    \begin{algorithmic}
        \Require Sequence of points $\{g_i\}_{i \in [2^M+1]}$, weight $\{\mu_s\}_{s\in[S]}$, constant $C_{\omega}$, binary tree layers $M$, tolerance $\eta$.

        \State Set the tolerance level 
        \begin{align*}
            \omega = C_{\omega} \sqrt{\sum_{s=1}^S\frac{\mu_s^2M^4\log(1/\delta_s)\log(M/\eta)}{\breve{n}_s^2\epsilon_s^2}}.
        \end{align*}
         \If{$g_i \geq \ldots \geq g_{2^M+1}$}
        \State \textbf{Break} and return $\{g_i\}_{i \in [2^M+1]}$.

        \Else
        \State Initialise the algorithm by setting $f_1 = g_1 + \omega$.

        \For{$i = 2, \ldots, 2^M+1$}
        \State Set $f_i = \min\{f_{i-1}, g_i + \omega\}$.

        \If{$f_i < g_i - \omega$}
        \State \textbf{Break} and return \textbf{NULL}.

        \Else 
        \State \textbf{Continue}
        \EndIf

        \EndFor
        \EndIf

        \Ensure Sequence of non-increasing numbers $\{f_i\}_{i \in [2^M+1]}$.
    \end{algorithmic}
   
\end{algorithm}

\subsection{Special case: Central differential privacy} \label{sec_cdp_algorithm}
In the special case when $S = 1$ and $N_1 = N$, \Cref{def_fdp} reduces to CDP as defined in \Cref{def_cdp}. Rather than implementing \Cref{alg_fair_fdp}, which targets the general setting when $S > 1$, we present in this subsection a simplified yet powerful algorithm dedicated to the CDP setting, given in \Cref{alg_fair_cdp}.

Compared with \Cref{alg_fair_fdp}, the main difference lies in \Cref{alg_fair_cdp} \textbf{S2.}, which is the estimation of adjusted threshold $\tau^*_{\DD,\alpha}$. Instead of using the noisy binary tree based algorithm, we show that shifting the empirical DD vertically by a correctly calibrated Gaussian noise is sufficient. The key idea is illustrated in \Cref{fig_tau_cdp}. 

\begin{figure}[!htbp]
    \centering
    \includegraphics[width=0.6\linewidth]{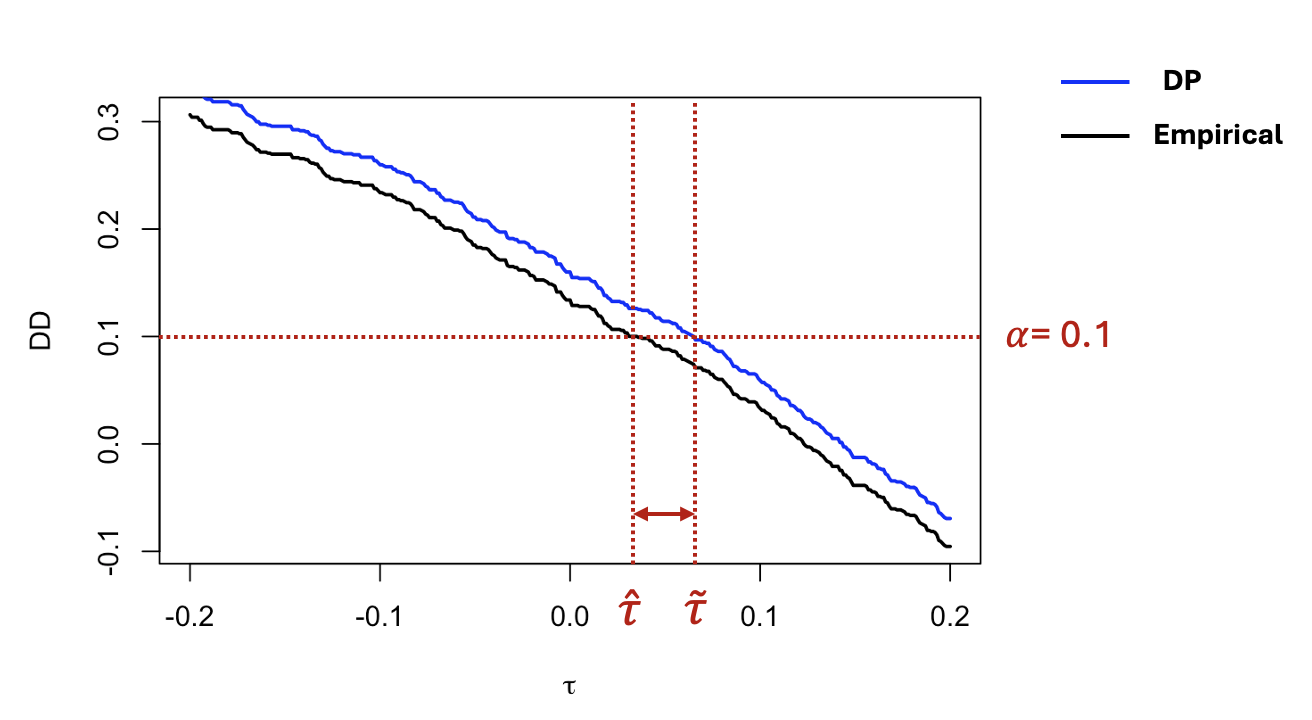}
    \caption{Graphical illustration of \textbf{S2.} of \Cref{alg_fair_cdp} when $\alpha = 0.1$.}
    \label{fig_tau_cdp}
\end{figure}

Given the fact that the empirical DD, $\widehat{\DD}$ in \eqref{eq_DD_tree} with $S = 1$, is monotonically non-increasing (\Cref{l_DD_non_increase}), the optimisation problems $\widehat{\tau}_{\DD,\alpha} = \argmin_{\tau \in \mathbb{R}}\{|\tau|:\, |\widehat{\DD}(\tau)| \leq \alpha \}$ is equivalent to finding the smallest cut-off point in magnitude at which $\widehat{\DD}(\tau)$ first reaches the fairness threshold $\alpha$ or $-\alpha$. Therefore, to construct a private cut-off point $\widetilde{\tau}_{\DD,\alpha}$, instead of adding noise directly to the cut-off point, we shift the $\widehat{\DD}$ function itself. Geometrically, this corresponds to perturbing the $\widehat{\DD}$ curve in the vertical direction (i.e.~$y$-axis). As a result, the location at which the perturbed $\widehat{\DD}$ curve crosses the threshold is shifted horizontally, producing a privacy-preserving cut-off point. By adding correctly calibrated noise and ensuring that each pointwise release of the $\widehat{\DD}$ function at any given $\tau$ satisfies CDP, we are able to show that the induced release of the cut-off point $\widetilde{\tau}_{\DD,\alpha}$ also satisfies CDP. We defer the formal privacy guarantees to \Cref{thm_privacy_fair} in \Cref{sec_theoretical}.

\begin{algorithm}[!htbp]
    \caption{\textsc{CDP-Fair} classifier under demographic disparity. \label{alg_fair_cdp}}
    \begin{algorithmic}
        \Require Data $\mathcal{D}$, disparity level $\alpha$, bandwidth parameters $h$.

        \State Let $n= \lfloor N/2 \rfloor$ and $\breve{n} = N - n$. Split the data in each site into training data $\{(X_i,A_i,Y_i)\}_{i=1}^{n}$ and calibration data $\{(X_i, A_i, Y_i)\}_{i=n+1}^{N} = \{\breve{X}_i, \breve{A}_i, \breve{Y}_i\}_{i=1}^{\breve{n}}$. We further denote $\breve{n}_{a} = |\{(X_i,A_i=a,Y_i)\}_{i=1}^{n_s}|$.

        \State \textbf{S1.} Implement \textbf{S1} in \Cref{alg_fair_fdp} with $S=1$ to estimate regression function $\eta_a$ and class probabilities $\pi_{a}$ using training data adhering to central DP constraints.

        \State \textbf{S2.} Estimate the optimal threshold using calibration data.

        \State Generate $\breve{w}\sim N(0,\sigma^2)$, where $\sigma = 2\sqrt{2\log(1.25/\delta)}/(\{\breve{n}_0\wedge \breve{n}_1\}\epsilon)$. Calculate
        \begin{align*}
            \widetilde{\DD}(\tau) = \frac{1}{\breve{n}_1}\sum_{y\in\{0,1\}}\sum_{i=1}^{\breve{n}_{1,y}}\indc\Big\{\widetilde{\eta}_1(\breve{X}^i_{1,y}) \geq \frac{1}{2}+\frac{\tau}{2\widetilde{\pi}_1}\Big\} -\frac{1}{\breve{n}_0}\sum_{y\in\{0,1\}}\sum_{i=1}^{\breve{n}_{0,y}}\indc\Big\{\widetilde{\eta}_0(\breve{X}^i_{0,y}) \geq \frac{1}{2}-\frac{\tau}{2\widetilde{\pi}_0}\Big\}+ \breve{w}.
        \end{align*}

        \State Set $\widetilde{\tau}_{\DD,\alpha} = \argmin_{\tau \in \mathbb{R}}\{|\tau|:\, |\widetilde{\DD}(\tau)| \leq \alpha \}$. 

        \Ensure $\widetilde{f}_{\DD,\alpha}(x,a)$ with 
        \begin{align*}
            \widetilde{f}_{\DD,\alpha}(x,a) = \begin{cases} \vspace{0.5em}
            1, &\widetilde{\eta}_a(x) \geq \frac{1}{2}+\frac{\widetilde{\tau}_{\DD,\alpha}(2a-1)}{2\widetilde{\pi}_a},\\ 
            0, &\widetilde{\eta}_a(x) < \frac{1}{2}+\frac{\widetilde{\tau}_{\DD,\alpha}(2a-1)}{2\widetilde{\pi_a}}.
        \end{cases}
        \end{align*}
    \end{algorithmic}
\end{algorithm}

\section{Theoretical properties} \label{sec_theoretical} 
In this section, we provide theoretical guarantees on privacy, fairness, and excess risk control for Algorithms \ref{alg_fair_fdp} and \ref{alg_fair_cdp}.

\subsection{Assumptions}

To kick off, we list a few assumptions below.
\begin{assumption}[Class and feature probability] \label{a_prob}
    We assume that the following holds for any $a \in \{0,1\}$.
    \begin{enumerate}[label=\textbf{\alph*}]
        \item \label{a_prob_pi_a} There exists an absolute constant $0< C_\pi <1$ such that the class probability $\pi_a = \mathbb{P}(A = a) \geq C_\pi$.

        \item \label{a_prob_p_x_a} There exists an absolute constant $L>0$ such that $p_{X|A=a}\in \mathcal{H}^\beta(L)$. We further assume that $\inf_{x\in [0,1]^d} p_{X|A=a}(x|a) \geq C_p$ for an absolute constant $C_p >0$.
    \end{enumerate}
\end{assumption}

\begin{assumption}[Kernel function] \label{a_kernel} Let the kernel function $K:\mathbb{R}^d \rightarrow \mathbb{R}^+$ satisfy the following conditions.
    \begin{enumerate}[label=\textbf{\alph*}]
        \item \label{a_kernel_bounded}  There exists an absolute constant $C_K>0$ such that $\sup_{x\in \mathbb{R}^d} K(x) \leq C_K$.

        \item \label{a_kernel_lipschitz} We assume that the function $K$ is Lipschitz, i.e.~there exists an absolute constant $C_{\mathrm{Lip}}>0$ such that for any $x,y \in \mathbb{R}^d$, $|K(x)-K(y)| \leq C_{\mathrm{Lip}}\|x-y\|_2$. 

        \item \label{a_kernel_adaptive} The kernel function $K$ is adaptive to H\"{o}lder class $\mathcal{H}^\beta(L)$, i.e.~for $L>0$ and any $f \in \mathcal{H}^\beta(L)$, it holds that $\sup_{x\in [0,1]^d}|\int_{[0,1]^d} K_h(x-u)f(u)\;\mathrm{d}u - f(x)| \leq C_{\mathrm{adp}} h^{\beta}$, where $C_{\mathrm{adp}} >0$ is an absolute constant only depending on $L$. 
    \end{enumerate}
\end{assumption}

In \Cref{a_prob}\ref{a_prob_pi_a}, we assume that the class probabilities are bounded away from~0~and~1, essential to ensure that a sufficient number of samples for each group can be observed. In \Cref{a_prob}\ref{a_prob_p_x_a}, we further regulate the smoothness and boundedness of the per-sensitive class density of feature distributions. Similar assumptions are commonly seen in the literature of non-parametric statistics \citep[e.g.][]{madrid2023change, auddy2025minimax}. 

\Cref{a_kernel} provides standard kernel assumptions in the nonparametric literature \citep[e.g.][]{kim2019uniform, madrid2023change}, holding for various kernels, such as uniform, Epanechnikov and Gaussian. In particular, \Cref{a_kernel}\ref{a_kernel_adaptive}~holds for any $\beta$-valid kernel functions.\footnote{For a fixed $\beta > 0$, a function $K : \mathbb{R}^d \rightarrow \mathbb{R}^+$ is a $\beta$-valid kernel if $\int_{\mathbb{R}^d} K(x)\; \mathrm{d}x = 1$, $\|K\|_{L_p} < \infty$ for all $p \geq 1$, $\int_{\mathbb{R}^d} \|x\|^{s} K(x)\;\mathrm{d}x < \infty$ and $\int_{\mathbb{R}^d} x^s K(x)\;\mathrm{d}x = 0 $ for all $ s = (s_1,\dots,s_d) \in \mathbb{Z}^d$ such that $1 \leq \sum_{i=1}^d s_i \leq \lfloor \beta \rfloor$, where for $x=(x_1,\dots,x_d)\in\mathbb{R}^d$, $x^s = \prod_{i=1}^d x_i^{s_i}$. See e.g.~Definition A.1 in \citet{rigollet2009optimal}.}

\begin{assumption}[Regression function] \label{a_posterior} 
    Recall $\tau^*_{\DD, \alpha}$, the optimal magnitude of the adjusted threshold for fairness, defined in \eqref{eq_bayes_fair}.  Suppose the following holds for any $a \in \{0,1\}$.
    \begin{enumerate}[label=\textbf{\alph*}]
        \item \label{a_posterior_holder}The class-wise regression function is H\"{o}lder continuous in $x$ over $[0,1]^d$, i.e.~there exists $L>0$ such that $\eta_a(x) \in \mathcal{H}^\beta(L)$.

        \item \label{a_posterior_margin} For any $\tau \in \mathbb{R}$, denote $T_a(\tau) =1/2+ \tau(2a-1)/(2\pi_a)$. We assume that the margin condition holds in a small neighbourhood around $T_a(\tau^*_{\DD,\alpha})$, i.e. for a given small $\varpi >0$, there exists absolute constants $C_m, \gamma \geq 0$ such that for any $\kappa >0$, $\sup_{|\xi| \leq \varpi} \mathbb{P}\{|\eta_a(X)-T_a(\tau^*_{\DD,\alpha}+\xi)| \leq \kappa\} \leq C_m\kappa^\gamma$. 

        \item \label{a_posterior_disparity} In addition to \Cref{a_posterior}\ref{a_posterior_margin}, in the case when $|\DD(\tau^*_{\DD,\alpha})| =\alpha$, we assume that there exists an absolute constant $c_m >0$ such that for any small $\kappa$ in the neighbourhood of $0$, $|\DD(\tau^*_{\DD,\alpha}+\kappa) -\DD(\tau^*_{\DD,\alpha})| \geq c_m |\kappa|^{\gamma}$.
    \end{enumerate}
\end{assumption}

In \Cref{a_posterior}, we provide assumptions on class-wise regression functions $\eta_a$, $a \in \{0,1\}$, specifically for the task of fairness-aware classification. Similar assumptions can also be found in \citet{zeng2024minimax}. 

\Cref{a_posterior}\ref{a_posterior_margin} characterises the decay rate of the regression function within a band of width $\varpi$ around the decision boundary. This assumption is a mild modification of the margin condition commonly used in the non-parametric classification literature \citep[e.g.][]{audibert2007fast, zeng2024bayes}. \Cref{a_posterior}\ref{a_posterior_margin} enables us to obtain theoretical guarantees on the estimation error of DD over the critical range relevant for searching the adjusted threshold $\widetilde{\tau}_{\DD,\alpha}$. Theoretically, it suffices for $\varpi$ to be larger than the order $\rho^{1/\gamma}$, where $\rho$ is given in \eqref{t_fdp_risk_eq1}.

\Cref{a_posterior}\ref{a_posterior_disparity} further controls the steepness
of DD in a small neighbourhood of $\tau^*_{\DD,\alpha}$ for the purpose of threshold $\tau^*_{\DD,\alpha}$ estimation. The smaller the gamma, the steeper the DD near the boundary, the easier the estimation task. At a high level, \Cref{a_posterior}\ref{a_posterior_disparity} can also be treated as lower bound counterpart for \Cref{a_posterior}\ref{a_posterior_margin} since by definition we have 
\begin{align*}
    c_m |\kappa|^{\gamma} \leq \;&|\DD(\tau^*_{\DD,\alpha}) - \DD(\tau^*_{\DD,\alpha}+\kappa)|\\
    \leq \;& \mathbb{P}\Big\{|\eta_1(X)-T_1(\tau^*_{\DD,\alpha})| < \frac{|\kappa|}{2\pi_1}\Big\} + \mathbb{P}\Big\{|\eta_0(X) - T_0(\tau^*_{\DD,\alpha}) |< \frac{\kappa}{2\pi_0}\Big\}. 
\end{align*}

\subsection{Privacy and fairness control}
With the above assumptions, we establish the privacy and fairness guarantee of Algorithms~\ref{alg_fair_fdp} and \ref{alg_fair_cdp} in this subsection.

\begin{theorem} \label{thm_privacy_fair}
    Denote $\widetilde{f}^{\mathrm{FDP}}_{\DD,\alpha}$ and $\widetilde{f}^{\mathrm{CDP}}_{\DD,\alpha}$ the output of Algorithms \ref{alg_fair_fdp} and \ref{alg_fair_cdp} respectively. Then under Assumptions \ref{a_prob}, \ref{a_kernel} and \ref{a_posterior}, the following holds.
    
    \begin{enumerate}
        \item \Cref{alg_fair_fdp} is $(\bm{\epsilon},\bm{\delta})$-FDP and \Cref{alg_fair_cdp} is $(\epsilon,\delta)$-CDP.
        
        \item \label{thm_privacy_fair_cdp} It holds that
        \begin{equation} \label{t_cdp_fair_eq1}
            \big|\DD(\widetilde{f}^{\mathrm{CDP}}_{\DD,\alpha})\big| \leq \alpha + \widetilde{O}_{\mathrm{p}}\Big(\sqrt{\frac{1}{N}}+\sqrt{\frac{1}{N^2\epsilon^2}}\Big),
        \end{equation}
        where the probability of $\DD(\widetilde{f}^{\mathrm{CDP}}_{\DD,\alpha})$ is taken over the test sample conditioning on training data, and $\mathrm{p}$ in $\widetilde{O}_{\mathrm{p}}$ captures the randomness of training data.

        \item \label{thm_privacy_fair_fdp} For absolute constants $C_1, C_2, C_3 >0$, we further assume that $\min_{s\in[S]} N_s^2 \epsilon_s^2 \geq C_1$; $\varpi$ given in \Cref{a_posterior}\ref{a_posterior_margin} satisfying $\varpi \geq C_2\rho^{1/\gamma}$; and  $\DD(0) \notin [\alpha - \zeta,\alpha] \cup [-\alpha,-\alpha+\zeta]$ where 
        \begin{align} \notag
            \rho =\;& \Bigg\{\sqrt{\sum_{s=1}^S \frac{\nu^2_s}{N_{s}h^d}} +  \max_{s \in [S]} \frac{\nu_s}{N_{s}h^{d}}+ h^\beta+ \sqrt{\sum_{s=1}^S\frac{\nu^2_s}{N^2_{s}\epsilon^2_s h^{2d}}}\Bigg\}^\gamma\\\label{t_fdp_risk_eq1}
            &+\sqrt{\sum_{s=1}^S \frac{\mu_s^2}{N_{s}}} + \max_{s \in [S]}\frac{\mu_s }{N_{s}}+\sqrt{\sum_{s=1}^S\frac{\mu_s^2}{N_{s}^2\epsilon_s^2}},
        \end{align}
        and $C_3\rho \leq \zeta <\alpha$.
         Then if we initialise \Cref{alg_fair_fdp} with the bin width $\theta = C_1\rho^{1/\gamma}$, it holds that 
         \begin{align*}
             \big|\DD(\widetilde{f}^{\mathrm{FDP}}_{\DD,\alpha})\big| \leq \alpha+ \widetilde{O}_{\mathrm{p}}(\rho),
         \end{align*}
          where the probability of $\DD(\widetilde{f}^{\mathrm{FDP}}_{\DD,\alpha})$ is taken over the test sample conditioning on training data, and $\mathrm{p}$ in $\widetilde{O}_{\mathrm{p}}$ captures the randomness of training data.

        \end{enumerate}
\end{theorem}

\Cref{thm_privacy_fair} is a direct consequence of Propositions  \ref{t_cdp_privacy_guarantee},  \ref{thm_cdp_fair_guarantee}, \ref{thm_fdp_guarantee}, \ref{thm_fdp_disparity_control} in the Appendix. It demonstrates that both Algorithms \ref{alg_fair_fdp} and \ref{alg_fair_cdp} satisfy the privacy constraints, while at the same time only inflate the pre-specified disparity level by a small offset term. 

As constructed in \textbf{S2.}~of \Cref{alg_fair_cdp}, the magnitude of the offset term associated to $\widetilde{f}^{\mathrm{CDP}}_{\DD,\alpha}$, $N^{-1/2} + (N\epsilon)^{-1}$, is determined by the high-probability upper bound on $|\widetilde{\DD}(\widetilde{\tau}_{\DD,\alpha}) - \DD(\widetilde{\tau}_{\DD,\alpha})|$, which captures both the deviation of the empirical distribution from its population counterpart, i.e.~$|\widehat{\DD}(\widetilde{\tau}_{\DD,\alpha}) - \DD(\widetilde{\tau}_{\DD,\alpha})|$, and the variance introduced by the Gaussian noise, i.e.~$|\breve{w}|$. The parametric rate, $N^{-1/2}$, is commonly observed for plug-in type estimators in the existing literature on fairness-aware classification without privacy constraints \citep[e.g.][]{hou2024finite, hu2025fairness}. Under additional CDP constraints, as expected, this classical parametric rate is accompanied by a standard additional private parametric rate \citep[e.g.][]{cai2021cost} of order $(N\epsilon)^{-1}$. Similar results on disparity control under CDP constraints can also be found in \citet{ghoukasian2024differentially}.

For $\widetilde{f}^{\mathrm{FDP}}_{\DD,\alpha}$ output by \Cref{alg_fair_fdp}, the disparity level will exceed the desired threshold up to the order of $\rho$, larger than the corresponding excess in \eqref{t_cdp_fair_eq1} for the CDP case when $S = 1$. Unlike CDP, privacy compositions during the communication between servers prohibit infinite pointwise evaluation of DD as discussed in \Cref{remark_tree}, thereby making the construction of a direct plug-in optimisation-based algorithm infeasible. The higher inflation term is primarily driven by the discretisation error introduced by the noisy binary tree construction in \Cref{alg_bin_tree}, and corresponds to the estimation error of DD per evaluation. 

\begin{remark}
    We would like to remark that if one instead insists on controlling the population unfairness $\DD(\widetilde{f}^{\mathrm{FDP}}_{\DD,\alpha})$ and $ \DD(\widetilde{f}^{\mathrm{CDP}}_{\DD,\alpha})$ below $\alpha$ with large probability, then provided that  $\alpha \gtrsim \rho$  or $\alpha \gtrsim N^{-1/2} + (N\epsilon)^{-1}$, it suffices to adjust the input of \Cref{alg_fair_fdp} and $\Cref{alg_fair_cdp}$ by $\alpha -\rho$ and $\alpha - N^{-1/2} + (N\epsilon)^{-1}$ respectively. 
\end{remark}

\subsection{Excess risk control}
We next present the excess risk control of \Cref{alg_fair_fdp}, with the excess risk control of \Cref{alg_fair_cdp} deferred to \Cref{t_fair_risk} in \Cref{sec_appendix_excess_risk_cdp}. 

Note that for any $f \in \mathcal{F}$, the ordinary excess risk $R(f)-R(f^*_{\DD,\alpha})$, where $f^*_{\DD,\alpha}$ is the optimal Bayes fairness-aware classifier given in \eqref{def_bayes_fair_classifier}, may be negative as $f_{\DD, \alpha}^\star$ does not necessarily minimise the excess risk, i.e.~$f_{\DD, \alpha}^\star \notin \argmin_{f\in \mathcal{F}} R(f)$. To make the excess risk control meaningful, we resort to the quantity $|R(f)-R(f^*_{\DD,\alpha})|$, which can be further decomposed as the non-negative fairness-aware excess risk $d_{\mathrm{fair}}(f, f^{\star}_{\DD, \alpha})$ (\Cref{def_fair_aware_risk}) and a disparity cost, namely $|R(f)-R(f^*_{\DD,\alpha})| \leq d_{\mathrm{fair}}(f, f^*_{\DD,\alpha}) +|\tau^*_{\DD,\alpha}|\cdot |\DD(f^*_{\DD,\alpha})-\DD(f)|$ \citep[e.g.~Proposition 4.2 in][]{zeng2024minimax}.

\begin{definition}[Fairness-aware excess risk under DD, Definition 4.1 in \citealp{zeng2024minimax}]\label{def_fair_aware_risk}
Let $\alpha \geq 0$ and let $f^*_{\DD,\alpha}$ be an $\alpha$-fair Bayes-optimal classifier in \eqref{def_bayes_fair_classifier}. For any classifier $f:[0,1]^d \times\{0,1\}\rightarrow [0,1]$, the fairness-aware excess risk under $\DD$ is defined as
\begin{equation*}
d_{\mathrm{fair}}(f, f^*_{\DD,\alpha})
= 2 \sum_{a\in\{0,1\}} \pi_a
\Big[\int \big\{ f(x,a) - f^*_{\DD,\alpha}(x,a) \big\}
\Big\{\frac{1}{2}+\frac{\tau^*_{\DD,\alpha}(2a-1)}{2\pi_a} - \eta_a(x) \Big\}\;
\mathrm{d}\mathbb{P}_{X\mid A=a}(x) \Big].
\end{equation*}
\end{definition}

\begin{remark}
    An important property of fairness-aware excess risk is that $d_{\mathrm{fair}}(f, f^*_{\DD,\alpha}) \geq 0$ for all $f \in \mathcal{F}$ since it follows from \eqref{eq_bayes_fair} that $f(x,a) - f^*_{\DD,\alpha}(x,a) \leq 0$ whenever $\eta_a(x) \geq 1/2 +(\tau^*_{\DD,\alpha}(2a-1))/(2\pi_a)$, and vice versa. In the case when $f^*_{\DD,\alpha}(x,a)$ is automatically fair, $d_{\mathrm{fair}}(f, f^*_{\DD,\alpha})$ reduces to standard excess risk, and it holds that $d_{\mathrm{fair}}(f, f^*_{\DD,\alpha}) = R(f) - R(f^*_{\DD,\alpha})$.
\end{remark}

We are now ready to present the excess risk control for \Cref{alg_fair_fdp} in \Cref{t_fdp_risk}.

\begin{theorem}\label{t_fdp_risk} 
    Let $C_1>0$ be an absolute constant. Under the same assumptions and notation as in \Cref{thm_privacy_fair}.\ref{thm_privacy_fair_fdp}, if we initialise \Cref{alg_fair_fdp} with the bin width $\theta = C_1\rho^{1/\gamma}$, then the following holds.

    \begin{enumerate}
        \item \label{t_fdp_risk_1} The fairness-aware excess risk satisfies
        \begin{align*}
            d_{\mathrm{fair}}(\widetilde{f}^{\mathrm{FDP}}_{\DD,\alpha}, f^*_{\DD,\alpha})= \widetilde{O}_{\mathrm{p}}\Bigg[&\Bigg\{\sqrt{\sum_{s=1}^S \frac{\nu^2_s}{N_{s}h^d}} +  \max_{s \in [S]} \frac{\nu_s}{N_{s}h^{d}}+ h^\beta+ \sqrt{\sum_{s=1}^S\frac{\nu^2_s}{N^2_{s}\epsilon^2_s h^{2d}}}\Bigg\}^{(1+\gamma)}\\
            &+\indc\{\tau^*_{\DD, \alpha} \neq 0\}\Bigg\{\sqrt{\sum_{s=1}^S \frac{\mu_s^2}{N_{s}}} + \max_{s \in [S]}\frac{\mu_s }{N_{s}}+\sqrt{\sum_{s=1}^S\frac{\mu_s^2}{N_{s}^2\epsilon_s^2}}\Bigg\}^{\frac{1+\gamma}{\gamma}}\Bigg].
        \end{align*}

        \item  \label{t_fdp_risk_2} In the fairness-impacted regime when $\tau^*_{\DD, \alpha}  \neq 0$, choose the bandwidth $h_{\opt} \in \mathbb{R}_+$ as the smallest positive real number satisfying 
        \begin{align} \label{t_fdp_risk_eq4}
            h^{-2\beta}_{\opt} \asymp h^d_{\opt}\sum_{s=1}^S (N_s \wedge N_s^2\epsilon_s^2h^d_{\opt}) + \Big\{\sum_{s=1}^S (N_s \wedge N_s^2\epsilon_s^2)\Big\}^{1/\gamma},
        \end{align}
        and pick the weights 
        \begin{align} \label{t_fdp_risk_eq2}
            \nu_s = \frac{u_s}{\sum_{j=1}^S u_j}, \text{ where } u_s = N_s \wedge N_s^2\epsilon_s^2h_{\opt}^d,
        \end{align}
        and
        \begin{align} \label{t_fdp_risk_eq3}
            \mu_s = \frac{u'_s}{\sum_{s=1}^S u'_s},  \text{where } u_s' = N_s \wedge N_s^2\epsilon_s^2.
        \end{align}
        Then it holds that 
        \begin{align} \label{t_fdp_risk_eq5}
            d_{\mathrm{fair}}(\widetilde{f}^{\mathrm{FDP}}_{\DD,\alpha}, f^*_{\DD,\alpha}) = \widetilde{O}_{\mathrm{p}}\Big(h_{\opt}^{\beta(1+\gamma)}\Big),
        \end{align}
        and
        \begin{align*}
            |R(\widetilde{f}^{\mathrm{FDP}}_{\DD,\alpha}) - R(f^*_{\DD,\alpha})| = d_{\mathrm{fair}}(\widetilde{f}^{\mathrm{FDP}}_{\DD,\alpha}, f^*_{\DD,\alpha}) + |\tau^*_{\DD,\alpha}|\widetilde{O}_{\mathrm{p}}(h_{\opt}^{\beta\gamma}).
        \end{align*}

        \item \label{t_fdp_risk_3} In the automatically fair regime when $\tau^*_{\DD, \alpha} = 0$, choose the bandwidth $h_{\opt} \in \mathbb{R}_+$ as the smallest positive real number satisfying
        \begin{align*}
            h^{-2\beta}_{\opt} \asymp h^d_{\opt}\sum_{s=1}^S (N_s \wedge N_s^2\epsilon_s^2h^d_{\opt}),
        \end{align*}
        and pick the same weights as \eqref{t_fdp_risk_eq2} and \eqref{t_fdp_risk_eq3}. Then it holds that 
        \begin{align*}
            d_{\mathrm{fair}}(\widetilde{f}^{\mathrm{FDP}}_{\DD,\alpha}, f^*_{\DD,\alpha}) = \widetilde{O}_{\mathrm{p}}\Big(h_{\opt}^{\beta(1+\gamma)}\Big),
        \end{align*}
        and
        \begin{align*}
            |R(\widetilde{f}^{\mathrm{FDP}}_{\DD,\alpha}) - R(f^*_{\DD,\alpha})| = d_{\mathrm{fair}}(\widetilde{f}_{\DD,\alpha}, f^*_{\DD,\alpha}) + |\tau^*_{\DD,\alpha}|\widetilde{O}_{\mathrm{p}}(h_{\opt}^{\beta\gamma}).
        \end{align*}
    \end{enumerate}    
\end{theorem}

We defer the proof of \Cref{t_fdp_risk} to \Cref{sec_app_proof_t_fdp_risk}. \Cref{t_fdp_risk} is the first time in literature, providing finite-sample guarantees for excess risk in fairness-aware classification under FDP constraints. To optimise for a heterogeneous distributed setting, we choose the bandwidth~$h$ by balancing the bias and variance trade-offs, and select weights $\{\nu_s\}_s$ and $\{\mu_s\}_s$ that are proportional to the effective sample size at each server.

To better interpret the results of $ d_{\mathrm{fair}}(\widetilde{f}_{\DD,\alpha}, f^*_{\DD,\alpha})$ in Theorems~\ref{t_fdp_risk}.\ref{t_fdp_risk_2} and~\ref{t_fdp_risk}.\ref{t_fdp_risk_3}, we present in \Cref{coro_fdp_risk_homo} the fairness-aware excess risk control for the homogeneous setting.

\begin{corollary}\label{coro_fdp_risk_homo}
     Under the same assumptions as in \Cref{t_fdp_risk}, in a homogeneous setting when $N_s = N$ and $\epsilon_s = \epsilon$ for all $s \in [S]$, if we pick weights as $\nu_s = \mu_s =S^{-1}$, bandwidth as 
     \[ h_{\opt} \asymp (SN)^{-\frac{(1+\gamma)}{2\beta+d}} + (SN^2\epsilon^2)^{-\frac{(1+\gamma)}{2\beta+2d}}+ \indc\big\{\tau^*_{\DD,\alpha} \neq 0\} \cdot \big\{(SN)^{-\frac{1+\gamma}{2\gamma\beta}} + (SN^2\epsilon^2)^{-\frac{1+\gamma}{2\gamma\beta}}\big\} \]
     and the bin width as 
     \[\theta \asymp (SN)^{-\frac{\beta}{2\beta+d}} + (SN^2\epsilon^2)^{-\frac{\beta}{2\beta+2d}} + (SN)^{-\frac{1}{2\gamma}} + (SN^2\epsilon^2)^{-\frac{1}{2\gamma}},\]
     it holds that 
      \begin{align} \notag
          d_{\mathrm{fair}}(\widetilde{f}_{\DD,\alpha}, f^*_{\DD,\alpha}) = \widetilde{O}_{\mathrm{p}}\Big(&(SN)^{-\frac{\beta(1+\gamma)}{2\beta+d}} + (SN^2\epsilon^2)^{-\frac{\beta(1+\gamma)}{2\beta+2d}} \\ \label{coro_fdp_risk_homo_eq1}
          &\hspace{6em} + \indc\big\{\tau^*_{\DD,\alpha} \neq 0\} \cdot \big\{(SN)^{-\frac{1+\gamma}{2\gamma}} + (SN^2\epsilon^2)^{-\frac{1+\gamma}{2\gamma}}\big\}\Big).
      \end{align}      
\end{corollary}

\begin{remark}[Choice of bandwidth $h$]
    The choice of bandwidth, $h_{\mathrm{opt}}$, in \Cref{coro_fdp_risk_homo} is guided by solving equation \eqref{t_fdp_risk_eq4}, which matches the bias with four variance terms arising from privacy and fairness constraints. We would like to remark that, in the homogeneous case, it suffices to choose $h_{\mathrm{opt}} = (SN)^{-\frac{(1+\gamma)}{2\beta+d}} + (SN^2\epsilon^2)^{-\frac{(1+\gamma)}{2\beta+2d}}$ without affecting the final rate. The choice of $h_{\mathrm{opt}}$ via \eqref{t_fdp_risk_eq4} in \Cref{t_fdp_risk} is solely to present the final rate for $ d_{\mathrm{fair}}(\widetilde{f}^{\mathrm{FDP}}_{\DD,\alpha}, f^*_{\DD,\alpha})$ in \eqref{t_fdp_risk_eq5} more clearly in the form $h_{\opt}^{\beta(1+\gamma)}$.
\end{remark}

At a high level, the upper bound in \Cref{coro_fdp_risk_homo} is of the form
\begin{align*}
    d_{\mathrm{fair}}(\widetilde{f}_{\DD,\alpha}, f^*_{\DD,\alpha}) \leq\;& \mbox{intrinsic cost of classificataion} + \mbox{cost of privacy in classification} \\
    &+ \mbox{non-private cost of fairness} + \mbox{private cost of fairness},
\end{align*}
where the cost of privacy corresponds to the terms involving $\epsilon$, and the cost of fairness corresponds to the terms that are active when $\tau^*_{\DD,\alpha} \neq 0$. This result decouples the excess risk into different sources regarding the privacy and fairness constraints.

When $\tau^*_{\DD,\alpha} =0$, we are in the automatically fair regime and $f^*_{\DD,\alpha}$ in \eqref{eq_bayes_fair} reduces to the unconstrained groupwise Bayes optimal classifier, indicating that the latter is intrinsically fair. Consequently, imposing fairness constraints incurs no additional excess risk. In this case, our results in~\eqref{coro_fdp_risk_homo_eq1} recover the excess risk bound, i.e.~$R(\widetilde{f}_{\DD,\alpha}) - R(f^*_{\DD,\alpha})$, for FDP non-parametric classification without fairness constraints, and match the minimax optimal results in \citet{auddy2025minimax}. The cost of privacy is reflected in a reduction of the effective sample size from $SN$ to $SN^2\epsilon^2$ and a loss in the exponent due to an additional dependence on the dimension $d$ in the denominator.

When fairness constraints are in effect, i.e.~$\tau^*_{\DD,\alpha} \neq 0$, the two additional terms $(SN)^{-(1+\gamma)/(2\gamma)}$ and $(SN^2\epsilon^2)^{-(1+\gamma)/(2\gamma)}$ in \eqref{coro_fdp_risk_homo_eq1} together quantify the cost of fairness, arising from the estimation of $\widetilde{\tau}_{\DD,\alpha}$. The former term represents the non-private cost of fairness and echoes the rate in \citet{zeng2024minimax}, which studies fairness-aware classification in a single-server setting without privacy constraints. The latter one reflects the additional cost induced by FDP constraints.

Equivalently, the cost of fairness can be rewritten as $\{(SN)^{-1/2} + (SN^2\epsilon^2)^{-1/2}\}^{(1+\gamma)/\gamma}$, where the term inside the bracket is the standard parametric rate under FDP \citep[e.g.][]{li2024federated}. This parametric term inside the bracket also highlights the importance of the binary tree based construction in \Cref{alg_fair_fdp}. A naive approach to estimate the adjusted threshold $\tau^*_{\DD,\alpha}$ might be a plug-in optimisation problem, based on perturbing site-wise empirical DD using Gaussian processes \citep[e.g.][]{hall2013differential} then aggregating. However, due to the non-smoothness of $\widehat{\DD}_s$ in \eqref{eq_DD_tree}, additional smoothing is required, which in turn introduces non-parametric rates inside the bracket and leads to strictly weaker guarantees. Moreover, the term $(1 + \gamma)$ in the numerator of the exponent in the cost of fairness is linked to the margin assumption in \Cref{a_posterior}\ref{a_posterior_margin}, and similar structures are commonly seen in the existing literature \citep[e.g.][]{audibert2007fast}. The term $\gamma$ in the denominator of the exponent is a consequence of local steepness of $\DD$ around $\tau^*_{\DD,\alpha}$ in \Cref{a_posterior}\ref{a_posterior_disparity}. The steepness, in turn, implies that $|\widetilde{\tau}_{\DD,\alpha}-\tau^*_{\DD,\alpha}| \leq |\DD(\widetilde{\tau}_{\DD,\alpha}) -\DD(\tau^*_{\DD,\alpha})|^{1/\gamma}$ and characterises the $1/\gamma$-type rate in the exponent.

Apart from the choices of bandwidth and site-wise weights, another important ingredient in the analysis is the choice of the bin width $\theta \asymp \rho^{1/\gamma}$ in the noisy binary tree construction (\Cref{alg_bin_tree}), since discretising the search interval introduces bias into the estimation. The bin width $\theta$ is therefore chosen to balance the discretisation bias, i.e.~$|\DD(\tau^*_{\mathcal{G}}) - \DD(\tau^*_{\DD,\alpha})|\lesssim \theta^\gamma$, and the error associated with DD estimation, i.e.~$|\widetilde{\DD}_{\downarrow}(\tau^*_{\mathcal{G}}) - \DD(\tau^*_{\mathcal{G}})|\lesssim \rho$, where $\tau^*_{\mathcal{G}}$ is the closed point to $\tau^*_{\DD,\alpha}$ in the searching grids. This choice subsequently ensures that the optimisation problem in \textbf{S3} of \Cref{alg_threshold_search_fdp} admits a feasible solution and, moreover, limits privacy composition across layers to a poly-logarithmic number of times, as the tree depth satisfies $M = \log_2(\theta^{-1})+1 \asymp \log_2(SN \wedge SN^2\epsilon^2)$. More discussions can be found in \Cref{remark_tree}.

\section{Numerical experiments} \label{sec_numerical}
In this section, we conduct numerical experiments to demonstrate the effectiveness of Algorithms~\ref{alg_fair_fdp} and \ref{alg_fair_cdp}. Experiments with simulated and real datasets are demonstrated in Sections \ref{sec_numerical_simulated} and \ref{sec_numerical_real}, respectively. The code for reproducing all experiments can be found at \url{https://github.com/GengyuXue/DP_Fair_classification}.

\subsection{Simulated data analysis}\label{sec_numerical_simulated}
In this subsection, we carried out independent experiments for Algorithms~\ref{alg_fair_fdp} and \ref{alg_fair_cdp} on simulated datasets to support our theoretical findings in \Cref{sec_theoretical}.

\subsubsection{Numerical experiments for CDP-Fair in Algorithm~\ref{alg_fair_cdp}} \label{sec_numerical_cdp}

We consider the case when $d=2$, where $X= (X_1, X_2)$. We generate the sensitive feature $A$ according to $\pi_1 = \mathbb{P}(A=1)= 0.3$ and the standard feature $(X_1,X_2)$ by $X_1|A=1 \sim \text{Beta}(4,2)$, $X_1|A=0 \sim \text{Beta}(4.5,2)$ and $X_2 \sim \text{Unif}(0,1)$. Given $X = (x_1,x_2)$ and $A=a$, the label $Y$ is then generated by following
\begin{align*}
    \eta_a(x_1,x_2)= \frac{1}{2} + \frac{1}{\pi}\text{arctan}\Big\{12(x_1 + x_2 -1)-0.3(2a-1)\Big\}.
\end{align*}

We generate training datasets with size $N = \{5000, 7000, 9000\}$ and evaluate performance on an independently generated test dataset of size $4000$. The privacy parameter is set as $\epsilon \in \{0.75,1,2,3, 4\}$. The privacy leakage parameter is chosen as $\delta = (N/2)^{-2}$.

The kernel $K$ is chosen to be the Gaussian kernel $(C_K = 1)$, and the bandwidth $h$ is selected via a three-fold cross-validation. Specifically, we choose the bandwidth that attains the lowest misclassification error on the training data without incorporating fairness constraints.

To generate the Gaussian process noise added to the regression functions, we perform an eigen-decomposition of the corresponding covariance matrix and retain only the first $35$ leading eigencomponents for computational efficiency. For simplicity, we assume that the class probabilities $\{\pi_a\}_{a \in \{0,1\}}$ are known and only use their non-private empirical counterparts during the implementation. During implementation, the training set $\mathcal D$ is randomly split into two equal-sized subsets, $\mathcal{D}_1 \cup \mathcal{D}_2$. One subset is used to estimate $\widetilde{\eta}_a$ (\textbf{S1} in \Cref{alg_fair_cdp}), while the other is used to estimate the threshold $\widetilde{\tau}_{\DD,\alpha}$ (\textbf{S2} in \Cref{alg_fair_cdp}).
Let $\widetilde{f}_1$ denote the classifier estimated using $\mathcal{D}_1$ for model estimation and $\mathcal{D}_2$ for threshold calibration, and let~$\widetilde{f}_2$ denote the classifier constructed with the roles of $\mathcal{D}_1$  and $\mathcal{D}_2$ reversed. To mitigate the randomness caused by random splitting, we adopt a cross-fitting approach and define the final probabilistic classifier as the average $\widetilde{f} = (\widetilde{f}_1 + \widetilde{f}_2)/2$.

We report the mean misclassification errors and empirical disparities along with their corresponding $95\%$ confidence bands over 200 iterations in \Cref{fig_cdp}. For comparisons, we also calculate the oracle misclassification error for group-wise classifiers with and without fairness constraints by simulations on an independent dataset of sample size $200000$. The simulated Bayes risk for the group-wise classifier without fairness constraint is $0.106$, and the intrinsic demographic disparity is $0.559$. The simulated Bayes risk for the fairness-aware Bayes classifier is plotted in dashed grey.

Across all sample sizes, \Cref{fig_cdp} exhibits the expected privacy, fairness and accuracy trade-off as shown in \Cref{t_fair_risk} in \Cref{sec_appendix_excess_risk_cdp}. Our designed classifier CDP-Fair successfully keeps the empirical disparities below the diagonal reference $y=x$ for each $\alpha$ in the second row of \Cref{fig_cdp}, indicating that the target disparity level is respected. In terms of classification accuracy, relaxing the fairness constraint (i.e.~increasing $\alpha$) yields a noticeable reduction in misclassification error when the fairness constraint is active (i.e.~small $\alpha$). Once $\alpha$ exceeds a critical threshold, the fairness constraint becomes inactive and the error curves flatten. Moreover, weakening the privacy constraint (i.e.~increasing $\epsilon$) reduces the excess risk for each fixed $\alpha$, with the performance converging towards the oracle benchmark as $\epsilon$ grows. Finally, under fixed privacy and fairness parameters $(\alpha,\epsilon)$, increasing the sample size $N$ leads to uniformly smaller misclassification error, and the difference in performance between CDP-Fair and the oracle Bayes fairness-aware classifier becomes less significant when $\epsilon$ is large.

\begin{figure}[!htbp]
    \centering
    \includegraphics[width=0.9\linewidth]{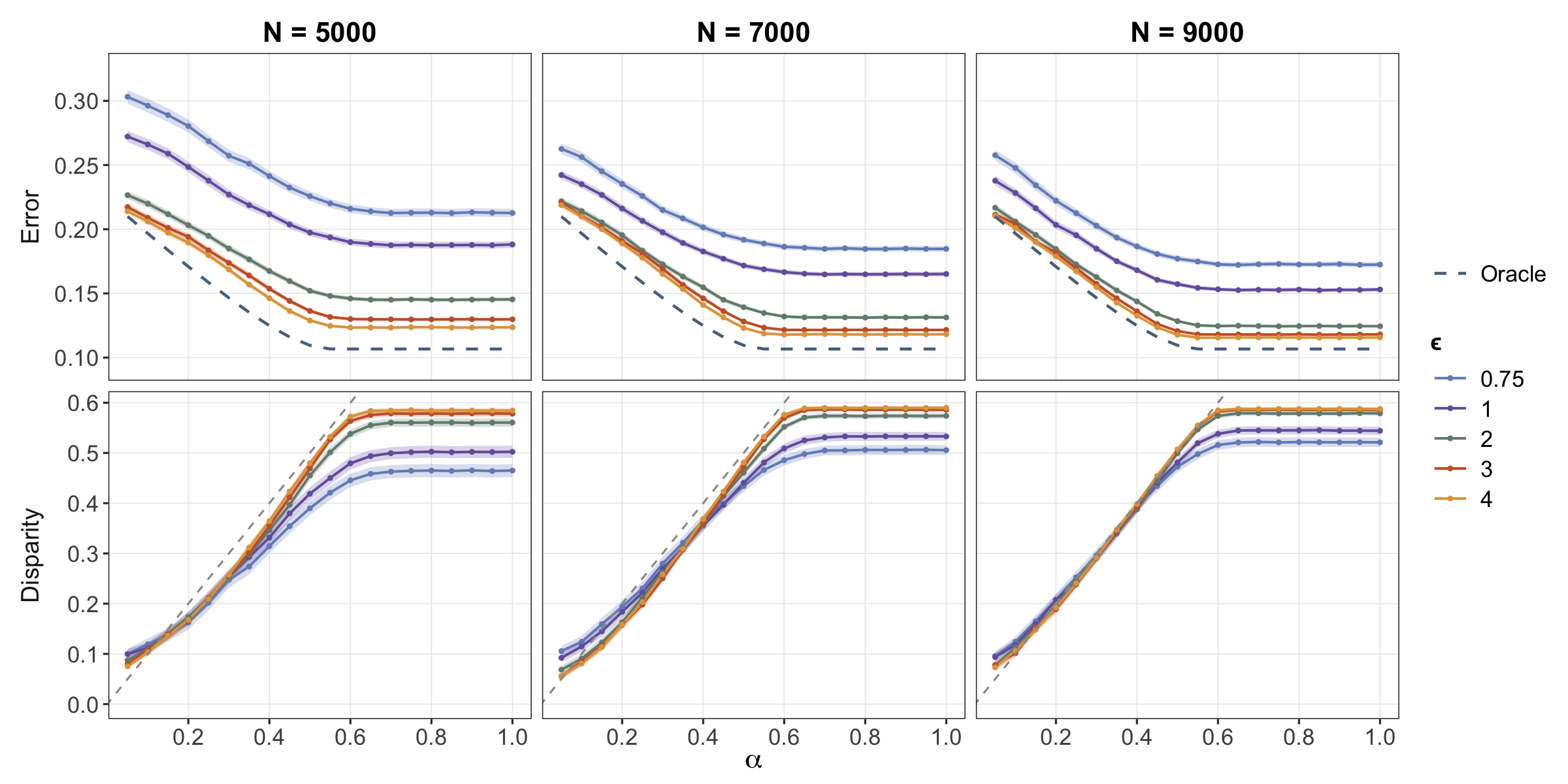}
    \caption{Means and $95\%$ confidence bands for misclassification errors and empirical disparities of \Cref{alg_fair_cdp} when $N \in \{5000, 7000, 9000\}$. The grey dashed line represents $y=x$.}
    \label{fig_cdp}
\end{figure}

\subsubsection{Numerical experiments for FDP-Fair in Algorithm~\ref{alg_fair_fdp}} \label{sec_numerical_FDP}
To demonstrate the effectiveness of FDP-Fair classifier in \Cref{alg_fair_fdp}, we generate data following the same mechanism as described in \Cref{sec_numerical_cdp}, then distribute across various servers. In this subsection, we consider a homogeneous setting where $N_s = N = 2000$, $\epsilon_s = \epsilon \in \{0.75,1, 2, 3, 4, 10\}$, $\delta_s = \delta = (N/2)^{-2}$ and vary the number of servers $S \in \{4, 5, 6\}$. We train the classifier using the distributed training data of total size $N_{\text{total}} = SN$, then test the trained classifier on an independent sample of size $2000$. 

In order to implement \Cref{alg_fair_fdp}, four tuning parameters need to be selected: bandwidth $h$, constants $C_\omega$ and $C_{\rho}$ and minimum interval length $\theta$. For all experiments, following the same approach as in \Cref{sec_numerical_cdp}, we select $h$ with a three-fold cross-validation on a randomly chosen site. To select $\theta$, we utilise our theoretical results and select $\theta$ such that $M = \max\{\lfloor \log_2 \big(\sum_{s=1}^S \min\{N_s, N_s^2\epsilon_s^2\}\big)\rfloor +1, 6\}$. For illustrative purposes, we fix $C_{\omega} = 0.1$ and $\rho = 0.03$ throughout the experiments. Sensitivity analysis for $C_{\omega}$ and $\rho$ when $S= 3$, $N_s = 2000$ and $\alpha =0.3$ are also carried out and defer the results to \Cref{sec_appendix_sensitivity}.

To ensure the feasibility of \Cref{alg_non_increasing} in the high privacy regime when the $\widetilde{\DD}$ curve fluctuates significantly, we adjust \Cref{alg_non_increasing} to \Cref{alg_non_increasing_simulation} detailed in \Cref{sec_appendix_non_increasing_numeircal} in practice. At a high level, \Cref{alg_non_increasing_simulation} outputs a sequence that is as close to non-increasing as possible, subject to bias control. 

We report the mean misclassification errors and empirical disparities along with their corresponding $95\%$ confidence bands over 200 iterations in \Cref{fig_fdp}. Overall, the patterns in \Cref{fig_fdp} closely mirror those observed in \Cref{sec_numerical_cdp} and are consistent with our theoretical results detailed in \Cref{t_fdp_risk}. In particular, as $S$ increases, the total training sample size $SN$ grows, which leads to uniformly better misclassification error.

\begin{figure}[!htbp]
    \centering
    \includegraphics[width=0.9\linewidth]{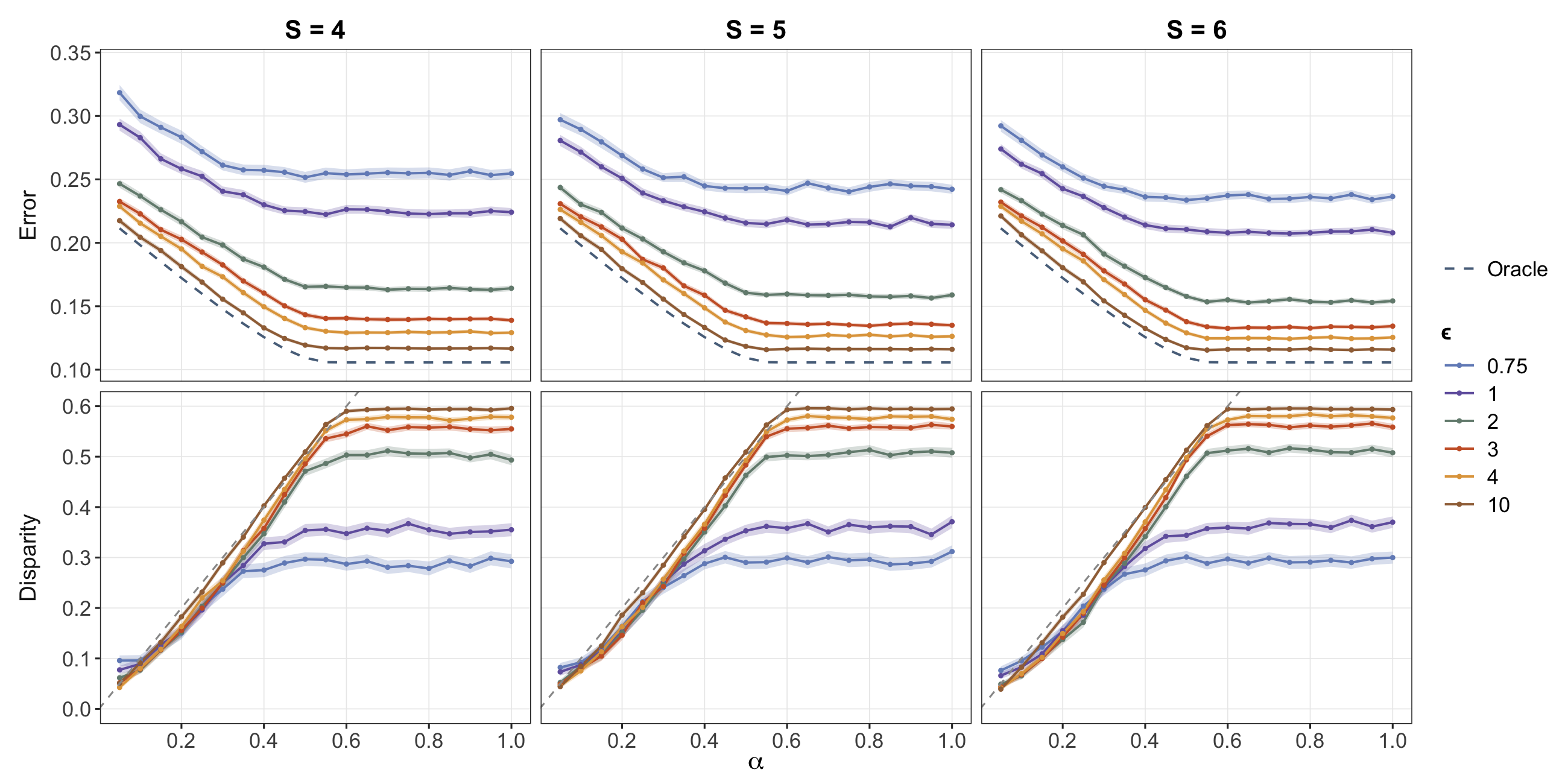}
     \caption{Means and $95\%$ confidence bands for misclassification errors and empirical disparities of \Cref{alg_fair_fdp} when $N_s = 2000$ and $S \in \{4,5,6\}$. The grey dashed line represents $y=x$. }
     \label{fig_fdp}
\end{figure}

\subsubsection{Effect of the number of servers \texorpdfstring{$S$}{}}
To further study the effect of the distributed nature of data across servers, we consider below the case when a fixed total of $N_{\text{total}} = 7200$ data points are equally distributed across $S \in \{1,2,3,4,5\}$ servers. In the case when $S=1$, we further apply CDP-Fair in \Cref{alg_fair_cdp}. Tuning parameters are selected following the same procedure as described in \Cref{sec_numerical_FDP} ($C_{\omega} = 0.1$, $\rho = 0.03$). After initial training, we use an independent sample of size $2500$ to test the performance of the trained classifier. The results are collected in \Cref{fig_effect_S}. 

The results in \Cref{fig_effect_S} highlight the effect of data decentralisation on both accuracy and disparity control. When $S=1$, FDP-Fair performs comparably to CDP-Fair, yielding nearly identical misclassification errors and empirical disparities across $\alpha$. As $S$ increases, i.e.~when the same amount of data is distributed across more servers, the misclassification error of FDP-Fair deteriorates. Intuitively, distributing a fixed total sample size $N_{\text{total}}$ across more servers reduces the local sample size per server, which increases estimation error at the server level. Moreover, aggregating output across servers also necessitates additional privacy-preserving communication, introducing extra noise. This phenomenon further echoes our theoretical results in \Cref{t_fdp_risk}.

\begin{figure}[!htbp]
    \centering
    \includegraphics[width=0.9\linewidth]{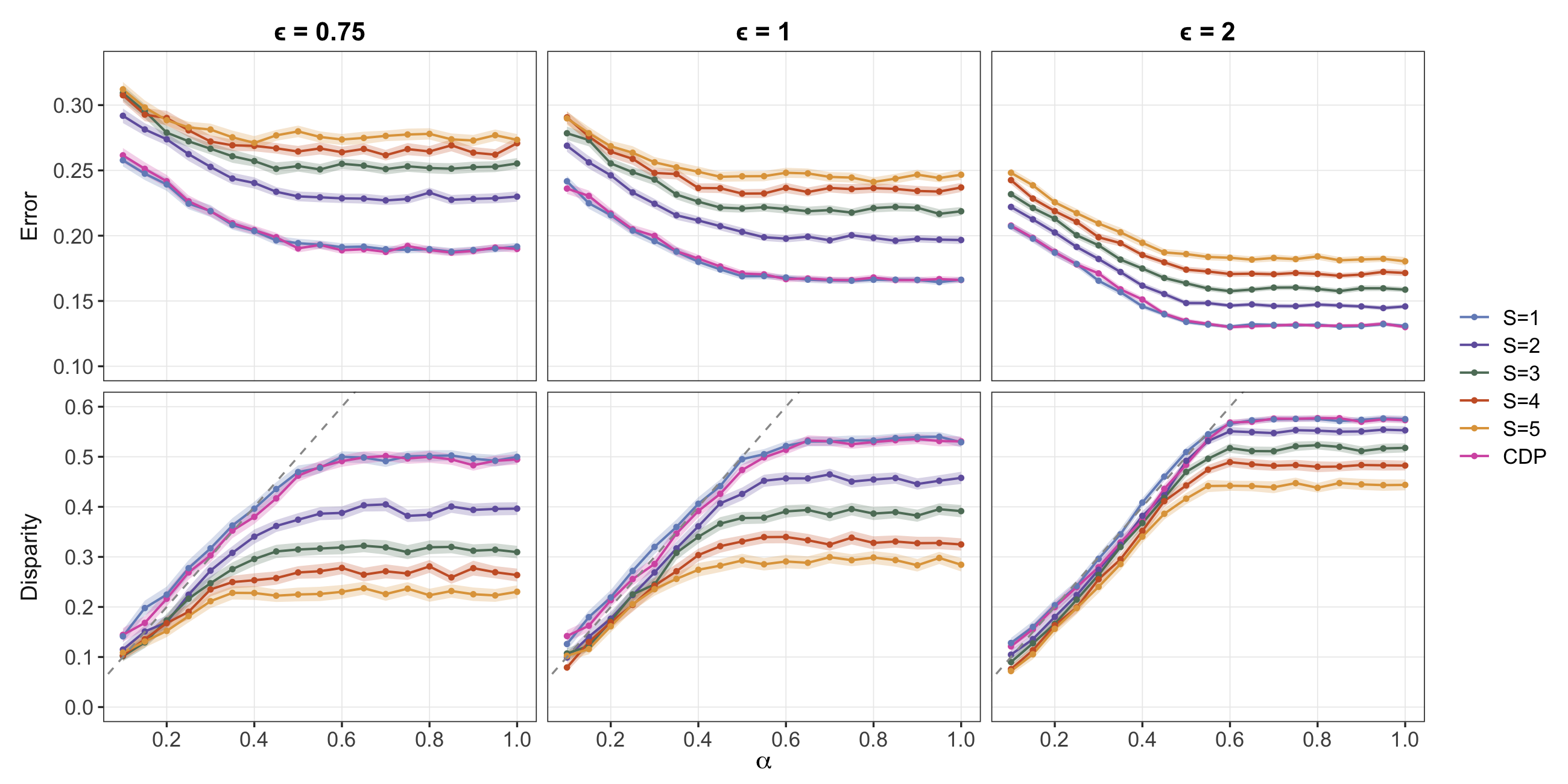}
    \caption{Means and $95\%$ confidence bands for misclassification errors and empirical disparities of Algorithms~\ref{alg_fair_fdp} and \ref{alg_fair_cdp} when $N_{\text{total}} = 7200$ and $S \in \{1,2,3,4,5\}$. The grey dashed line represents $y=x$.}
    \label{fig_effect_S}
\end{figure}

\subsection{Real data analysis}\label{sec_numerical_real}
In this section, we compare the performances of CDP-Fair and FDP-Fair with classifiers that satisfy either privacy and fairness constraints or none, on two widely used datasets: AdultCensus and Law school entrance datasets \citep[e.g.][]{zeng2024bayes}.

For the privacy-only baseline, we implement CDP-Fair with the fairness step \textbf{S2} removed. For the fairness-only baseline, we adopt the post-processing approach of \citet{zeng2024minimax}, using kernel density estimators to obtain the initial regression function estimates. As an unconstrained benchmark, we use a naive plug-in implementation of the group-wise Bayes classifier without privacy or fairness constraints. Finally, we implement CDP-Fair and FDP-Fair with $S=1$ following the same procedure detailed in \Cref{sec_numerical_simulated}, and set $C_{\omega}=0.1$ and $\rho=0.02$ throughout.

\subsubsection{AdultCensus dataset} \label{sec_numerical_adult}
With the AdultCensus dataset, we aim to classify whether an individual’s annual income exceeds \$50000. We select age, workclass, and education level as features $X$, and take gender as the protected attribute $A$. The full dataset contains $n_{\text{total}}=48842$ observations. In each experiment, we randomly selected $70\%$ of the data as the training set, and the rest $30\%$ is used as the test dataset. The experiment is repeated for $200$ times, and we report the mean and $95\%$ confidence band in \Cref{fig_adult}. To reflect the post-processing nature of CDP-Fair and FDP-Fair, all unfair classifiers are trained using only half of the available training samples.

As expected, imposing privacy constraints deteriorates classification accuracy. For a fixed fairness level $\alpha$, privacy-preserving classifiers incur larger misclassification errors than their non-private counterparts. As $\epsilon$ increases, the gap steadily narrows, and performance moves towards the non-private regime. In terms of disparity control, classifiers trained without any fairness constraints are substantially unfair, with empirical disparities remaining well above the prescribed tolerance for a wide range of small $\alpha$. By contrast, CDP-Fair and FDP-Fair successfully satisfy the fairness requirement, with empirical disparities controlled below the target level $\alpha$. Moreover, as $\epsilon$ increases, CDP-Fair and FDP-Fair converge to the fair non-private benchmark of \citet{zeng2024bayes}. Finally, the CDP and FDP implementations yield very similar performance, as evidenced by the close overlap of their curves across all panels.

\begin{figure}[!htbp]
    \centering
    \includegraphics[width=0.95\linewidth]{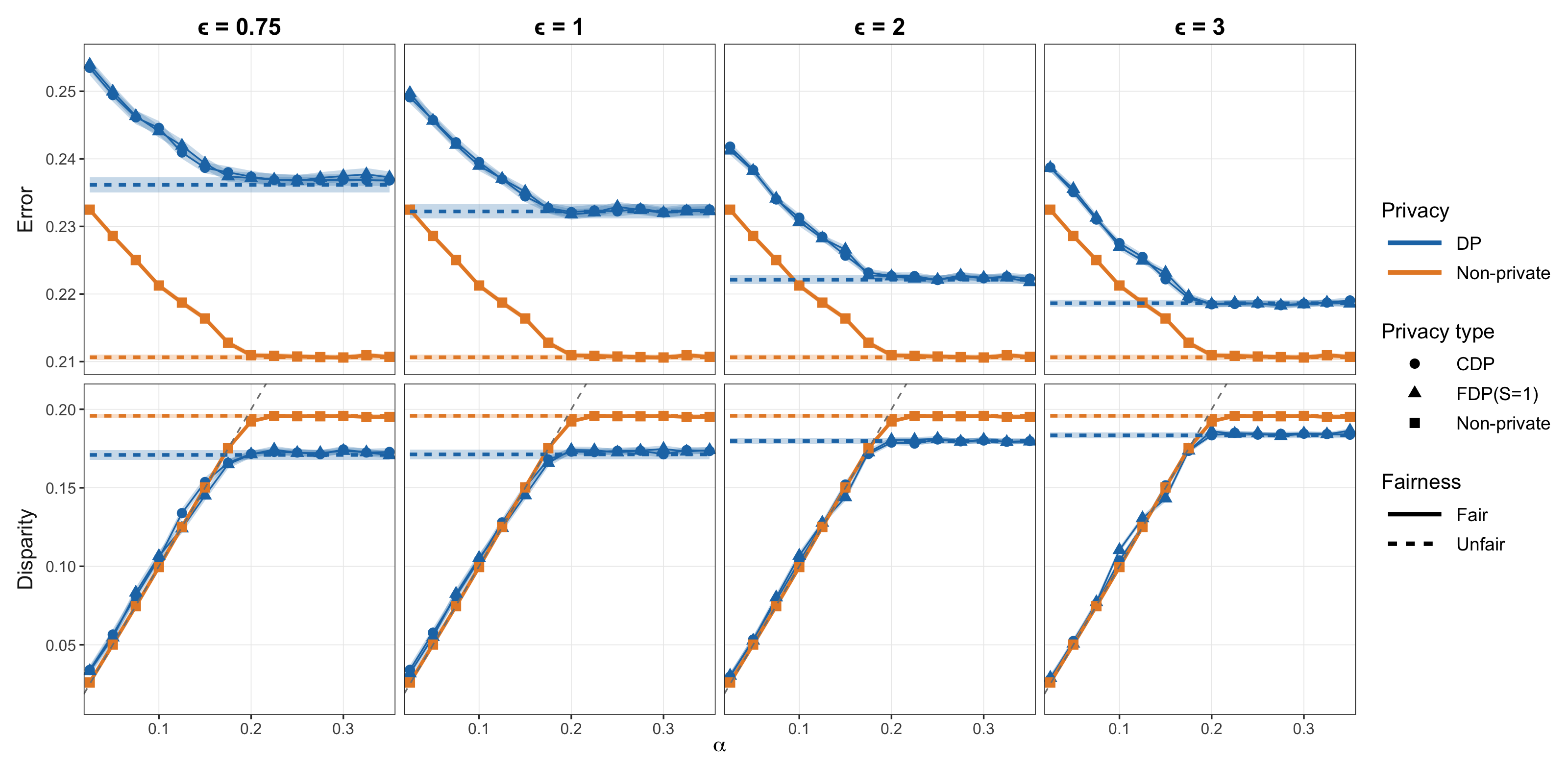}
    \caption{Means and $95\%$ confidence bands for misclassification errors and empirical disparities on the AdultCensus dataset.}
    \label{fig_adult}
\end{figure}

\subsubsection{Law school entrance dataset}
We further apply our algorithms to a selected subset of the Law school entrance dataset with $n_{\text{total}} = 40000$. We aim to classify if a student gets admitted to law school using the standard features $X$, such as LSAT scores, undergraduate GPA and more. The protected attribute $A$ is taken as race, i.e.~white vs non-white. Numerical results are summarised in \Cref{fig_law}. 

Similar trends as in \Cref{sec_numerical_adult} are observed. One difference compared with \Cref{fig_adult} is that in the disparity plots of \Cref{fig_law}, once the fairness constraint becomes inactive for the fair non-private benchmark, the empirical disparities of CDP-Fair and FDP-Fair continue to increase, though they still remain below the prespecified threshold. This gap diminishes as $\epsilon$ increases. We attribute this behaviour to the bias induced by the data-selection step. In the high-privacy regime, $\mathbb{P}(\widetilde{Y}=1\mid A=1)$ can be driven close to $0$, while $\mathbb{P}(\widetilde{Y}=1\mid A=0)$ becomes comparatively larger than in the non-private case, due to a higher frequency of $\widetilde{\eta}_0$ taking values close to $1$.

\begin{figure}[!htbp]
    \centering
    \includegraphics[width=0.95\linewidth]{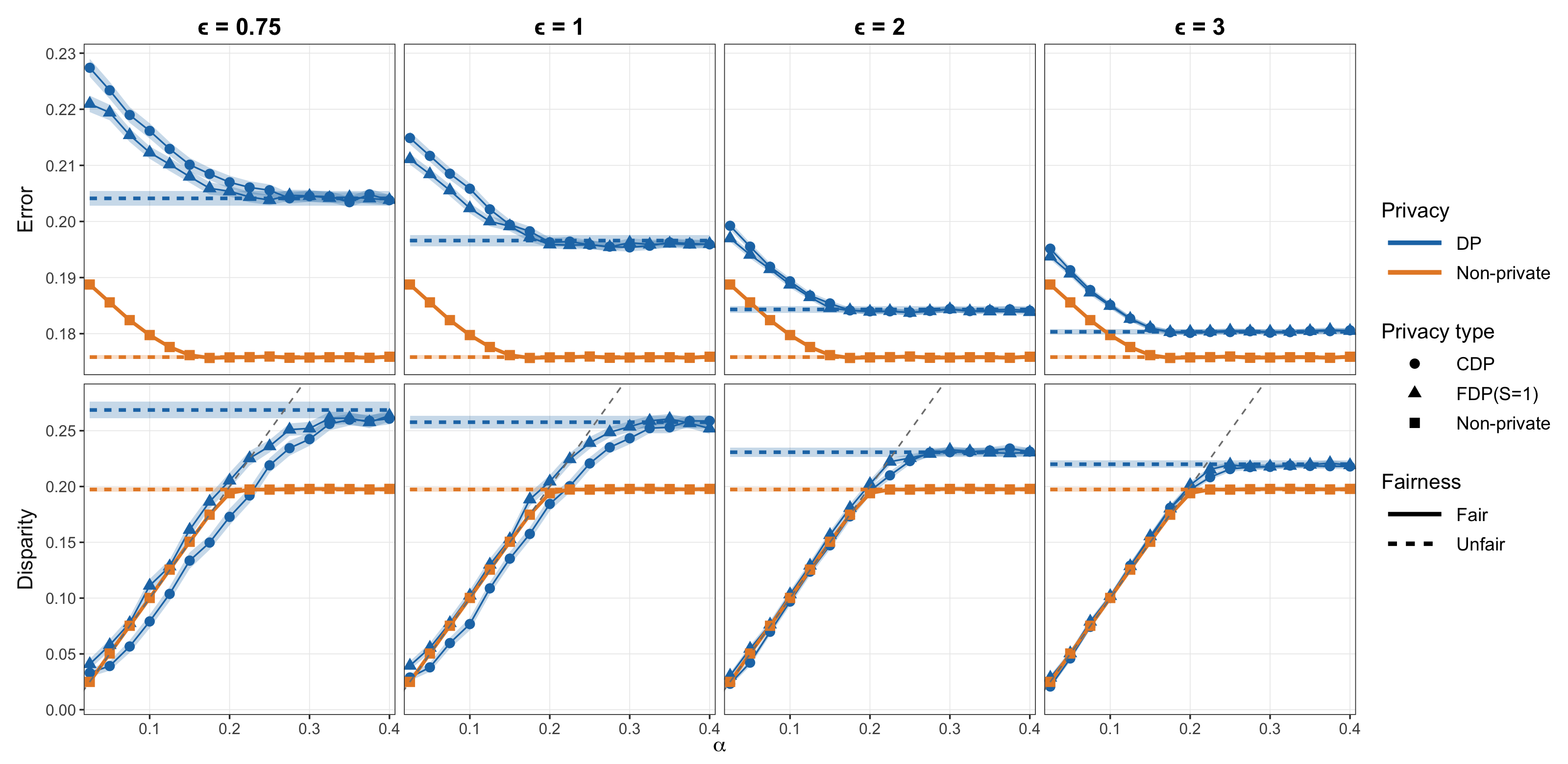}
    \caption{Means and $95\%$ confidence bands for misclassification errors and empirical disparities on the Law school entrance dataset.}
    \label{fig_law}
\end{figure}

\section{Conclusion} \label{sec_conclusion}
In this paper, we study private fairness-aware classification in a distributed setting. To the best of our knowledge, this is the first time seen in the literature. To handle the decentralised nature of data analysis, we propose a two-step plug-in classifier, FDP-fair, which is built upon Gaussian mechanisms and private binary trees. In the special case when there is only one server, we further provide a simplified yet effective algorithm, CDP-Fair. Under appropriate assumptions, theoretical guarantees for privacy, fairness and excess risk controls are provided, which are further supported by extensive numerical experiments on both synthetic and real~datasets.

We envisage several potential extensions. Firstly, our algorithms depend on the availability of sensitive features $A$, which may be restricted in certain practical settings due to privacy concerns. When sensitive features are available during training but not testing, one possible way is to predict $A$ using standard feature $X$ \citep[e.g.][]{zeng2024bayes,hou2024finite}. In the case when $A$ is available but subject to privacy constraints, one may instead adopt a label local DP framework \citep[e.g.][]{zhao2025theoretical}, under which only noisy versions of $A$ are released for use in the analysis. Secondly, in some scenarios, data from different servers may differ in distribution. Existing literature, e.g.~\citet{auddy2025minimax}, attempts to study differentially private transfer learning in non-parametric classifications. It would be interesting to further investigate the effect of data heterogeneity under fairness constraints.

\section*{Acknowledgements}
Xue is supported by EPSRC programme grant EP/Z531327/1. Yu is partially supported by the Philip Leverhulme Prize and EPSRC programme grant EP/Z531327/1.

\bibliographystyle{apalike}
\bibliography{reference}

@inproceedings{jagielski2019differentially,
  title={Differentially private fair learning},
  author={Jagielski, Matthew and Kearns, Michael and Mao, Jieming and Oprea, Alina and Roth, Aaron and Sharifi-Malvajerdi, Saeed and Ullman, Jonathan},
  booktitle={International Conference on Machine Learning},
  pages={3000--3008},
  year={2019},
  organization={PMLR}
}

@inproceedings{mangold2023differential,
  title={Differential privacy has bounded impact on fairness in classification},
  author={Mangold, Paul and Perrot, Micha{\"e}l and Bellet, Aur{\'e}lien and Tommasi, Marc},
  booktitle={International Conference on Machine Learning},
  pages={23681--23705},
  year={2023},
  organization={PMLR}
}

@inproceedings{
zhou2024differentially,
title={Differentially Private Worst-group Risk Minimization},
author={Xinyu Zhou and Raef Bassily},
booktitle={Forty-first International Conference on Machine Learning},
year={2024},
url={https://openreview.net/forum?id=ElNxZ40tBJ}
}

@inproceedings{ghoukasian2024differentially,
  title={Differentially private fair binary classifications},
  author={Ghoukasian, Hrad and Asoodeh, Shahab},
  booktitle={2024 IEEE International Symposium on Information Theory (ISIT)},
  pages={611--616},
  year={2024},
  organization={IEEE}
}

@article{zeng2024bayes,
  title={Bayes-optimal fair classification with linear disparity constraints via pre-, in-, and post-processing},
  author={Zeng, Xianli and Cheng, Guang and Dobriban, Edgar},
  journal={arXiv preprint arXiv:2402.02817},
  year={2024}
}

@article{massart1990tight,
  title={The tight constant in the Dvoretzky-Kiefer-Wolfowitz inequality},
  author={Massart, Pascal},
  journal={The annals of Probability},
  pages={1269--1283},
  year={1990},
  publisher={JSTOR}
}

@article{dvoretzky1956asymptotic,
  title={Asymptotic minimax character of the sample distribution function and of the classical multinomial estimator},
  author={Dvoretzky, Aryeh and Kiefer, Jack and Wolfowitz, Jacob},
  journal={The Annals of Mathematical Statistics},
  pages={642--669},
  year={1956},
  publisher={JSTOR}
}

@book{vershynin2018high,
  title={High-dimensional probability: An introduction with applications in data science},
  author={Vershynin, Roman},
  volume={47},
  year={2018},
  publisher={Cambridge university press}
}

@article{zeng2024minimax,
  title={Minimax Optimal Fair Classification with Bounded Demographic Disparity},
  author={Zeng, Xianli and Cheng, Guang and Dobriban, Edgar},
  journal={arXiv preprint arXiv:2403.18216},
  year={2024}
}

@article{audibert2007fast,
author = {Jean-Yves Audibert and Alexandre B. Tsybakov},
title = {{Fast learning rates for plug-in classifiers}},
volume = {35},
journal = {The Annals of Statistics},
number = {2},
publisher = {Institute of Mathematical Statistics},
pages = {608 -- 633},
year = {2007}
}

@article{rigollet2009optimal,
author = {Philippe Rigollet and R{\'e}gis Vert},
title = {{Optimal rates for plug-in estimators of density level sets}},
volume = {15},
journal = {Bernoulli},
number = {4},
publisher = {Bernoulli Society for Mathematical Statistics and Probability},
pages = {1154 -- 1178},
year = {2009}
}

@book{adler2007random,
  title={Random fields and geometry},
  author={Adler, Robert J and Taylor, Jonathan E},
  year={2007},
  publisher={Springer}
}

@article{hall2013differential,
  title={Differential privacy for functions and functional data},
  author={Hall, Rob and Rinaldo, Alessandro and Wasserman, Larry},
  journal={The Journal of Machine Learning Research},
  volume={14},
  number={1},
  pages={703--727},
  year={2013},
  publisher={JMLR. org}
}

@article{dwork2014algorithmic,
  title={The algorithmic foundations of differential privacy},
  author={Dwork, Cynthia and Roth, Aaron and others},
  journal={Foundations and trends{\textregistered} in theoretical computer science},
  volume={9},
  number={3--4},
  pages={211--407},
  year={2014},
  publisher={Now Publishers, Inc.}
}

@article{hu2025fairness,
  title={Fairness-aware Bayes optimal functional classification},
  author={Hu, Xiaoyu and Xue, Gengyu and Lin, Zhenhua and Yu, Yi},
  journal={arXiv preprint arXiv:2505.09471},
  year={2025}
}

@article{hou2024finite,
  title={Finite-sample and distribution-free fair classification: Optimal trade-off between excess risk and fairness, and the cost of group-blindness},
  author={Hou, Xiaotian and Zhang, Linjun},
  journal={arXiv preprint arXiv:2410.16477},
  year={2024}
}

@article{bousquet2002bennett,
  title={A Bennett concentration inequality and its application to suprema of empirical processes},
  author={Bousquet, Olivier},
  journal={Comptes Rendus Mathematique},
  volume={334},
  number={6},
  pages={495--500},
  year={2002},
  publisher={Elsevier}
}

@incollection{van1996weak,
  title={Weak convergence},
  author={Van Der Vaart, Aad W and Wellner, Jon A},
  booktitle={Weak convergence and empirical processes: with applications to statistics},
  pages={16--28},
  year={1996},
  publisher={Springer}
}

@article{hung2025optimal,
  title={Optimal Cox regression under federated differential privacy: coefficients and cumulative hazards},
  author={Hung, Elly KH and Yu, Yi},
  journal={arXiv preprint arXiv:2508.19640},
  year={2025}
}

@article{smith2021making,
  title={Making the most of parallel composition in differential privacy},
  author={Smith, Josh and Asghar, Hassan Jameel and Gioiosa, Gianpaolo and Mrabet, Sirine and Gaspers, Serge and Tyler, Paul},
  journal={arXiv preprint arXiv:2109.09078},
  year={2021}
}

@incollection{dudley2016vn,
  title={VN Sudakov’s work on expected suprema of Gaussian processes},
  author={Dudley, Richard M},
  booktitle={High Dimensional Probability VII: The Carg{\`e}se Volume},
  pages={37--43},
  year={2016},
  publisher={Springer}
}

@book{durrett2019probability,
  title={Probability: theory and examples},
  author={Durrett, Rick},
  volume={49},
  year={2019},
  publisher={Cambridge university press}
}

@inproceedings{dwork2006calibrating,
  title={Calibrating noise to sensitivity in private data analysis},
  author={Dwork, Cynthia and McSherry, Frank and Nissim, Kobbi and Smith, Adam},
  booktitle={Theory of cryptography conference},
  pages={265--284},
  year={2006},
  organization={Springer}
}

@article{yousefpour2021opacus,
  title={Opacus: User-friendly differential privacy library in PyTorch},
  author={Yousefpour, Ashkan and Shilov, Igor and Sablayrolles, Alexandre and Testuggine, Davide and Prasad, Karthik and Malek, Mani and Nguyen, John and Ghosh, Sayan and Bharadwaj, Akash and Zhao, Jessica and others},
  journal={arXiv preprint arXiv:2109.12298},
  year={2021}
}

@inproceedings{song2021evading,
  title={Evading the curse of dimensionality in unconstrained private glms},
  author={Song, Shuang and Steinke, Thomas and Thakkar, Om and Thakurta, Abhradeep},
  booktitle={International Conference on Artificial Intelligence and Statistics},
  pages={2638--2646},
  year={2021},
  organization={PMLR}
}

@misc{uscensus,
  author = {{United States Census Bureau}},
  title = {Census Bureau Sets Key Parameters to Protect Privacy in 2020 Census Results},
  year = 2021,
  howpublished = {\url{https://www.census.gov/newsroom/press-releases/2021/2020-census-key-parameters.html}}
}

@article{karwa2017finite,
  title={Finite sample differentially private confidence intervals},
  author={Karwa, Vishesh and Vadhan, Salil},
  journal={arXiv preprint arXiv:1711.03908},
  year={2017}
}

@article{butucealocal2020,
author = {Cristina Butucea and Amandine Dubois and Martin Kroll and Adrien Saumard},
title = {{Local differential privacy: Elbow effect in optimal density estimation and adaptation over Besov ellipsoids}},
volume = {26},
journal = {Bernoulli},
number = {3},
pages = {1727 -- 1764},
year = {2020},
doi = {10.3150/19-BEJ1165}
}

@article{cai2024optimal,
  title={Optimal federated learning for nonparametric regression with heterogeneous distributed differential privacy constraints},
  author={Cai, T Tony and Chakraborty, Abhinav and Vuursteen, Lasse},
  journal={arXiv preprint arXiv:2406.06755},
  year={2024}
}

@incollection{angwin2022machine,
  title={Machine bias},
  author={Angwin, Julia and Larson, Jeff and Mattu, Surya and Kirchner, Lauren},
  booktitle={Ethics of data and analytics},
  pages={254--264},
  year={2022},
  publisher={Auerbach Publications}
}

@inproceedings{quinonero2023disentangling,
  title={Disentangling and operationalizing ai fairness at linkedin},
  author={Qui{\~n}onero Candela, Joaquin and Wu, Yuwen and Hsu, Brian and Jain, Sakshi and Ramos, Jennifer and Adams, Jon and Hallman, Robert and Basu, Kinjal},
  booktitle={Proceedings of the 2023 ACM Conference on Fairness, Accountability, and Transparency},
  pages={1213--1228},
  year={2023}
}

@article{bakalar2021fairness,
  title={Fairness on the ground: Applying algorithmic fairness approaches to production systems},
  author={Bakalar, Chlo{\'e} and Barreto, Renata and Bergman, Stevie and Bogen, Miranda and Chern, Bobbie and Corbett-Davies, Sam and Hall, Melissa and Kloumann, Isabel and Lam, Michelle and Candela, Joaquin Qui{\~n}onero and others},
  journal={arXiv preprint arXiv:2103.06172},
  year={2021}
}

@article{pessach2022review,
  title={A review on fairness in machine learning},
  author={Pessach, Dana and Shmueli, Erez},
  journal={ACM Computing Surveys (CSUR)},
  volume={55},
  number={3},
  pages={1--44},
  year={2022},
  publisher={ACM New York, NY}
}

@article{calmon2017optimized,
  title={Optimized pre-processing for discrimination prevention},
  author={Calmon, Flavio and Wei, Dennis and Vinzamuri, Bhanukiran and Natesan Ramamurthy, Karthikeyan and Varshney, Kush R},
  journal={Advances in Neural Information Processing Systems},
  volume={30},
  year={2017}
}

@article{johndrow2019algorithm,
  title={An algorithm for removing sensitive information},
  author={Johndrow, James E and Lum, Kristian},
  journal={The Annals of Applied Statistics},
  volume={13},
  number={1},
  pages={189--220},
  year={2019},
  publisher={JSTOR}
}

@article{cho2020fair,
  title={A fair classifier using kernel density estimation},
  author={Cho, Jaewoong and Hwang, Gyeongjo and Suh, Changho},
  journal={Advances in Neural Information Processing Systems},
  volume={33},
  pages={15088--15099},
  year={2020}
}

@inproceedings{celis2019classification,
  title={Classification with fairness constraints: A meta-algorithm with provable guarantees},
  author={Celis, L Elisa and Huang, Lingxiao and Keswani, Vijay and Vishnoi, Nisheeth K},
  booktitle={Proceedings of the conference on fairness, accountability, and transparency},
  pages={319--328},
  year={2019}
}

@inproceedings{martinez2020minimax,
  title={Minimax pareto fairness: A multi objective perspective},
  author={Martinez, Natalia and Bertran, Martin and Sapiro, Guillermo},
  booktitle={International conference on machine learning},
  pages={6755--6764},
  year={2020},
  organization={PMLR}
}

@article{bagdasaryan2019differential,
  title={Differential privacy has disparate impact on model accuracy},
  author={Bagdasaryan, Eugene and Poursaeed, Omid and Shmatikov, Vitaly},
  journal={Advances in neural information processing systems},
  volume={32},
  year={2019}
}

@inproceedings{ganev2022robin,
  title={Robin hood and matthew effects: Differential privacy has disparate impact on synthetic data},
  author={Ganev, Georgi and Oprisanu, Bristena and De Cristofaro, Emiliano},
  booktitle={International Conference on Machine Learning},
  pages={6944--6959},
  year={2022},
  organization={PMLR}
}

@article{say2025fairness,
  title={Fairness Meets Privacy: Integrating Differential Privacy and Demographic Parity in Multi-class Classification},
  author={Say, Lilian and Denis, Christophe and Pinot, Rafael},
  journal={arXiv preprint arXiv:2511.18876},
  year={2025}
}

@article{denis2024fairness,
  title={Fairness guarantees in multi-class classification with demographic parity},
  author={Denis, Christophe and Elie, Romuald and Hebiri, Mohamed and Hu, Fran{\c{c}}ois},
  journal={Journal of Machine Learning Research},
  volume={25},
  number={130},
  pages={1--46},
  year={2024}
}

@article{li2020multi,
  title={Multi-site fMRI analysis using privacy-preserving federated learning and domain adaptation: ABIDE results},
  author={Li, Xiaoxiao and Gu, Yufeng and Dvornek, Nicha and Staib, Lawrence H and Ventola, Pamela and Duncan, James S},
  journal={Medical image analysis},
  volume={65},
  pages={101765},
  year={2020},
  publisher={Elsevier}
}

@article{heyndrickx2023melloddy,
  title={Melloddy: Cross-pharma federated learning at unprecedented scale unlocks benefits in qsar without compromising proprietary information},
  author={Heyndrickx, Wouter and Mervin, Lewis and Morawietz, Tobias and Sturm, No{\'e} and Friedrich, Lukas and Zalewski, Adam and Pentina, Anastasia and Humbeck, Lina and Oldenhof, Martijn and Niwayama, Ritsuya and others},
  journal={Journal of chemical information and modeling},
  volume={64},
  number={7},
  pages={2331--2344},
  year={2023},
  publisher={ACS Publications}
}

@article{auddy2025minimax,
  title={Minimax and adaptive transfer learning for nonparametric classification under distributed differential privacy constraints},
  author={Auddy, Arnab and Cai, T Tony and Chakraborty, Abhinav},
  journal={Journal of the Royal Statistical Society Series B: Statistical Methodology},
  pages={qkaf070},
  year={2025},
  publisher={Oxford University Press UK}
}

@article{li2024federated,
  title={Federated transfer learning with differential privacy},
  author={Li, Mengchu and Tian, Ye and Feng, Yang and Yu, Yi},
  journal={arXiv preprint arXiv:2403.11343},
  year={2024}
}

@article{xue2024optimal,
  title={Optimal estimation in private distributed functional data analysis},
  author={Xue, Gengyu and Lin, Zhenhua and Yu, Yi},
  journal={arXiv preprint arXiv:2412.06582},
  year={2024}
}

@article{madrid2023change,
  title={Change point detection and inference in multivariate non-parametric models under mixing conditions},
  author={Madrid Padilla, Carlos Misael and Xu, Haotian and Wang, Daren and MADRID PADILLA, OSCAR HERNAN and Yu, Yi},
  journal={Advances in Neural Information Processing Systems},
  volume={36},
  pages={21081--21134},
  year={2023}
}

@inproceedings{kim2019uniform,
  title={Uniform convergence rate of the kernel density estimator adaptive to intrinsic volume dimension},
  author={Kim, Jisu and Shin, Jaehyeok and Rinaldo, Alessandro and Wasserman, Larry},
  booktitle={International Conference on Machine Learning},
  pages={3398--3407},
  year={2019},
  organization={PMLR}
}

@article{cai2021cost,
  title={The cost of privacy: Optimal rates of convergence for parameter estimation with differential privacy},
  author={Cai, T Tony and Wang, Yichen and Zhang, Linjun},
  journal={The Annals of Statistics},
  volume={49},
  number={5},
  pages={2825--2850},
  year={2021},
  publisher={Institute of Mathematical Statistics}
}

@article{zhao2025theoretical,
  title={On Theoretical Limits of Learning with Label Differential Privacy},
  author={Zhao, Puning and Ma, Chuan and Shen, Li and Wang, Shaowei and Fan, Rongfei},
  journal={arXiv preprint arXiv:2502.14309},
  year={2025}
}

\newpage
\appendix 
\section*{Appendices}
All technical details and additional numerical results are collected in the Appendices. Additional numerical experiments are presented in \Cref{sec_appendix_numerical}. We provide all proofs related to Algorithms~\ref{alg_fair_cdp} and \ref{alg_fair_fdp} in Appendices~\ref{sec_appendix_cdp} and \ref{sec_appendix_fdp}, respectively. For completeness, additional background results are provided in \Cref{sec_appendix_additional_background}.

Throughout the appendix, with a slight abuse of notation, unless specifically stated otherwise, let $C_1, C_2, \ldots >0$ denote constants whose values may vary from place to place and depend on $d$. Let $\widehat{\cdot}_s$ denote standard empirical estimators without privacy constraints for site $s$, $s \in [S]$, and $\widehat{\cdot}$ the aggregated one. Similarly, let $\widetilde{\cdot}_s$ denote the site-wise estimators that satisfy privacy constraints, and let $\widetilde{\cdot}$ denote the aggregated one in the central server.

For notational simplicity, throughout the appendix, we write $\widetilde{\tau} = \widetilde{\tau}_{\DD,\alpha}$.

\section{Additional numerical experiments} \label{sec_appendix_numerical}
\subsection{Additional algorithm} \label{sec_appendix_non_increasing_numeircal}
To ensure the feasibility of the \Cref{alg_non_increasing} in the high privacy regime when the DD curve fluctuates significantly, we adjust \Cref{alg_non_increasing} to \Cref{alg_non_increasing_simulation} detailed below in practice.

\begin{algorithm}[!htbp]
\caption{\textsc{Non.increasing (numerical)}($\{g_i\}_{i \in 2^M+1}$, $\{\mu_s\}_{s\in[S]}$, $C_{\omega}$, $M$, $\eta$)}
\label{alg_non_increasing_simulation}
\begin{algorithmic}
  \Require Sequence of number  $\{g_i\}_{i \in 2^M+1}$, weight $\{\mu_s\}_{s\in [S]}$, constants $C_{\omega}$, binary tree layer $M$, tolerance $\eta$.

  \State Set the tolerance level 
    \begin{align*}
        \omega = C_{\omega} \sqrt{\sum_{s=1}^S\frac{\mu_s^2M^4\log(1/\delta_s)\log(M/\eta)}{\breve{n}_s^2\epsilon_s^2}}.
    \end{align*}
   \If{$g_i \geq \ldots \geq g_{2^M+1}$}
   \State \textbf{Break} and return $\{g_i\}_{i \in [2^M+1]}$.
   \Else
   \State Initialize $f_n = \min\{g_n+\omega,\,1\}$.
   \For{$i = n-1, n-2, \dots, 1$}
   \State Set $f_i \gets \min\bigl(\min\{g_i+\omega,\,1\},\; \max(\max\{g_i-\omega,\,0\},\; f_{i+1})\bigr)$.
   \EndFor
   \EndIf
   \Ensure $\{f_i\}_{i \in2^N+1}$.
\end{algorithmic}
\end{algorithm}
\subsection{Sensitivity analysis} \label{sec_appendix_sensitivity}
Under the same setup detailed in \Cref{sec_numerical_FDP}, we carry out sensitivity analysis for tuning parameters $C_{\omega}$ and $\rho$ in a setting where $S= 3$, $n_s = 2000$, $n_{\text{test}} = 2000$ and $\alpha =0.3$. Numerical experiments are collected in Figures~\ref{fig_sensitivity_rho} and~\ref{fig_sensitivity_Cw} below.

As detailed in \Cref{alg_fair_fdp}, $\rho$ determines the tolerance band used in selecting $\widetilde{\tau}$ through the optimisation problem $\widetilde{\tau} = \argmin_{\tau \in \mathcal{G}}\{|\tau|: |\widetilde{\DD}_{\downarrow}(\tau)| \in [\alpha-\rho, \alpha+\rho]\}$.
Since $\widetilde{\DD}{\downarrow}(\cdot)$ is, with high probability, a monotonically non-increasing function, choosing the feasible $\tau\in\mathcal{G}$ with the smallest magnitude tends to pick a threshold for which $\widetilde{\DD}{\downarrow}(\widetilde{\tau})$ lies near the upper end of the admissible band, i.e.~closer to $\alpha+\rho$. Consequently, a larger $\rho$ typically leads to a larger disparity of $\widetilde{f}$, potentially exceeding $\alpha$, yet remaining below $\alpha+\rho$. This behaviour is precisely illustrated in \Cref{fig_sensitivity_rho}.

In \Cref{alg_non_increasing_simulation}, $C_\omega$ controls the bias incurred when constructing a sequence that is as close to non-increasing as possible. Based on the numerical results in \Cref{fig_sensitivity_Cw}, the performance is not particularly sensitive to $C_\omega$ in terms of either misclassification error or empirical disparity. The slightly upward trend in the disparity plot is mainly because a larger $C_\omega$ increases the possibilities that the adjusted sequence is non-increasing, which can in turn lead to a slightly more accurate choice of $\widetilde{\tau}$. However, the corresponding effect on misclassification error is less significant.

\begin{figure}[!htbp]
    \centering
    \includegraphics[width=0.9\linewidth]{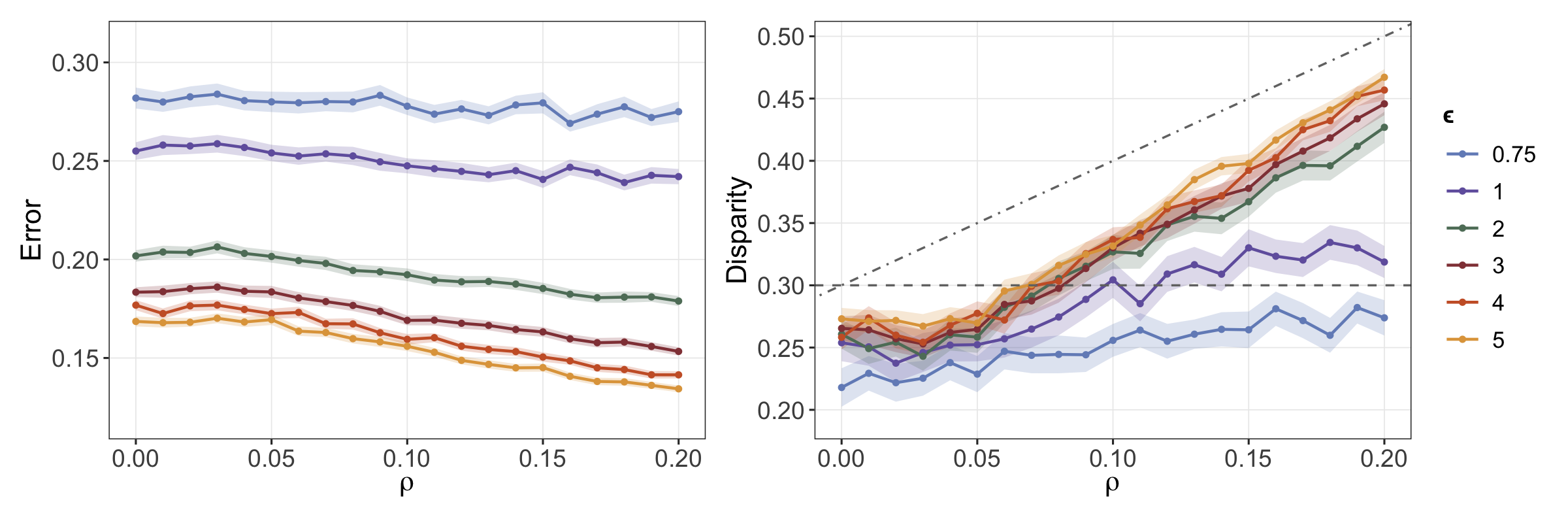}
    \caption{Sensitivity analysis for $\rho$. The grey dotted lines are $y =0.3$ and $y = 0.3+\rho$, respectively. }
    \label{fig_sensitivity_rho}
\end{figure}

\begin{figure}[!htbp]
    \centering
    \includegraphics[width=0.9\linewidth]{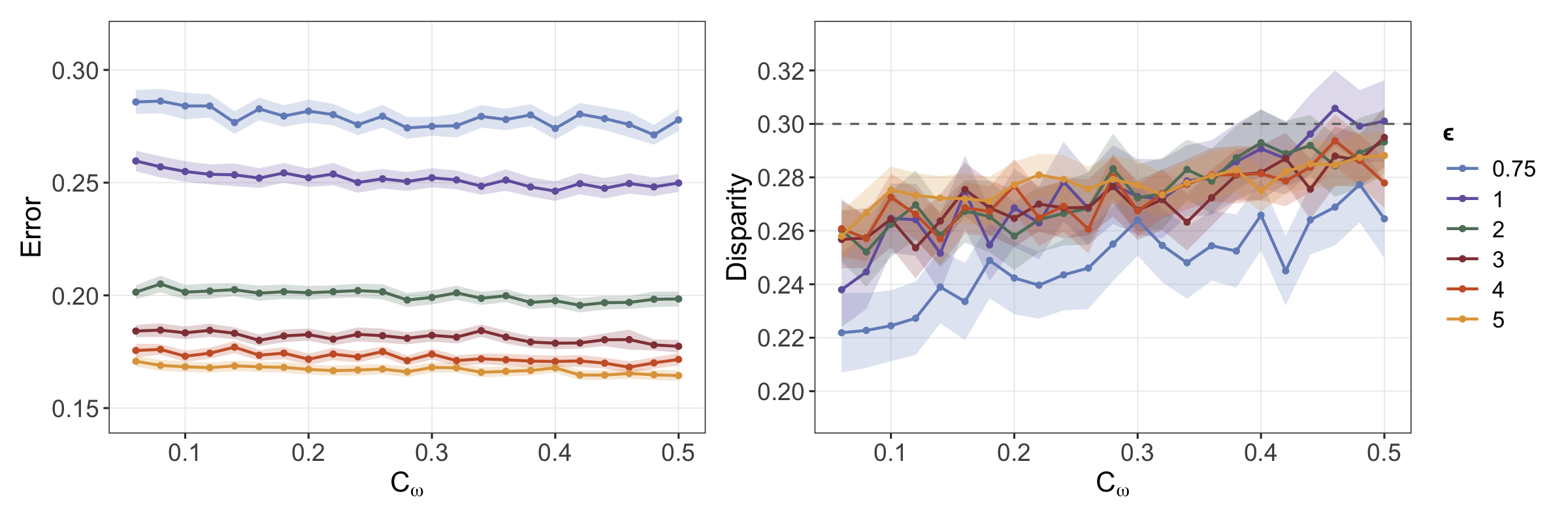}
    \caption{Sensitivity analysis for $C_{\omega}$. The grey dotted line represents $y =0.3$. }
    \label{fig_sensitivity_Cw}
\end{figure}

\section{Theoretical guarantee for Algorithm \ref{alg_fair_cdp}} \label{sec_appendix_cdp}
\subsection{Privacy guarantee}
\begin{proposition} \label{t_cdp_privacy_guarantee}
    \Cref{alg_fair_cdp} is $(\epsilon,\delta)$-CDP.
\end{proposition}
\begin{proof}
    We divide the proof into two steps, with the first step quantifying the sensitivity parameters used in Lemmas \ref{l_GM_univariate} and \ref{l_GM_func}, and the second step involving DP compositions. 
    
    \noindent \textbf{Step 1: Upper bound on sensitivity}
        In this step, we upper bound $\ell_2$ sensitivities for terms of interests in the calculation of $\widetilde{\pi}_a$, $\widetilde{p}_{X|A=a}$, $\widetilde{\eta}_a$ and $\widetilde{\DD}$. Without loss of generality, we consider the case when $A_1=a=1$ and $Y_1 =1$.  We further assume that only the first data point in the dataset changes i.e.~$\mathcal{D} = \{(X_i, A_i, Y_i)\}_{i \in [n]}$, $\mathcal{D}' = (X_1', A_1', Y_1') \cup  \{(X_i, A_i, Y_i)\}_{i \in [n]\backslash \{1\}}$. Similar notation is considered for the calibration data $\breve{\mathcal{D}}$ and $\breve{\mathcal{D}}'$. Five general cases are considered below, and other cases can be justified similarly:
    \begin{itemize}
        \item[] Case 1: $A_1 = A_1'$, $X_1 = X_1'$ and $Y_1 \neq Y_1'$;

        \item[] Case 2: $A_1 = A_1'$, $X_1 \neq X_1'$ and $Y_1 = Y_1'$;

        \item[] Case 3: $A_1 = A_1'$, $X_1 \neq X_1'$ and $Y_1 \neq Y_1'$;

        \item[] Case 4: $A_1 \neq A_1'$, $(X_1, Y_1) =  (X_1', Y_1')$;

        \item[] Case 5: $A_1 \neq A_1'$, $(X_1, Y_1) \neq (X_1', Y_1')$.
    \end{itemize} 

    \begin{itemize}
        \item For $\widetilde{\pi}_a$: 
        \begin{enumerate}
            \item Cases 1-3: Since $A_1 = A_1'$ in these cases, we have that $n_{a,0}+n_{a,1} = n'_{a,0}+n'_{a,1}$. Thus,
            \begin{align*}
                \Delta = \Big|\frac{n_{a,0}+n_{a,1}}{n} - \frac{n'_{a,0}+n'_{a,1}}{n}\Big| = 0.
            \end{align*}

            \item Cases 4 and 5: Since $A_1 \neq A_1'$ in these cases, we have that $|n_{a,0}+n_{a,1}- (n'_{a,0}+n'_{a,1})| =1$. Thus, 
            \begin{align*}
                \Delta = \Big|\frac{n_{a,0}+n_{a,1}}{n} - \frac{n'_{a,0}+n'_{a,1}}{n}\Big| = \frac{1}{n}.
            \end{align*}
        \end{enumerate}

        \item  For $\widetilde{p}_{X|A=a}$: The proof follows from a similar argument as the one on \citet{auddy2025minimax}. Denote $\|\cdot\|_{\mathcal{K}}$ the RKHS norm of the space given by linear combination $\{\sum_{i}\theta_iK\{(X_i-\cdot)/h\}: \theta_i \in \mathbb{R}\}$. We also rewrite
        \[\frac{1}{n_a}\sum_{y \in \{0,1\}}\sum_{i=1}^{n_{a,y}}K_{h}(X_{a,y}^i -x) = \frac{1}{n_a} \sum_{i=1}^n K_{h}(X_i -x)\indc\{A_i=a\}.\]
        \begin{enumerate}
            \item Cases 1-3: Since $A_1 = A_1'$ in these cases, we have that 
            \begin{align*}
                \Delta_{\mathcal{K}} =\;& \Big\|\frac{1}{n_a} \sum_{i=1}^n K_{h}(X_i -\cdot)\indc\{A_i=a\}-\frac{1}{n_a'} \sum_{i=1}^{n} K_{h}(X'_i -\cdot)\indc\{A'_i=a\}\Big\|_{\mathcal{K}}\\
                = \;& \frac{1}{n_a}\Big\|K_{h}(X_1 -\cdot) - K_{h}(X_1' -\cdot)\Big\|_{\mathcal{K}} \leq  \frac{2\sqrt{C_K}}{n_ah^d},
            \end{align*}
            the last inequality follows from \Cref{a_kernel}\ref{a_kernel_bounded}.

            \item Cases 4 and 5: Since $A_1 \neq A_1'$ in these two cases, we have that $n_a' = n_a -1$ and 
            \begin{align*}
                \Delta_{\mathcal{K}} = \;& \Big\|\frac{1}{n_a} \sum_{i=1}^n K_{h}(X_i -\cdot)\indc\{A_i=a\}-\frac{1}{n_a-1} \sum_{i=1}^{n} K_{h}(X'_i -\cdot)\indc\{A'_i=a\}\Big\|_{\mathcal{K}}\\
                \leq \;&  \frac{1}{n_a-1}\Big(1-\frac{n_a-1}{n_a}\Big)\Big\|\sum_{i=2}^n K_{h}(X_i -\cdot)\indc\{A_i=a\}\Big\|_{\mathcal{K}}\\
                &+ \frac{1}{n_a}\Big\|K_{h}(X_1 -\cdot)\Big\|_{\mathcal{K}}\\
                \leq \;& \frac{2\sqrt{C_K}}{n_ah^d}. 
            \end{align*}            
        \end{enumerate}

        \item For $\widetilde{\eta}_a$: The proof is very similar to the proof of $\widetilde{p}_{X|A=a}$. Rewrite
        \begin{align*}
            \frac{1}{n_a}\sum_{i=1}^{n_{a,1}} K_{h}(X_{a,1}^i -x) = \frac{1}{n_a }\sum_{i=1}^n
            Y_iK_h(X_i -x)\indc\{A_i=1\}.
        \end{align*}
        \begin{enumerate}
            \item Cases 1 and 3: In these two cases, we have that $A_1 =A_1' =1$ and $1= Y_1 \neq Y_1' =0$. Thus, 
            \begin{align*}
                \Delta_{\mathcal{K}} =\;& \Big\|\frac{1}{n_a}\sum_{i=1}^n
            Y_iK_h(X_i -\cdot)\indc\{A_i=1\} -\frac{1}{n_a}\sum_{i=1}^n
            Y'_iK_h(X'_i -\cdot)\indc\{A'_i=1\}\Big\|_{\mathcal{K}}\\
            = \;& \frac{1}{n_a}\Big\|Y_1K_h(X_1 -\cdot)\Big\|_{\mathcal{K}} \leq \frac{\sqrt{C_K}}{n_ah^d}.
            \end{align*}

            \item Case 2: In this case, we have that $A_1 =A_1' =1$ and $Y_1 = Y_1' =1$. Thus,
            \begin{align*}
                \Delta_{\mathcal{K}} = \frac{1}{n_a}\Big\|Y_1K_h(X_1 -\cdot)- Y'_1K_h(X'_1 -\cdot)\Big\|_{\mathcal{K}} \leq \frac{2\sqrt{C_K}}{n_ah^d}.
            \end{align*}

            \item Cases 4 and 5:  we have that $1= A_1 \neq A_1' =0$. Thus,
            \begin{align*}
                \Delta_{\mathcal{K}} = \;& \Big\|\frac{1}{n_a}\sum_{i=1}^n
            Y_iK_h(X_i -\cdot)\indc\{A_i=1\} -\frac{1}{n_a-1}\sum_{i=1}^n
            Y'_iK_h(X'_i -\cdot)\indc\{A'_i=1\}\Big\|_{\mathcal{K}}\\
            =\;& \frac{1}{n_a(n_a-1)}\Big\|\sum_{i=2}^n
            Y_iK_h(X_i -\cdot)\indc\{A_i=1\}\Big\|_{\mathcal{K}} + \frac{1}{n_a}\Big\|Y_1K_h(X_1 -\cdot)\Big\|_{\mathcal{K}}\\
            \leq\;& \frac{2\sqrt{C_K}}{n_ah^d}.
            \end{align*}
        \end{enumerate}

        \item For $\widetilde{\DD}(\tau)$: Rewrite
        \begin{align*}
            &\frac{1}{\breve{n}_1}\sum_{y\in\{0,1\}}\sum_{i=1}^{\breve{n}_{1,y}}\indc\Big\{\widetilde{\eta}_1(\breve{X}^i_{1,y}) \geq \frac{1}{2}+\frac{\tau}{2\widetilde{\pi}_1}\Big\} -\frac{1}{\breve{n}_0}\sum_{y\in\{0,1\}}\sum_{i=1}^{\breve{n}_{0,y}}\indc\Big\{\widetilde{\eta}_0(\breve{X}^i_{0,y}) \geq \frac{1}{2}-\frac{\tau}{2\widetilde{\pi}_0}\Big\}\\
            =\;& \frac{1}{\breve{n}_1}\sum_{i=1}^{\breve{n}}\indc\Big\{\widetilde{\eta}_1(\breve{X}_i) \geq \frac{1}{2}+\frac{\tau}{2\widetilde{\pi}_1}\Big\} \indc\{A_i=1\} -\frac{1}{\breve{n}_0}\sum_{i=1}^{\breve{n}}\indc\Big\{\widetilde{\eta}_0(\breve{X}_i) \geq \frac{1}{2}-\frac{\tau}{2\widetilde{\pi}_0}\Big\}\indc\{A_i=0\}.
        \end{align*}
        \begin{enumerate}
            \item Cases 1-3: In these three cases, we have that $A_1 = A_1' =1$, therefore, it holds that 
            \begin{align*}
                \Delta = \frac{1}{\breve{n}_1}\Big|\indc\Big\{\widetilde{\eta}_1(\breve{X}_1) \geq \frac{1}{2}+\frac{\tau}{2\widetilde{\pi}_1}\Big\} - \indc\Big\{\widetilde{\eta}_1(\breve{X}_1') \geq \frac{1}{2}+\frac{\tau}{2\widetilde{\pi}_1}\Big\}\Big| \leq \frac{1}{\breve{n}_1}.
            \end{align*}

            \item  Cases 4-6: In these three cases, we have that $1 =A_1 \neq A_1' =0$. Therefore, we have that $\breve{n}_1' = \breve{n}_1-1$ and $\breve{n}_0' = \breve{n}_0+1$. Thus, 
            \begin{align*}
                \Delta \leq\;& \Big|\frac{1}{\breve{n}_1(\breve{n}_1-1)}\sum_{i=2}^n\indc\Big\{\widetilde{\eta}_1(\breve{X}_i) \geq \frac{1}{2}+\frac{\tau}{2\widetilde{\pi}_1}\Big\} \indc\{A_i=1\}\\
                & - \frac{1}{\breve{n}_0(\breve{n}_0+1)}\Big|\sum_{i=2}^{n}\indc\Big\{\widetilde{\eta}_0(\breve{X}_i) \geq \frac{1}{2}-\frac{\tau}{2\widetilde{\pi}_0}\Big\}\indc\{A_i=0\}\Big|\\
                & + \Big|\frac{1}{\breve{n}_1}\indc\Big\{\widetilde{\eta}_1(\breve{X}_1) \geq \frac{1}{2}+\frac{\tau}{2\widetilde{\pi}_1}\Big\} - \frac{1}{\breve{n}_0+1}\indc\Big\{\widetilde{\eta}_0(\breve{X}'_1) \geq \frac{1}{2}-\frac{\tau}{2\widetilde{\pi}_0}\Big\}\Big|\\
                \leq \;& \frac{1}{\breve{n}_1} + \frac{1}{\breve{n}_0+1} \leq \frac{2}{\breve{n}_0 \wedge \breve{n}_1}.
            \end{align*}
        \end{enumerate}
        
    \end{itemize}

    \noindent \textbf{Step 2: DP guarantee.} In this step we will show that the final output is $(\epsilon,\delta)$-DP.
    
    \noindent \textbf{Step 2-1: Privacy guarantee for each estimator.} By \textbf{Step 1} and Lemmas \ref{l_GM_univariate} and \ref{l_GM_func}, it holds that $\widetilde{\pi}_0$ and $\widetilde{\pi}_1$ are $(\epsilon/4,\delta/4)$-DP each; $\widetilde{p}_{X|A=0}$ and $\widetilde{p}_{X|A=1}$ are $(\epsilon/4,\delta/4)$-DP each; $\widetilde{\eta}_0$ and $\widetilde{\eta}_1$ are $(\epsilon/4,\delta/4)$-DP each; and $\widetilde{\DD}(\tau)$ is $(\epsilon,\delta)$-DP for a given $\tau \in \mathbb{R}$. To show the privacy guarantee for $\widetilde{\tau}$, note that for any $\tau \geq 0$, we have that
    \begin{align*}
         \mathbb{P}(\widetilde{\tau} \leq \tau|\mathcal{D}, \breve{\mathcal{D}}) \leq \mathbb{P}(\widetilde{\DD}(\tau) \leq \alpha|\mathcal{D}, \breve{\mathcal{D}}) \leq e^\epsilon\mathbb{P}(\widetilde{\DD}'(\tau) \leq \alpha|\mathcal{D}, \breve{\mathcal{D}}')+ \delta,
    \end{align*}
    where the first inequality follows from the fact that $\widetilde{\DD}(\tau)$ is a non-increasing function as shown in \Cref{l_DD_non_increase} and $\widetilde{\DD}(\widetilde{\tau}) \leq \alpha$, the second inequality follows as $\widetilde{\DD}(\tau)$ is $(\epsilon,\delta)$-DP by the Gaussian mechanism. Consider the following set: $\mathcal{A} = \{t\geq 0: \widetilde{\DD}'(t) \leq \alpha\}$. In the case when $\widetilde{\DD}'(\tau) \leq \alpha$, we have $\tau \in \mathcal{A}$ and $\mathcal{A}$ is non-empty. Define $\tau_{\mathcal{A}} = \inf \mathcal{A}$ and we have $\tau_{\mathcal{A}} \leq \tau$. Recall, by construction, we have that $\widetilde{\tau}' = \argmin_{t \in \mathbb{R}} \{|t| :  |\widetilde{\DD}'(t)| \leq \alpha \}$, 
    which is the first time the fairness regime is hit. Consequently, we have that $\widetilde{\tau}' \leq \tau_{\mathcal{A}}$, and further $\{\widetilde{\DD}'(\tau) \leq \alpha\} \subseteq \{\widetilde{\tau}' \leq \tau\}$. Thus,
    \begin{align*}
         \mathbb{P}(\widetilde{\tau} \leq \tau|\mathcal{D}, \breve{\mathcal{D}})  \leq\;&  e^\epsilon\mathbb{P}(\widetilde{\DD}'(\tau) \leq \alpha|\mathcal{D}, \breve{\mathcal{D}}')+ \delta\\
         \leq\;& e^\epsilon \mathbb{P}(\widetilde{\tau}' \leq  \tau |\mathcal{D}, \breve{\mathcal{D}}')+ \delta.
    \end{align*}
    By a similar and even simpler argument, for any $\tau <0$, we also have that 
    \begin{align*}
        \mathbb{P}(\widetilde{\tau} \leq \tau|\mathcal{D}, \breve{\mathcal{D}}) \leq \;& \mathbb{P}(\widetilde{\DD}(\tau) \leq -\alpha|\mathcal{D}, \breve{\mathcal{D}})\\
        \leq \;& e^\epsilon\mathbb{P}(\widetilde{\DD}'(\tau) \leq -\alpha|\mathcal{D}, \breve{\mathcal{D}}')+ \delta \\
        \leq \;& e^\epsilon\mathbb{P}(\widetilde{\tau}' \leq \tau|\mathcal{D}, \breve{\mathcal{D}}') + \delta,
    \end{align*}
    where the last inequality follows since $\widetilde{\DD}'$ is a non-increasing function and the fact that $\widetilde{\DD}'(\tau) \leq -\alpha$ implies $\widetilde{\DD}'(t) \leq -\alpha$ for all $t \geq \tau$, hence $\widetilde{\tau}' \leq \tau$. Thus, combining both cases together, we show that the release of $\widetilde{\tau}$ is $(\epsilon, \delta)$-DP.

    \noindent\textbf{Step 2-2: DP composition.}
    In \textbf{S1} of \Cref{alg_fair_cdp}, since $\widetilde{p}_{X|A=1}$ and $\widetilde{p}_{X|A=0}$ are computed using disjoint dataset, by parallel composition, if each of them satisfies $(\epsilon/4, \delta/4)$-DP, then both of them together are guaranteed to be $(\epsilon/4, \delta/4)$-DP. A similar idea follows for $\widetilde{\eta}_1$ and  $\widetilde{\eta}_0$. Therefore, together with the fact that  $\widetilde{\pi}_0$ and $\widetilde{\pi}_1$ are each $(\epsilon/4, \delta/4)$-DP, by \Cref{l_dp_composition}, we can show that \textbf{S1} is $(\epsilon,\delta)$-CDP. In \textbf{S2}, we use an independent calibration dataset; thus, by the parallel composition property of DP \citep[e.g.][]{smith2021making}, we have that \Cref{alg_fair_cdp} is $(\epsilon,\delta)$-DP.
\end{proof}

\subsection{Disparity control}
\begin{proposition}[Fairness guarantee] \label{thm_cdp_fair_guarantee}
    For any $\alpha >0$, let $\widetilde{f}_{\DD, \alpha}$ denote the classifier output by \Cref{alg_fair_cdp}, we have, for any $\eta \in (0,1/2)$,~that
    \begin{equation*}
        \mathbb{P}\Big\{|\DD(\widetilde{f}_{\DD, \alpha})|\leq \alpha + C_1\sqrt{\frac{\log(1/\eta)}{\breve{n}}} + C_2\sqrt{\frac{\log(1/
        \eta)\log(1/\delta)}{\breve{n}^2\epsilon^2}}\Big\} \geq 1-\eta,
    \end{equation*}
    where $C_1, C_2>0$ are absolute constants and the probability in $\DD(\widetilde{f}_{\DD, \alpha})$ is taken over the test sample conditioning on training data and $\mathbb{P}$ is taken over the training data.
\end{proposition}

\begin{proof}
   With a slight abuse of notation, in this proof, denote $\mathcal{D} = \{\widetilde{\eta}_0, \widetilde{\eta}_1, \widetilde{\pi}_0, \widetilde{\pi}_1\}$. Conditioning on $\mathcal{D}$, by the Dvoretzky--Kiefer--Wolfowitz inequality \citep{dvoretzky1956asymptotic,massart1990tight}, we have, for any $a \in \{0,1\}$ and $t >0$,
   \begin{align*}
       \mathbb{P}\Big[\sup_{\tau \in \mathbb{R}}\Big|&\frac{1}{\breve{n}_a}\sum_{y\in\{0,1\}}\sum_{i=1}^{\breve{n}_{a,y}}\indc\Big\{\widetilde{\eta}_a(\breve{X}^i_{a,y}) \geq\frac{1}{2}+\frac{\tau(2a-1)}{2\widetilde{\pi}_a}\Big\}\\
       & - \mathbb{P}_{X|A=a}\Big\{\widetilde{\eta}_a(X) \geq\frac{1}{2}+\frac{\tau(2a-1)}{2\widetilde{\pi}_a}\Big|\mathcal{D}\Big\}\Big| \geq t \Big|\mathcal{D}\Big] \leq C_1\exp(-2\breve{n}_at^2).
   \end{align*}
   Taking another expectation over $\mathcal{D}$, we have by the Tower property that 
   \begin{align}
       \mathbb{P}\Big[\sup_{\tau \in \mathbb{R}}\Big|&\frac{1}{\breve{n}_a}\sum_{y\in\{0,1\}}\sum_{i=1}^{\breve{n}_{a,y}}\indc\Big\{\widetilde{\eta}_a(\breve{X}^i_{a,y}) \geq\frac{1}{2}+\frac{\tau(2a-1)}{2\widetilde{\pi}_a}\Big\} \nonumber \\
       & - \mathbb{P}_{X|A=a}\Big\{\widetilde{\eta}_a(X) \geq\frac{1}{2}+\frac{\tau(2a-1)}{2\widetilde{\pi}_a}\Big|\mathcal{D}\Big\}\Big| \geq t \Big] \leq C_1\exp(-2\breve{n}_at^2). \label{eq-fair-control-cdp-proof-1}
   \end{align}
   Additionally, by standard Gaussian tail properties \citep[e.g.~Proposition 2.1.2 in][]{vershynin2018high}, we have with probability at least $1-\eta/6$ that 
   \begin{align*}
       |\breve{w}| \leq \frac{C_2\sqrt{\log(1/\delta)\log(1/\eta)}}{\{\breve{n}_0\wedge \breve{n}_1\}\epsilon}.
   \end{align*}
   Consequently, taking $t = C_3 \sqrt{\log(1/\eta)/\{\breve{n}_0\wedge \breve{n}_1\}}$ and by a union bound argument, we have with probability at least $1-\eta/2$ that 
   \begin{align*}
       &|\DD(\widetilde{f}_{\DD, \alpha})| \\
       =\;& \Big|\mathbb{P}_{X|A=1}\Big\{\widetilde{\eta}_1(X) \geq\frac{1}{2}+\frac{\widetilde{\tau}_{\DD,\alpha}}{2\widetilde{\pi}_1}\Big|\mathcal{D}\Big\}- \mathbb{P}_{X|A=0}\Big\{\widetilde{\eta}_0(X) \geq\frac{1}{2}-\frac{\widetilde{\tau}_{\DD,\alpha}}{2\widetilde{\pi}_0}\Big|\mathcal{D}\Big\}\Big|\\
       \leq \;& \Big|\frac{1}{\breve{n}_1}\sum_{y\in\{0,1\}}\sum_{i=1}^{\breve{n}_{1,y}}\indc\Big\{\widetilde{\eta}_1(\breve{X}^i_{1,y}) \geq \frac{1}{2}+\frac{\widetilde{\tau}_{\DD,\alpha}}{2\widetilde{\pi}_1}\Big\} -\frac{1}{\breve{n}_0}\sum_{y\in\{0,1\}}\sum_{i=1}^{\breve{n}_{0,y}}\indc\Big\{\widetilde{\eta}_0(\breve{X}^i_{0,y}\geq \frac{1}{2}-\frac{\widetilde{\tau}_{\DD,\alpha}}{2\widetilde{\pi}_0}\Big\}\Big| \\
       &+ C_4\sqrt{\frac{\log(1/\eta)}{\breve{n}_0\wedge \breve{n}_1}}\\
       \leq \;& |\widetilde{\DD}(\widetilde{\tau}_{\DD,\alpha})| + |\breve{w}| +C_4\sqrt{\frac{\log(1/\eta)}{\breve{n}_0\wedge \breve{n}_1}}\\
       \leq \;& \alpha + C_4\sqrt{\frac{\log(1/\eta)}{\breve{n}_0\wedge \breve{n}_1}} + \frac{C_2\sqrt{\log(1/\delta)\log(1/\eta)}}{\{\breve{n}_0\wedge \breve{n}_1\}\epsilon},
   \end{align*}
   where the first inequality follows from the triangle inequality and \eqref{eq-fair-control-cdp-proof-1}, the second inequality follows from the triangle inequality and the definition of $\widetilde{\DD}$, and the last inequality follows from the construction in \textbf{S2} of \Cref{alg_fair_cdp}. The final result holds by a further union bound argument with the event $\{\breve{n}_1 \asymp \breve{n}_2 \asymp \breve{n}\}$ in \Cref{l_n_a_n_same_order}.
   
\end{proof}

\subsection{Excess risk control} \label{sec_appendix_excess_risk_cdp}
\begin{theorem} \label{t_fair_risk}
    Suppose Assumptions \ref{a_prob}, \ref{a_kernel} and \ref{a_posterior} holds. Denote
    \[\rho = C_1\Big[\Big\{\sqrt{\frac{\log\{(h^{-d}\vee n)/\eta\}}{nh^d}}+\sqrt{\frac{\log(1/\delta)\log(1/
        \eta)}{n^2\epsilon^2 h^{2d}}}+h^\beta\Big\}^\gamma +\sqrt{\frac{\log(1/\eta)}{\breve{n}}}+\sqrt{\frac{\log(1/\delta)\log(1/\eta)}{\breve{n}^2\epsilon^2}}\Big],\]
    and we further assume that $\DD(0) \notin [\alpha - \zeta, \alpha] \cup [-\alpha,-\alpha+\zeta]$ where $\rho \leq \zeta < 
    \alpha$. Then for any $\min\{n, \breve{n}\} \geq 4C_1^2\log(1/\eta)/C_\pi^2$ and absolute constants $C_1, \ldots, C_7 >0$, the following holds.
    \begin{enumerate} 
        \item  \label{t_fair_risk_1}For any $\eta \in (0,1/2)$ such that
        \[nh^d \geq C_2\log\{(h^{-d}\vee n)/\eta\}, \; n^2\epsilon^2 h^{2d} \geq C_3\log(1/\delta)\log(1/\eta), \; \breve{n}\epsilon^2 \geq C_4\log(1/\delta)\log(1/\eta),\]
        then it holds with probability at least $1-\eta$ that
        \begin{align*}
            d_{\mathrm{fair}}(\widetilde{f}_{\DD,\alpha}, f^*_{\DD,\alpha}) \leq\;& C_5\Big[\sqrt{\frac{\log\{(h^{-d}\vee n)/\eta\}}{nh^d}}+\sqrt{\frac{\log(1/\delta)\log(1/
        \eta)}{n^2\epsilon^2 h^{2d}}}+h^\beta \\
            &\hspace{0.5cm} + \indc\{\tau^*_{\DD,\alpha}\neq 0\}\cdot\Big\{\sqrt{\frac{\log(1/\eta)}{\breve{n}}}+\sqrt{\frac{\log(1/\delta)\log(1/\eta)}{\breve{n}^2\epsilon^2}}\Big\}^{1/
    \gamma}\Big]^{1+\gamma}.
        \end{align*}

        \item \label{t_fair_risk_2} Additionally, if $n\asymp \breve{n} \asymp N$ up to poly-logarithmic factors, then
        \begin{enumerate}
            \item In the fairness-impacted regime when $\tau^*_{\DD,\alpha}\neq 0$, set the bandwidth as
            \[h= C_6\{N^{-\frac{1}{2\beta+d}} \vee (N^2\epsilon^2)^{-\frac{1}{2\beta+2d}} \vee N^{-\frac{1}{2\gamma\beta}}   \vee (N^2\epsilon^2)^{-\frac{1}{2\gamma \beta}}\},\]
            it then holds that 
            \begin{align*}
                d_{\mathrm{fair}}(\widetilde{f}_{\DD,\alpha}, f^*_{\DD,\alpha}) = O_{\mathrm{p}}\Big\{N^{-\frac{\beta(1+\gamma)}{2\beta+d}} \vee (N^2\epsilon^2)^{-\frac{\beta(1+\gamma)}{2\beta+2d}} \vee N^{-\frac{1+\gamma}{2\gamma}} \vee (N^2\epsilon^2)^{-\frac{1+\gamma}{2\gamma}}\Big\},
            \end{align*}
            and 
            \begin{align*}
                |R(\widetilde{f}_{\DD,\alpha})-R(f^*_{\DD,\alpha})| \leq d_{\mathrm{fair}}(\widetilde{f}_{\DD,\alpha}, f^*_{\DD,\alpha}) +|\tau^*_{\DD,\alpha}|\cdot O_{\mathrm{p}}\{N^{-\frac{1}{2}} \vee (N^2\epsilon^2)^{-1}\}.
            \end{align*}

            \item In the automatically fair regime when $\tau^*_{\DD,\alpha} = 0$, pick the bandwidth
            \[h = C_7\{N^{-\frac{1}{2\beta+d}} \vee (N^2\epsilon^2)^{-\frac{1}{2\beta+2d}}\},\]
            then it holds that 
            \begin{align*}
                R(\widetilde{f}_{\DD,\alpha})-R(f^*_{\DD,\alpha}) = d_{\mathrm{fair}}(\widetilde{f}_{\DD,\alpha}, f^*_{\DD,\alpha}) = O_{\mathrm{p}}\Big\{N^{-\frac{\beta(1+\gamma)}{2\beta+d}} \vee (N^2\epsilon^2)^{-\frac{\beta(1+\gamma)}{2\beta+2d}}\Big\}.
            \end{align*}
        \end{enumerate}
    \end{enumerate}
\end{theorem}

\begin{proof}[Proof of \Cref{t_fair_risk}]
    To prove \Cref{t_fair_risk}, it suffices to prove \Cref{t_fair_risk}(\ref{t_fair_risk_1}), as \Cref{t_fair_risk}(\ref{t_fair_risk_2}) is a direct consequence. 

    For notational simplicity, throughout the proof, we write $\tau^*_{\DD,\alpha} = \tau^*$ and $\widetilde{\tau}_{\DD,\alpha} = \widetilde{\tau}$. Denote 
    \[T^*_a =\frac{1}{2}+\frac{\tau^*(2a-1)}{2\pi_a},\;\; \text{and}\;\; \widetilde{T}_a = \frac{1}{2}+\frac{\widetilde{\tau}(2a-1)}{2\widetilde{\pi}_a} \]
        
    To control the fairness-aware risk, we have 
    \begin{align} \notag
        &d_{\mathrm{fair}}(\widetilde{f}_{\DD,\alpha}, f^*_{\DD,\alpha})\\ \notag
        =\;& 2 \sum_{a\in\{0,1\}} \pi_a
    \Big[\int \big\{\widetilde{f}_{\DD,\alpha}(x, a) - f^*_{\DD,\alpha}(x,a) \big\}
    \big\{T^*_a  - \eta_a(x) \big\}\;
    \mathrm{d}\mathbb{P}_{X\mid A=a}(x) \Big]\\ \notag
    = \;& 2 \sum_{a\in\{0,1\}} \pi_a
    \Big[\int \big[\indc\{\widetilde{\eta}_a(x) \geq \widetilde{T}_a\} - \indc\{\eta_a(x) \geq T^*_a\} \big]
    \big\{T^*_a  - \eta_a(x) \big\}\;
    \mathrm{d}\mathbb{P}_{X\mid A=a}(x) \Big]\\ \notag
    = \;& 2 \sum_{a\in\{0,1\}} \pi_a
    \Big[\int \big[\indc\{0> \eta_a(x)-T_a^* \geq \eta_a(x) - \widetilde{\eta}_a(x)+ \widetilde{T}_a-T^*_a\} \big]
    \big\{T^*_a  - \eta_a(x) \big\}\;
    \mathrm{d}\mathbb{P}_{X\mid A=a}(x) \Big]\\ \notag
    &-2\sum_{a\in\{0,1\}} \pi_a
    \Big[\int \big[\indc\{0\leq \eta_a(x) -T_a^* < \eta_a(x)-\widetilde{\eta}_a(x)+\widetilde{T}_a-T^*_a\}\big]
    \big\{T^*_a  - \eta_a(x) \big\}\;
    \mathrm{d}\mathbb{P}_{X\mid A=a}(x) \Big]\\ \notag
    \leq \;& 2\sum_{a\in\{0,1\}} \pi_a \Big[\int \indc\big\{|\eta_a(x) -T_a^*| \leq \|\widetilde{\eta}_a- \eta_a\|_\infty + |\widetilde{T}_a - T_a^*|\big\}\big|T^*_a  - \eta_a(x) \big|\;\mathrm{d}\mathbb{P}_{X\mid A=a}(x)\Big]\\ \label{t_fair_risk_eq1}
    \leq \;& C_1 \max_{a \in \{0,1\}} (\|\widetilde{\eta}_a- \eta_a\|_\infty + |\widetilde{T}_a - T_a^*|)^{1+\gamma},
    \end{align}
    where the last inequality follows from \Cref{a_posterior}. Consider the following two events,
    \[\mathcal{E}_1 = \Big\{\|\widetilde{\eta}_a- \eta_a\|_\infty \leq C_2 \Big\{\sqrt{\frac{\log\{(h^{-d}\vee n)/\eta\}}{nh^d}}+\frac{\sqrt{\log(1/\delta)\log(1/\eta)}}{n\epsilon h^d}+h^\beta \Big\},\; a\in \{0,1\}\Big\},\]
    and 
    \begin{align} \label{t_fair_risk_eq2}
        \mathcal{E}_2 = \Big\{|\widetilde{T}_a - T_a^*| \leq \varepsilon_T, \; a \in \{0, 1\}\Big\},
    \end{align}
    where 
    \begin{align*}
        \varepsilon_T =\;& C_3 \Big[\sqrt{\frac{\log(1/\eta)}{n}}+\sqrt{\frac{\log(1/\delta)\log(1/\eta)}{n^2\epsilon^2}}\\
        &+\indc\{\tau^*\neq 0\}\cdot\Big[\sqrt{\frac{\log\{(h^{-d}\vee n)/\eta\}}{nh^d}}+\frac{\sqrt{\log(1/\delta)\log(1/
        \eta)}}{n\epsilon h^d}+h^\beta \\
        & \hspace{2.8cm} +\Big\{\sqrt{\frac{\log(1/\eta)}{\breve{n}}}+\sqrt{\frac{\log(1/\delta)\log(1/\eta)}{\breve{n}^2\epsilon^2}}\Big\}^{1/
    \gamma}\Big]\Big].
    \end{align*}
    By Lemmas \ref{l_eta_est_sup} and \ref{l_n_a_n_same_order} and a union bound argument on $a \in \{0,1\}$, we have that $\mathbb{P}(\mathcal{E}_1) \geq 1-\eta/2$. To control the probability of $\mathcal{E}_2$ happening, note that
    \begin{align*}
        |\widetilde{T}_a - T_a^*| \leq \;& \Big|\frac{\widetilde{\tau}}{\widetilde{\pi}_a} - \frac{\tau^*}{\pi_a}\Big| \leq |\widetilde{\tau}|\cdot \Big|\frac{1}{\widetilde{\pi}_a} - \frac{1}{\pi_a}\Big| + \frac{1}{\pi_a}|\widetilde{\tau} - \tau| = (I)+(II).
    \end{align*}
    To control $(I)$, first note that by Lemmas \ref{l_tau_est_err} and \ref{l_range_tau}, we have that with probability at least $1-\eta/6$ that 
    \begin{align*}
        |\widetilde{\tau}| \leq\;& |\tau^*| +C_4\indc\{\tau^*\neq 0\}\cdot\Big[\sqrt{\frac{\log\{(h^{-d}\vee n)/\eta\}}{nh^d}}+\frac{\sqrt{\log(1/\delta)\log(1/
        \eta)}}{n\epsilon h^d}+h^\beta \\
        & \hspace{5cm}+\Big\{\sqrt{\frac{\log(1/\eta)}{\breve{n}}}+\sqrt{\frac{\log(1/\delta)\log(1/\eta)}{\breve{n}^2\epsilon^2}}\Big\}^{1/\gamma}\Big]\\
        \leq\;& \min\{\pi_0, \pi_1\} + C_5 \leq C_6.
    \end{align*}
    Also, by a \Cref{l_pi_a_est} and a similar argument leading to \eqref{l_DD_est_err_eq4}, it holds with probability at least $1-\eta/6$ that, for $a \in \{0,1\}$, 
    \[\Big|\frac{1}{\widetilde{\pi}_a} - \frac{1}{\pi_a}\Big| \leq C_7\Big\{\sqrt{\frac{\log(1/\eta)}{n}}+\sqrt{\frac{\log(1/\delta)\log(1/\eta)}{n^2\epsilon^2}}\Big\}.\]
    To control $(II)$, by \Cref{l_tau_est_err}, we have with probability at least $1-\eta/6$ that 
    \begin{align*}
        (II) \leq\;& \frac{C_7}{C_\pi}\indc\{\tau^*\neq 0\}\cdot\Big[\sqrt{\frac{\log\{(h^{-d}\vee n)/\eta\}}{nh^d}}+\frac{\sqrt{\log(1/\delta)\log(1/
        \eta)}}{n\epsilon h^d}+h^\beta \\
        & \hspace{0.8cm} +\Big\{\sqrt{\frac{\log(1/\eta)}{\breve{n}}}+\sqrt{\frac{\log(1/\delta)\log(1/\eta)}{\breve{n}^2\epsilon^2}}\Big\}^{1/
    \gamma}\Big].
    \end{align*}
    Applying a union bound argument, we thus have that $\mathbb{P}(\mathcal{E}_2) \geq 1-\eta/2$. Hence, a further union bound argument gives $\mathbb{P}(\mathcal{E}_1\cap\mathcal{E}_2) \geq 1-\eta$ and the Lemma thus follows by conditioning on $\mathcal{E}_1\cap\mathcal{E}_2$ happening.

\end{proof}

\subsubsection{Upper bound on \texorpdfstring{$\|\widetilde{\eta}_a- \eta_a\|_\infty$}{}}
\begin{lemma}\label{l_eta_est_sup}
    Under the same assumptions as in \Cref{t_fair_risk}, it holds for any $\eta \in (0,1/2)$ and $a \in \{0,1\}$ that 
    \begin{align*}
        \mathbb{P}\Big[\|\widetilde{\eta}_a- \eta_a\|_\infty \geq C_1\Big\{\sqrt{\frac{\log\{(h^{-d}\vee n_a)/\eta\}}{n_ah^d}}+\frac{\sqrt{\log(1/\delta)\log(1/\eta)}}{n_a\epsilon h^d}+h^\beta\Big\}\Big] \leq \eta.
    \end{align*}
\end{lemma}

\begin{proof}
    By the triangle inequality, we have that 
    \begin{align}\notag
        \|\widetilde{\eta}_a- \eta_a\|_\infty \leq\;& \Big\|\frac{1}{\widetilde{p}_{X|A=a}(\cdot|a)}\Big\{\frac{1}{n_a}\sum_{i=1}^{n_{a,1}} K_{h}(X_{a,1}^i -\cdot)\Big\} - \eta_a(\cdot)\Big\|_\infty + \frac{8\sqrt{2C_K\log(5/\delta)}}{n_a\epsilon h^d}\Big\|\frac{W_{2,a}(\cdot)}{\widetilde{p}_{X|A=a}(\cdot|a)} \Big\|_{\infty}\\ \label{l_eta_est_sup_eq3}
        =\;& (I)+ (II).
    \end{align}
    
    \noindent \textbf{Step 1: Upper bound on Term $(I)$.} To control $(I)$, note that 
    \begin{align*}
        (I) =\;& \sup_{x \in [0,1]^d}\Big|\frac{1}{\widetilde{p}_{X|A=a}(x|a)}\Big\{\frac{1}{n_a}\sum_{i=1}^{n_{a,1}} K_{h}(X_{a,1}^i -x)\Big\} - \frac{p_{X,Y|A=a}(x,1|a)}{p_{X|A=a}(x|a)}\Big|\\
        \leq \;& \sup_{x \in [0,1]^d}\Big|\frac{1}{\widetilde{p}_{X|A=a}(x|a)}\Big\{\frac{1}{n_a}\sum_{i=1}^{n_{a,1}} K_{h}(X_{a,1}^i -x)- \int_{u} K_{h}(u -x)p_{X,Y|A=a}(u,1|a)\;\mathrm{d}u\Big\}\Big|\\
        &+ \sup_{x \in [0,1]^d}\Big|\frac{1}{\widetilde{p}_{X|A=a}(x|a)}\Big\{\int_{u} K_{h}(u -x)p_{X,Y|A=a}(u,1|a)\;\mathrm{d}u - p_{X,Y|A=a}(x,1|a)\Big\}\Big|\\
        & +\sup_{x \in [0,1]^d}\Big|\frac{p_{X,Y|A=a}(x,1|a)}{\widetilde{p}_{X|A=a}(x|a)} -\frac{p_{X,Y|A=a}(x,1|a)}{p_{X|A=a}(x|a)}\Big|\\
        =\;& (I)_1 + (I)_2 + (I)_3.
    \end{align*}
    To control $(I)_1$, we have that 
    \begin{align*}
        (I)_1 \leq \sup_{x \in [0,1]^d} \frac{1}{\widetilde{p}_{X|A=a}(x|a)} \cdot \sup_{x \in [0,1]^d} \Big|\frac{1}{n_a}\sum_{i=1}^{n_{a,1}} K_{h}(X_{a,1}^i -x)- \int_{u} K_{h}(u -x)p_{X,Y|A=a}(u,1|a)\;\mathrm{d}u\Big|.
    \end{align*}
    Therefore, by a union bound argument and Lemmas \ref{l_est_p_yx_a_kernel_sup} and \ref{l_p_xa_const_order}, it holds that 
    \begin{align*}
        \mathbb{P}\Big\{(I)_1 \geq C_1\sqrt{\frac{\log\{(h^{-d}\vee n_a)/\eta\}}{n_ah^d}}\Big\} \leq \frac{\eta}{6}.
    \end{align*}
    To control $(I)_2$, we have that 
    \begin{align*}
        (I)_2 \leq\;&  \sup_{x \in [0,1]^d} \frac{1}{\widetilde{p}_{X|A=a}(x|a)} \\
        & \hspace{2cm} \cdot \sup_{x \in [0,1]^d}\Big|\int_{u} K_{h}(u -x)p_{X,Y|A=a}(u,1|a)\;\mathrm{d}u - p_{X,Y|A=a}(x,1|a)\Big|\\
        \leq \;& C_{\text{adp}} h^\beta\sup_{x \in [0,1]^d} \frac{1}{\widetilde{p}_{X|A=a}(x|a)}, 
    \end{align*}
    where the last inequality follows from \Cref{a_kernel}\ref{a_kernel_adaptive} and the fact that both $\eta_a$ and $p_{X|A=a}$ are H\"{o}lder over $[0,1]^d$ under Assumptions \ref{a_prob}\ref{a_prob_p_x_a} and \ref{a_posterior}\ref{a_posterior_holder}. Therefore, by \Cref{l_p_xa_const_order}, it holds that 
    \begin{align*}
        \mathbb{P}\Big\{(I)_2 \geq C_2h^{\beta}\Big\} \leq \eta/6.
    \end{align*}
    To control $(I)_3$, note that 
    \begin{align*}
        (I)_3 \leq\;& \sup_{x \in [0,1]^d} p_{X,Y|A=a}(x,1|a) \cdot \sup_{x \in [0,1]^d} \Big|\frac{1}{\widetilde{p}_{X|A=a}(x|a)} -\frac{1}{p_{X|A=a}(x|a)}\Big|\\
        \leq \;& C_3\sup_{x \in [0,1]^d}\Big|\frac{\widetilde{p}_{X|A=a}(x|a)- p_{X|A=a}(x|a)}{\widetilde{p}_{X|A=a}(x|a)p_{X|A=a}(x|a)}\Big|\\
        \leq \;& C_3 \sup_{x \in [0,1]^d} \frac{1}{\widetilde{p}_{X|A=a}(x|a)p_{X|A=a}(x|a)} \cdot  \sup_{x \in [0,1]^d} \big|\widetilde{p}_{X|A=a}(x|a)- p_{X|A=a}(x|a)\big|\\
        \leq \;& \frac{C_3}{C_p} \sup_{x \in [0,1]^d} \frac{1}{\widetilde{p}_{X|A=a}(x|a)} \cdot  \sup_{x \in [0,1]^d} \big|\widetilde{p}_{X|A=a}(x|a)- p_{X|A=a}(x|a)\big|,
    \end{align*}
    where the second inequality follows from the fact that $p_{X,Y|A=a}$ is continuous over $[0,1]^d$, hence is bounded, and the fourth inequality follows from \Cref{a_prob}\ref{a_prob_p_x_a}. Thus, by a union bound argument and Lemmas \ref{l_p_x_a_est_sup} and \ref{l_p_xa_const_order}, it holds that
    \begin{align*}
        \mathbb{P}\Big[(I)_3 \geq C_4\Big\{\sqrt{\frac{\log\{(h^{-d}\vee n_a)/\eta\}}{n_ah^d}}+\frac{\sqrt{\log(1/\delta)\log(1/\eta)}}{n_a\epsilon h^d}+h^\beta\Big\}\Big] \leq \frac{\eta}{6}.
    \end{align*}
    Thus, by another union bound, we have that 
    \begin{align} \label{l_eta_est_sup_eq1}
        \mathbb{P}\Big[(I) \geq C_5\Big\{\sqrt{\frac{\log\{(h^{-d}\vee n_a)/\eta\}}{n_ah^d}}+\frac{\sqrt{\log(1/\delta)\log(1/\eta)}}{n_a\epsilon h^d}+h^\beta\Big\}\Big] \leq \frac{\eta}{2}.
    \end{align}

    \noindent \textbf{Step 2: Upper bound on $(II)$.} Note that 
    \begin{align*}
        (II) \leq \frac{8\sqrt{2C_K\log(8/\delta)}}{n_a\epsilon h^d}\sup_{x\in[0,1]^d}\Big|W_{2,a}(x)\Big| \cdot \sup_{x \in [0,1]^d} \frac{1}{\widetilde{p}_{X|A=a}(x|a)}.
    \end{align*}
    Therefore, by a union bound argument and Lemmas \ref{l_p_xa_const_order} and \ref{l_sup_gaussian}, we have that 
    \begin{align}\label{l_eta_est_sup_eq2}
        \mathbb{P}\Big\{(II) \geq \frac{C_6\sqrt{\log(1/\delta)\log(1/\eta)}}{n_a\epsilon h^d}\Big\} \leq \frac{\eta}{2}.
    \end{align}

    \noindent \textbf{Step 3: Combine results.} The lemma thus follows by substituting the results in \eqref{l_eta_est_sup_eq1} and \eqref{l_eta_est_sup_eq2} into \eqref{l_eta_est_sup_eq3}, and a union bound argument.

\end{proof}

\begin{lemma} \label{l_est_p_yx_a_kernel_sup}
    Under the same assumptions as in \Cref{t_fair_risk}, it holds for any $\eta \in (0,1/2)$ and $a \in \{0,1\}$ that 
    \begin{align*}
        \mathbb{P}\Big\{\sup_{x \in [0,1]^d}\Big|\frac{1}{n_a}\sum_{i=1}^{n_{a,1}} K_{h}(X_{a,1}^i -x)- \int_{u} K_{h}(u -x)&p_{X,Y|A=a}(u,1|a)\;\mathrm{d}u\Big|\\
        &\geq C_1\sqrt{\frac{\log\{(h^{-d}\vee n_a)/\eta\}}{n_ah^d}}\Big\} \leq \eta.
    \end{align*}
\end{lemma}

\begin{proof}
    The proof follows a similar argument to the one in the proof of \Cref{l_est_p_xa_kernel_sup}. Here, we only include the difference. With a slight abuse of notation, we aggregate all data points with $A_i=a$ together and rewrite 
    \[\frac{1}{n_a}\sum_{i=1}^{n_{a,1}} K_{h}(X_{a,1}^i -x) = \frac{1}{n_a }\sum_{i=1}^{n_{a}}Y^i_{a}K_h(X_{a}^i -x).\]
    Note that for any $x \in [0,1]^d$, under \Cref{a_kernel}\ref{a_kernel_bounded}, we have that 
    \[|Y^i_{a}K_h(X_{a}^i -x)| =\Big|\frac{1}{h^d}Y^i_{a}K\Big(\frac{X_{a}^i -x}{h}\Big)\Big| \leq C_Kh^{-d},\]
    and also
    \begin{align*}
        \mathbb{E}\Big\{\frac{1}{n_a}\sum_{i=1}^{n_{a}}Y^i_{a}K_h(X_{a}^i -x)\Big\} =\;& \frac{1}{n_a} \sum_{i=1}^{n_{a}} \sum_{y \in \{0,1\}} \int y K_h(u-x)p_{X,Y|A=a}(u,y|a)\;\mathrm{d}u\\
        =\;& \int_{u} K_{h}(u -x)p_{X,Y|A=a}(u,1|a)\;\mathrm{d}u.
    \end{align*}
    To control the variance, we observe that for any $i \in [n_a]$, 
    \begin{align} \notag
        \var\{Y^i_{a}K_h(X_{a}^i -x)\} \leq\;& \mathbb{E}\Big\{\frac{1}{h^{2d}}(Y^i_a)^2K\Big(\frac{X_{a}^i -x}{h}\Big)^2\Big\} \leq \frac{1}{h^{2d}}\mathbb{E}\Big\{K\Big(\frac{X_{a}^i -x}{h}\Big)^2\Big\}\\ \label{l_est_p_yx_a_kernel_sup_eq1}
        \leq \;& \frac{C_1}{h^{2d}} \int h^dK(t)^2 \;\mathrm{d}t \leq C_1C_K^2h^{-d},
    \end{align}
    where the second inequality follows from the fact that $|Y_a^i| \leq 1$, the third inequality follows from \Cref{a_prob}\ref{a_prob_p_x_a} and the last inequality follows from \Cref{a_kernel}\ref{a_kernel_bounded}. Therefore, by the Bernstein inequality for bounded distributions \citep[e.g.~Theorem 2.8.4 in][]{vershynin2018high}, it holds that for any $0< t <1$,
    \begin{align*}
        &\mathbb{P}\Big\{\Big|\frac{1}{n_a }\sum_{i=1}^{n_{a}}Y^i_{a}K_h(X_{a}^i -x)-\int_{u} K_{h}(u -x)p_{X,Y|A=a}(u,1|a)\;\mathrm{d}u\Big| \geq t\Big\}\\
        \leq \;& 2\exp\Big(- \frac{n_at^2}{C_1C_K^2h^{-d}+C_K t h^{-d}}\Big) \leq C_2\exp(-n_ah^dt^2).
    \end{align*}
    The rest of the proof to control the sup-norm is by a similar argument using the covering lemma as in \textbf{Step 2} in the proof of \Cref{l_est_p_xa_kernel_sup}, hence is omitted here. 
    
\end{proof}

\subsubsection{Upper bound on \texorpdfstring{$|\widetilde{\tau} -\tau^*|$}{}}
\begin{lemma} \label{l_tau_est_err}
    Suppose the same assumptions as in \Cref{t_fair_risk} hold. If we additionally assume that 
    \begin{align*}
        \varpi \geq \;& C_1\Big[\sqrt{\frac{\log\{(h^{-d}\vee n)/\eta\}}{nh^d}}+\frac{\sqrt{\log(1/\delta)\log(1/
        \eta)}}{n\epsilon h^d}+h^\beta \\
        & \hspace{1cm} +\Big\{\sqrt{\frac{\log(1/\eta)}{\breve{n}}}+\sqrt{\frac{\log(1/\delta)\log(1/\eta)}{\breve{n}^2\epsilon^2}}\Big\}^{1/\gamma}\Big],
    \end{align*}
    where $\varpi$ is given in \Cref{a_posterior}\ref{a_posterior_margin}, then 
    \begin{align*}
        \mathbb{P}\Big[|\widetilde{\tau} -\tau^*| \geq& C_2\indc\{\tau^*\neq 0\}\cdot\Big[\sqrt{\frac{\log\{(h^{-d}\vee n)/\eta\}}{nh^d}}+\frac{\sqrt{\log(1/\delta)\log(1/
        \eta)}}{n\epsilon h^d}+h^\beta \\
        & \hspace{0.8cm} +\Big\{\sqrt{\frac{\log(1/\eta)}{\breve{n}}}+\sqrt{\frac{\log(1/\delta)\log(1/\eta)}{\breve{n}^2\epsilon^2}}\Big\}^{1/
    \gamma}\Big]\Big] \leq \eta.
    \end{align*}
\end{lemma}
\begin{proof}
    Take 
    \begin{align*}
        t =\;& C_1\Big[\sqrt{\frac{\log\{(h^{-d}\vee n)/\eta\}}{nh^d}}+\frac{\sqrt{\log(1/\delta)\log(1/
        \eta)}}{n\epsilon h^d}+h^\beta \\
        & \hspace{1cm} +\Big\{\sqrt{\frac{\log(1/\eta)}{\breve{n}}}+\sqrt{\frac{\log(1/\delta)\log(1/\eta)}{\breve{n}^2\epsilon^2}}\Big\}^{1/\gamma}\Big],
    \end{align*}
    by assumption, we have that $0 < t \leq \varpi$.
    Consider the event 
    \begin{align*}
        \mathcal{E}_{\DD} = \;& \{|\widetilde{\DD}(\tau^*+t) - \DD(\tau^*+t)| \leq \rho_{\DD}\} \cap \{|\widetilde{\DD}(\tau^*-t) - \DD(\tau^*-t)| \leq \rho_{\DD}\} \\
        &\cap \{|\widetilde{\DD}(\tau^*) - \DD(\tau^*)| \leq \rho_{\DD}\},
    \end{align*}
    where 
    \begin{align*}
        \rho_{\DD} =\;&C_2\Big[\Big\{\sqrt{\frac{\log\{(h^{-d}\vee n)/\eta\}}{nh^d}}+\frac{\sqrt{\log(1/\delta)\log(1/
        \eta)}}{n\epsilon h^d}+h^\beta\Big\}^\gamma \\
        & \hspace{2em} +\sqrt{\frac{\log(1/\eta)}{\breve{n}}}+\sqrt{\frac{\log(1/\delta)\log(1/\eta)}{\breve{n}^2\epsilon^2}}\Big].
    \end{align*}
     We have by \Cref{l_DD_est_err} and a union bound argument that $\mathbb{P}(\mathcal{E}_{\DD}) \geq 1-\eta$. The rest of the proof is constructed conditionally on $\mathcal{E}_{\DD}$ happening. 

    \noindent \textbf{Case 1: When $|\DD(0)| \leq \alpha - \rho_{\DD}$.} In this case, we have $\tau^* = 0$. Moreover, on the event $\mathcal{E}_{\DD}$, by the triangle inequality, we have $|\widetilde{\DD}(0)| \leq |\DD(0)| + \rho_{\DD} \leq \alpha$. Consequently, $\mathbb{P}(|\widetilde{\tau} -\tau^*| = 0) \geq \mathbb{P}(\mathcal{E}_{\DD})$.

    \noindent \textbf{Case 2: When $\alpha - \rho_{\DD} \leq |\DD(0)| \leq \alpha$.} Since by the assumption in \Cref{t_fair_risk}, we have that $\DD(0) \notin [\alpha - \rho_{\DD},\alpha)\cup (-\alpha,-\alpha+\rho_{\DD}] \subseteq [\alpha - \zeta,\alpha)\cup (-\alpha,-\alpha+\zeta]$. This case will therefore never happen.

    \noindent \textbf{Case 3: When $\tau^* >0$ and $\DD(\tau^*) = \alpha$.} This is similar to \textbf{Case 3} in the proof of Lemma 13 in \citet{hu2025fairness}. It holds that 
    \begin{align*}
        \mathbb{P}(\widetilde{\tau}> \tau^* + t) \leq\;&  \mathbb{P}\{\widetilde{\DD}(\tau^* + t) >\alpha, \mathcal{E}_{\DD}\} + \mathbb{P}(\mathcal{E}_{\DD}^c)\\
        =\;& \mathbb{P}\{\widetilde{\DD}(\tau^* + t) - \DD(\tau^* + t) > \DD(\tau^*) -\DD(\tau^* + t), \mathcal{E}_{\DD}\} + \mathbb{P}(\mathcal{E}_{\DD}^c)\\
        \leq \;& \mathbb{P}(\rho_{\DD} > c_mt^{\gamma}) + \mathbb{P}(\mathcal{E}_{\DD}^c) =\mathbb{P}(\mathcal{E}_{\DD}^c),
    \end{align*}
    where the first inequality follows from \Cref{l_DD_non_increase}, the second inequality follows from \Cref{a_posterior}\ref{a_posterior_disparity} and the last inequality follows as $t \geq c_m^{-1/\gamma}\rho_{\DD}^{1/\gamma}$. Similarly, as $t \geq c_m^{-1/\gamma}\rho_{\DD}^{1/\gamma}$, it also holds that 
    \begin{align*}
        \mathbb{P}(\widetilde{\tau} < \tau^* - t) \leq\;& \mathbb{P}\{\widetilde{\DD}(\tau^* - t)  < \alpha, \mathcal{E}_{\DD}\} +\mathbb{P}(\mathcal{E}_{\DD}^c)\\
        \leq \;& \mathbb{P}\{\DD(\tau^*-t) -\widetilde{\DD}(\tau^* - t)  > \DD(\tau^*-t)-\DD(\tau^*), \mathcal{E}_{\DD}\} +\mathbb{P}(\mathcal{E}_{\DD}^c)\\
        \leq \;& \mathbb{P}(\rho_{\DD}  > c_mt^{\gamma}, \mathcal{E}_{\DD}) +\mathbb{P}(\mathcal{E}_{\DD}^c) = \mathbb{P}(\mathcal{E}_{\DD}^c).
    \end{align*}

    \noindent \textbf{Case 4: when $\tau^* <0$ and $\DD(\tau^*) = -\alpha$.} This is similar to Case 4 in the proof of Lemma~13 in \citet{hu2025fairness}. Most of the proof follows from a similar idea as the one to the proof of \textbf{Case 3}, and we omit it here.

\end{proof}

\begin{lemma} \label{l_DD_est_err}
    Recall that $\varpi$ is the parameter given in \Cref{a_posterior}\ref{a_posterior_margin}. Under the same assumptions as in \Cref{t_fair_risk}, it holds for any $\eta \in (0,1/2)$ that
    \begin{align*}
        \sup_{|\kappa| \leq \varpi}\mathbb{P}\Big[&|\widetilde{\DD}(\tau^* +\kappa) - \DD(\tau^* + \kappa)| 
        \\ &\geq C_1\Big[\Big\{\sqrt{\frac{\log\{(h^{-d}\vee n)/\eta\}}{nh^d}}+\frac{\sqrt{\log(1/\delta)\log(1/
        \eta)}}{n\epsilon h^d}+h^\beta\Big\}^\gamma \\
        &\hspace{1cm}+\sqrt{\frac{\log(1/\eta)}{\breve{n}}}+\sqrt{\frac{\log(1/\delta)\log(1/\eta)}{\breve{n}^2\epsilon^2}}\Big] \Big] \leq \eta.
    \end{align*}
\end{lemma}

\begin{proof}
    With a slight abuse of notation, in this proof, denote $\mathcal{D} = \{\widetilde{\eta}_0, \widetilde{\eta}_1, \widetilde{\pi}_0, \widetilde{\pi}_1\}$. By triangle inequality, it holds that 
    \begin{align*}
        &\sup_{|\kappa| \leq \varpi}|\widetilde{\DD}(\tau^*+\kappa) - \DD(\tau^*+\kappa)| \\
        \leq \;& \sup_{|\kappa| \leq \varpi} |\widetilde{\DD}(\tau) - \mathbb{E}\{\widetilde{\DD}(\tau)|\mathcal{D}\}| + \sup_{|\kappa| \leq \varpi}|\mathbb{E}\{\widetilde{\DD}(\tau^*+\kappa)|\mathcal{D}\} - \DD(\tau^*+\kappa)|\\
        = \;& (I)+ (II),
    \end{align*}
    where $\mathbb{E}\{\widetilde{\DD}(\tau)|\mathcal{D}\}$ is given by
    \begin{align} \label{l_DD_est_err_eq1}
        \mathbb{E}\{\widetilde{\DD}(\tau)|\mathcal{D}\} =\;& \mathbb{P}_{X|A=1}\Big\{\widetilde{\eta}_1(X) \geq\frac{1}{2}+\frac{\tau}{2\widetilde{\pi}_1}\Big|\mathcal{D}\Big\} -\mathbb{P}_{X|A=0}\Big\{\widetilde{\eta}_0(X) \geq\frac{1}{2}-\frac{\tau}{2\widetilde{\pi}_0}\Big| \mathcal{D}\Big\},
    \end{align}
    and $\DD(\tau)$ is
    \begin{align} \label{l_DD_est_err_eq2}
       \DD(\tau) =\;& \mathbb{P}_{X|A=1}\Big\{\eta_1(X) \geq\frac{1}{2}+\frac{\tau}{2\pi_1}\Big\} - \mathbb{P}_{X|A=0}\Big\{\eta_0(X) \geq\frac{1}{2}-\frac{\tau}{2\pi_0}\Big\}.
   \end{align}

    \noindent \textbf{Step 1: Upper bound on $(I)$.} To control $(I)$, by the Dvoretzky--Kiefer--Wolfowitz inequality \citep{dvoretzky1956asymptotic,massart1990tight}, we have, for any $a \in \{0,1\}$ and $t_1 >0$,
   \begin{align*}
       \mathbb{P}\Big[\sup_{\tau \in \mathbb{R}}\Big|&\frac{1}{\breve{n}_a}\sum_{y\in\{0,1\}}\sum_{i=1}^{\breve{n}_{a,y}}\indc\Big\{\widetilde{\eta}_a(\breve{X}^i_{a,y}) \geq\frac{1}{2}+\frac{\tau(2a-1)}{2\widetilde{\pi}_a}\Big\}\\
       & - \mathbb{P}_{X|A=a}\Big\{\widetilde{\eta}_a(X) \geq\frac{1}{2}+\frac{\tau(2a-1)}{2\widetilde{\pi}_a}\Big|\mathcal{D}\Big\}\Big| \geq t_1 \Big|\mathcal{D}\Big] \leq C_1\exp(-2n_at_1^2).
   \end{align*}
   Moreover, by standard Gaussian tail properties \citep[e.g.~Proposition 2.1.2 in][]{vershynin2018high}, we have that for any $t_2 >0$,
   \begin{align*}
       \mathbb{P}(|\breve{w}| \geq t_2)\leq C_2 \exp(-t_2^2/\sigma_2^2),
   \end{align*}
   where $\sigma_2 = 2\sqrt{2\log(1.25/\delta)}/(\{\breve{n}_0\wedge \breve{n}_1\}\epsilon)$ is given in \Cref{alg_fair_cdp}. Taking 
   \[t_1 = C_3\sqrt{\frac{\log(1/\eta)}{\{\breve{n}_0 \wedge \breve{n}_1\}}}, \;\; \text{and}\;\; t_2 = C_4\sqrt{\frac{\log(1/\delta)\log(1/\eta)}{\{\breve{n}_0 \wedge \breve{n}_1\}^2\epsilon^2}},\]
   by a union bound argument, it holds with probability at least $1-\eta/2$ that
   \begin{align*}
       (I) \leq C_5\Big\{\sqrt{\frac{\log(1/\eta)}{\{\breve{n}_0 \wedge \breve{n}_1\}}}+\sqrt{\frac{\log(1/\delta)\log(1/\eta)}{\{\breve{n}_0 \wedge \breve{n}_1\}^2\epsilon^2}}\Big\}.
   \end{align*}

   \noindent \textbf{Step 2: Upper bound on $(II)$.} Denote $T^*_a =1/2+ \tau^*(2a-1)/(2\pi_a)$. We have that 
    \begin{align*}
       &\sup_{|\kappa| \leq \varpi}\Big|\mathbb{P}_{X|A=1}\Big\{\widetilde{\eta}_1(X) \geq \frac{1}{2}+\frac{\tau^*+\kappa}{2\widetilde{\pi}_1}\Big|\mathcal{D}\Big\} - \mathbb{P}_{X|A=1}\Big\{\eta_1(X) \geq\frac{1}{2}+\frac{\tau^*+\kappa}{2\pi_1}\Big\}\Big|\\
       \leq \;& \sup_{|\kappa| \leq \varpi}\int \Big|\indc\Big\{\widetilde{\eta}_1(x) \geq \frac{1}{2}+\frac{\tau^*}{2\widetilde{\pi}_1} + \frac{\kappa}{2\widetilde{\pi}_1}\Big\} - \indc\Big\{\eta_1(x) \geq T^*_1 + \frac{\kappa}{2\pi_1}\Big\} \Big|\;\mathbb{d}\mathbb{P}_{X|A=1,\mathcal{D}}(x)\\
       \leq \;& \sup_{|\kappa| \leq \varpi}\int \indc\Big\{\Big |\eta_1(x)-T^*_1 - \frac{\kappa}{2\pi_1}\Big| \leq \|\widetilde{\eta}_1 - \eta_1\|_{\infty} + \frac{|\tau^* +\kappa|}{2}\cdot\Big|\frac{1}{\widetilde{\pi}_1}-\frac{1}{\pi_1}\Big| \Big\}\;\mathbb{d}\mathbb{P}_{X|A=1,\mathcal{D}}(x) \\
       \leq \;&\sup_{|\kappa| \leq \varpi}\int \indc\Big\{\Big |\eta_1(x)-T^*_1 - \frac{\kappa}{2\pi_1}\Big| \leq \|\widetilde{\eta}_1 - \eta_1\|_{\infty} + \frac{|\tau^* |+\varpi}{2}\cdot\Big|\frac{1}{\widetilde{\pi}_1}-\frac{1}{\pi_1}\Big| \Big\}\;\mathbb{d}\mathbb{P}_{X|A=1,\mathcal{D}}(x) \\
       \leq \;& C_m \Big( \|\widetilde{\eta}_1 - \eta_1\|_{\infty} + \frac{|\tau^*| +\varpi}{2} \cdot \Big|\frac{1}{\widetilde{\pi}_1}-\frac{1}{\pi_1}\Big|\Big)^\gamma,
   \end{align*}
   where the last inequality follows from \Cref{a_posterior}\ref{a_posterior_margin}.  Consider the following events:
   \[\mathcal{E}_1 = \Big\{\|\widetilde{\eta}_a - \eta_a\|_{\infty} \leq C_6\Big\{\sqrt{\frac{\log\{(h^{-d}\vee n)/\eta\}}{nh^d}}+\frac{\sqrt{\log(1/\delta)\log(1/
        \eta)}}{n\epsilon h^d}+h^\beta\Big\},\; a\in \{0,1\}\Big\}\]
    \[\mathcal{E}_2 = \Big\{|\widetilde{\pi}_a - \pi_a| \leq C_7\Big\{\sqrt{\frac{\log(1/\eta)}{n}}+\sqrt{\frac{\log(1/\delta)\log(1/\eta)}{n^2\epsilon^2}}\Big\},\; a\in \{0,1\}\Big\}\]
    and 
    \[\mathcal{E}_3 = \{C_8 n \leq n_a \leq n\;\; \text{and}\;\; C_9\breve{n} \leq \breve{n}_a \leq \breve{n}, \; a\in \{0,1\}\}.\]
    
    By Lemmas \ref{l_eta_est_sup}, \ref{l_pi_a_est} and \ref{l_n_a_n_same_order} and a union bound argument, we have that $\mathbb{P}(\mathcal{E}_1 \cap \mathcal{E}_2 \cap \mathcal{E}_3) \geq 1-\eta/4$. On the event $\mathcal{E}_2$, it holds from the triangle inequality that 
    \begin{align*}
        \widetilde{\pi}_1 \geq \pi_1 - C_7\Big\{\sqrt{\frac{\log(1/\eta)}{n}}+\sqrt{\frac{\log(1/\delta)\log(1/\eta)}{n^2\epsilon^2}}\Big\}\Big\} \geq \frac{C_\pi}{2},
    \end{align*}
    where the last inequality follows from \Cref{a_prob}\ref{a_prob_pi_a}. Consequently, we have that 
    \begin{align}\label{l_DD_est_err_eq4}
        \Big|\frac{1}{\widetilde{\pi}_1}-\frac{1}{\pi_1}\Big| = \frac{|\widetilde{\pi}_1 - \pi_1|}{\widetilde{\pi}_1 \cdot \pi_1} \leq \frac{2}{C_\pi^2}|\widetilde{\pi}_1 - \pi_1|.
    \end{align}
    Thus, we have that 
    \begin{align} \notag
        &\Big( \|\widetilde{\eta}_1 - \eta_1\|_{\infty} + \frac{|\tau^*| +\varpi}{2} \cdot \Big|\frac{1}{\widetilde{\pi}_1}-\frac{1}{\pi_1}\Big|\Big)^\gamma\\
        \leq \;& C_{10}\Big\{\sqrt{\frac{\log\{(h^{-d}\vee n)/\eta\}}{nh^d}}+\frac{\sqrt{\log(1/\delta)\log(1/
        \eta)}}{n\epsilon h^d}+h^\beta\\ \notag
        & \hspace{3cm}+\sqrt{\frac{\log(1/\eta)}{n}}+\sqrt{\frac{\log(1/\delta)\log(1/\eta)}{n^2\epsilon^2}}\Big\}^\gamma\\ \label{l_DD_est_err_eq3}
        \leq\;& C_{11}\Big\{\sqrt{\frac{\log\{(h^{-d}\vee n)/\eta\}}{nh^d}}+\frac{\sqrt{\log(1/\delta)\log(1/
        \eta)}}{n\epsilon h^d}+h^\beta\Big\}^\gamma,
    \end{align}
    where the first inequality follows from \Cref{l_range_tau} and the last inequality follows on the event $\mathcal{E}_3$. Similarly, we have that 
    \begin{align*}
        &\sup_{|\kappa| \leq \varpi}\Big|\mathbb{P}_{X|A=0}\Big\{\widetilde{\eta}_0(X) \geq \frac{1}{2}-\frac{\tau^*+\kappa}{2\widetilde{\pi}_0}\Big|\mathcal{D}\Big\} - \mathbb{P}_{X|A=1}\Big\{\eta_0(X) \geq\frac{1}{2}-\frac{\tau^*+\kappa}{2\pi_0}\Big\}\Big|\\
         \leq \;& \sup_{|\kappa| \leq \varpi}\int \Big|\indc\Big\{\widetilde{\eta}_0(x) \geq \frac{1}{2}-\frac{\tau^*}{2\widetilde{\pi}_0} - \frac{\kappa}{2\widetilde{\pi}_0}\Big\} - \indc\Big\{\eta_0(x) \geq T^*_0 - \frac{\kappa}{2\pi_0}\Big\} \Big|\;\mathbb{d}\mathbb{P}_{X|A=0,\mathcal{D}}(x)\\
         \leq \;&\sup_{|\kappa| \leq \varpi} \int \indc\Big\{\Big |\eta_0(x)-T^*_0 +\frac{\kappa}{2\pi_0} \Big| \leq \|\widetilde{\eta}_0 - \eta_0\|_{\infty} + \frac{|\tau^* +\kappa|}{2} \cdot \Big|\frac{1}{\widetilde{\pi}_0}-\frac{1}{\pi_0}\Big|\Big\}\;\mathbb{d}\mathbb{P}_{X|A=0,\mathcal{D}}(x)\\
         \leq \;& C_m \Big( \|\widetilde{\eta}_0 - \eta_0\|_{\infty} + \frac{|\tau^*| +\varpi}{2} \cdot \Big|\frac{1}{\widetilde{\pi}_0}-\frac{1}{\pi_0}\Big|\Big)^\gamma.
    \end{align*}
    By a similar argument leading to \eqref{l_DD_est_err_eq3}, we have that with probability at least $1-\eta/4$, we have that 
    \begin{align*}
        &\sup_{|\kappa| \leq \varpi}\Big|\mathbb{P}_{X|A=0}\Big\{\widetilde{\eta}_0(X) \geq \frac{1}{2}-\frac{\tau^*+\kappa}{2\widetilde{\pi}_0}\Big|\mathcal{D}\Big\} - \mathbb{P}_{X|A=1}\Big\{\eta_0(X) \geq\frac{1}{2}-\frac{\tau^*+\kappa}{2\pi_0}\Big\}\Big| \\
        \leq \;& C_{12}\Big\{\sqrt{\frac{\log\{(h^{-d}\vee n)/\eta\}}{nh^d}}+\frac{\sqrt{\log(1/\delta)\log(1/
        \eta)}}{n\epsilon h^d}+h^\beta\Big\}^\gamma.
    \end{align*}
    By a union bound argument, we have with probability at least $1-\eta$ that 
    \begin{align*}
        (II) \leq\;& C_{13}\Big[\Big\{\sqrt{\frac{\log\{(h^{-d}\vee n)/\eta\}}{nh^d}}+\frac{\sqrt{\log(1/\delta)\log(1/
        \eta)}}{n\epsilon h^d}+h^\beta \Big\}^\gamma \\
        & \hspace{1cm} +\sqrt{\frac{\log(1/\eta)}{\breve{n}}}+\sqrt{\frac{\log(1/\delta)\log(1/\eta)}{\breve{n}^2\epsilon^2}}\Big].
    \end{align*}

\end{proof}

\begin{lemma} \label{l_range_tau}
    It holds that $|\tau^*_{\DD,\alpha}| \leq \min\{\pi_0, \pi_1\}$.
\end{lemma}

\begin{proof}
    Recall that $g_{\DD,\tau}(x,a) = \indc\{\eta_{a}(x) \geq 1/2+ \tau(2a-1)/(2\pi_a)\}$.
    
    Note that if $\tau = \pi_1$, then $g_{\DD,\tau}(x,1) = 0$ for all $x \in [0,1]^d$. Also, 
    \begin{align} \notag
        \DD(\pi_1) &=  -\mathbb{P}_{X|A=0}\Big\{\eta_0(X) \geq\frac{1}{2}-\frac{\pi_1}{2\pi_0}\Big\}\\ \label{l_range_tau_eq1}
        & \begin{cases}
            = -1, \quad &\pi_1\geq \pi_0,\\
            \leq 0, & \pi_1 < \pi_0.
        \end{cases}
    \end{align}

    If $\tau = \pi_0$, then $g_{\DD,\tau}(x,0) = 1$ for all $x \in [0,1]^d$. Then it holds that 
     \begin{align} \label{l_range_tau_eq2}
        \DD(\pi_0) &= \mathbb{P}_{X|A=1}\Big\{\eta_1(X) \geq\frac{1}{2}+\frac{\pi_0}{2\pi_1}\Big\} - 1 \leq 0.
    \end{align}

    If $\tau = -\pi_1$, then $g_{\DD,\tau}(x,1) = 1$ for all $x \in [0,1]^d$. Then it holds that 
    \begin{align*}
         \DD(-\pi_1) &=  1-\mathbb{P}_{X|A=0}\Big\{\eta_0(X) \geq\frac{1}{2}+\frac{\pi_1}{2\pi_0}\Big\} \geq 0.   
    \end{align*}

    If $\tau = -\pi_0$, then $g_{\DD,\tau}(x,0) = 0$ for all $x \in [0,1]^d$. Then it holds that 
    \begin{align} \notag
         \DD(-\pi_0) &= \mathbb{P}_{X|A=1}\Big\{\eta_1(X) \geq\frac{1}{2}-\frac{\pi_0}{2\pi_1}\Big\}\\ \notag
        & \begin{cases}
            \geq 0, \quad &\pi_1 < \pi_0,\\
            = 1, & \pi_1 \geq \pi_0.
        \end{cases}
    \end{align}

     Thus by \Cref{prop_D_and_R}(i), \eqref{l_range_tau_eq1} and \eqref{l_range_tau_eq2}, if $\DD(0)> \alpha = \DD(\tau^*_{\DD,\alpha}) > 0 \geq \{\DD(\pi_1)\vee \DD(\pi_0)\}$, we then have $0 < \tau^*_{\DD,\alpha} < \min\{\pi_0, \pi_1\}$.
    
    Similarly, if $\DD(0) < -\alpha = \DD(\tau^*_{\DD,\alpha}) <0 \leq \{\DD(-\pi_1)\wedge \DD(-\pi_0)\}$, we then have that $\max\{-\pi_0, -\pi_1\}<\tau^*_{\DD,\alpha}<0$

    Furthermore, we have that if $|\DD(0)| \leq \alpha$, then $\tau^*_{\DD,\alpha} = 0$. Thus, the lemma follows.
    
\end{proof}

\begin{lemma} \label{l_DD_non_increase}
    $\widetilde{\DD}(\cdot)$ and $\widehat{\DD}(\cdot)$ are non-increasing functions in $\tau$.
\end{lemma}

\begin{proof}
    For any $\tau_1 < \tau_2$, we have that 
    \begin{align*}
        &\widetilde{\DD}(\tau_1)-  \widetilde{\DD}(\tau_2) = \widehat{\DD}(\tau_1)-  \widehat{\DD}(\tau_2) \\
        =\;&  \frac{1}{\breve{n}_1}\sum_{y\in\{0,1\}}\sum_{i=1}^{\breve{n}_{1,y}}\indc\Big\{\widetilde{\eta}_1(\breve{X}^i_{1,y}) \geq \frac{1}{2}+\frac{\tau_1}{2\widetilde{\pi}_1}\Big\} -\frac{1}{\breve{n}_0}\sum_{y\in\{0,1\}}\sum_{i=1}^{\breve{n}_{0,y}}\indc\Big\{\widetilde{\eta}_0(\breve{X}^i_{0,y}) \geq \frac{1}{2}-\frac{\tau_1}{2\widetilde{\pi}_0}\Big\} \\
        &-\frac{1}{\breve{n}_1}\sum_{y\in\{0,1\}}\sum_{i=1}^{\breve{n}_{1,y}}\indc\Big\{\widetilde{\eta}_1(\breve{X}^i_{1,y}) \geq \frac{1}{2}+\frac{\tau_2}{2\widetilde{\pi}_1}\Big\} +\frac{1}{\breve{n}_0}\sum_{y\in\{0,1\}}\sum_{i=1}^{\breve{n}_{0,y}}\indc\Big\{\widetilde{\eta}_0(\breve{X}^i_{0,y}) \geq \frac{1}{2}-\frac{\tau_2}{2\widetilde{\pi}_0}\Big\}\\
        =\;& \frac{1}{\breve{n}_1}\sum_{y\in\{0,1\}}\sum_{i=1}^{\breve{n}_{1,y}}\indc\Big\{\frac{1}{2}+\frac{\tau_1}{2\widetilde{\pi}_1}\leq \widetilde{\eta}_1(\breve{X}^i_{1,y}) < \frac{1}{2}+\frac{\tau_2}{2\widetilde{\pi}_1}\Big\} \\
        &+ \frac{1}{\breve{n}_0}\sum_{y\in\{0,1\}}\sum_{i=1}^{\breve{n}_{0,y}}\indc\Big\{\frac{1}{2}-\frac{\tau_2}{2\widetilde{\pi}_0} \leq \widetilde{\eta}_0( \breve{X}^i_{0,y}) < \frac{1}{2}-\frac{\tau_1}{2\widetilde{\pi}_0}\Big\}\\
        \geq \;& 0.
    \end{align*}
    Thus, the lemma follows.
\end{proof}

\subsubsection{Upper bound on \texorpdfstring{$\|\widetilde{p}_{X|A=a}- p_{X|A=a}\|_\infty$}{}}
\begin{lemma} \label{l_p_x_a_est_sup}
    Under same assumptions as in \Cref{t_fair_risk}, it holds for any $\eta \in (0,1/2)$ and $a \in \{0,1\}$ that 
    \begin{align*}
        \mathbb{P}\Bigg[\|\widetilde{p}_{X|A=a}- p_{X|A=a}\|_\infty \geq C_1\Bigg\{\sqrt{\frac{\log\{(h^{-d}\vee n_a)/\eta\}}{n_ah^d}}+\frac{\sqrt{\log(1/\delta)\log(1/
        \eta)}}{n_a\epsilon h^d}+h^\beta\Bigg\}\Bigg] \leq \eta.
    \end{align*}
\end{lemma}
\begin{proof}
    Note that
    \begin{align*}
        &\|\widetilde{p}_{X|A=a}- p_{X|A=a}\|_\infty\\
        \leq \;& \sup_{x \in [0,1]^d}\Big|\frac{1}{n_a}\sum_{y \in \{0,1\}}\sum_{i=1}^{n_{a,y}}K_{h}(X_{a,y}^i -x)-p_{X|A=a}(x|a)\Big|+\frac{8\sqrt{2C_K\log(5/\delta)}}{n_a\epsilon h^d} \sup_x |W_{1,a}(x)|\\
        \leq \;& \sup_{x \in [0,1]^d}\Big|\frac{1}{n_a}\sum_{y \in \{0,1\}}\sum_{i=1}^{n_{a,y}}K_{h}(X_{a,y}^i -x)-\int_u K_{h}(u -x)p_{X|A=a}(u|a)\;\mathrm{d}u\Big|\\
        & + \sup_x\Big|\int_u K_{h}(u -x)p_{X|A=a}(u|a)\;\mathrm{d}u -p_{X|A=a}(x|a)\Big| \\
        &+ \frac{8\sqrt{2C_K\log(5/\delta)}}{n_a\epsilon h^d} \sup_x |W_{1,a}(x)|\\
        = \;& (I)+(II)+(III).
    \end{align*}

    To control $(I)$, by \Cref{l_est_p_xa_kernel_sup}, we have that 
    \begin{align*}
        \mathbb{P}\Bigg\{(I) \geq C_1\sqrt{\frac{\log\{(h^{-d}\vee n_a)/\eta\}}{n_ah^d}}\Bigg\} \leq \frac{\eta}{2}.
    \end{align*}

    To control $(II)$, by Asssumptions \ref{a_prob}\ref{a_prob_p_x_a} and \ref{a_kernel}\ref{a_kernel_adaptive}, it holds that 
    \begin{align*}
        (II) \leq C_{\mathrm{adp}}h^\beta.
    \end{align*}

    To control $(III)$, by \Cref{l_sup_gaussian}, it holds that 
    \begin{align*}
        \mathbb{P}\Bigg\{(III) \geq \frac{C_2\sqrt{\log(1/\delta)\log(1/
        \eta)}}{n_a\epsilon h^d}\Bigg\} \leq \frac{\eta}{2}.
    \end{align*}
    The lemma thus follows by a union bound argument. 

\end{proof}

\begin{lemma} \label{l_p_xa_const_order}
    Under the same assumption as in \Cref{t_fair_risk}, it holds for any $\eta \in (0,1/2)$ and $a \in \{0,1\}$ that 
    \begin{align*}
        \mathbb{P}\Big\{C_1 \leq \inf_{x \in [0,1]^d} \widetilde{p}_{X|A=a}(x) \leq \sup_{x \in [0,1]^d} \widetilde{p}_{X|A=a}(x) \leq C_2 \Big\} \geq 1- \eta.
    \end{align*}
\end{lemma}

\begin{proof}
    Consider the event
    \begin{align*}
        \mathcal{E}_1 = \Bigg\{\|\widetilde{p}_{X|A=a}- p_{X|A=a}\|_\infty \leq C_1\Bigg\{\sqrt{\frac{\log\{(h^{-d}\vee n_a)/\eta\}}{n_ah^d}}+\frac{\sqrt{\log(1/\delta)\log(1/\eta)}}{n_a\epsilon h^d}+h^\beta\Bigg\}\Bigg\}.
    \end{align*}
    By \Cref{l_p_x_a_est_sup}, it holds that $\mathbb{P}(\mathcal{E}_1) \geq 1-\eta$. The rest of the proof is constructed on the event $\mathcal{E}_1$. Denote
    \begin{align*}
        t = C_1\Bigg\{\sqrt{\frac{\log\{(h^{-d}\vee n_a)/\eta\}}{n_ah^d}}+\frac{\sqrt{\log(1/\delta)\log(1/\eta)}}{n_a\epsilon h^d}+h^\beta\Bigg\},
    \end{align*}
    $x_* = \min\arginf_{x \in [0,1]^d} \widetilde{p}_{X|A=a}(x)$ and $x^* = \min\argmsup_{x \in [0,1]^d} \widetilde{p}_{X|A=a}(x)$. By the triangle inequality, we have that 
    \begin{align*}
        \widetilde{p}_{X|A=a}(x^*) \leq p_{X|A=a}(x^*) + t \leq C_3,
    \end{align*}
    where the second inequality follows from the fact that $p_{X|A=a}(x|a)$ is H\"{o}lder continuous over a compact interval in \Cref{a_prob}\ref{a_prob_p_x_a}, hence bounded. Similarly, we have that 
    \begin{align*}
         \widetilde{p}_{X|A=a}(x_*) \geq p_{X|A=a}(x_*) - t \geq C_p - t \geq C_p/2,
    \end{align*}
    where the second inequality follows from \Cref{a_prob}\ref{a_prob_p_x_a}. The lemma thus follows.

\end{proof}

\begin{lemma} \label{l_est_p_xa_kernel_sup}
    Under the same assumptions as in \Cref{t_fair_risk}, it holds for any $\eta \in (0,1/2)$ and $a \in \{0,1\}$ that 
    \begin{align*}
        \mathbb{P}\Big\{\sup_{x \in [0,1]^d}\Big|\frac{1}{n_a}\sum_{y \in \{0,1\}}\sum_{i=1}^{n_{a,y}} K_{h}(X_{a,y}^i -x)-\int_u K_{h}(u -x)&p_{X|A=a}(u|a)\;\mathrm{d}u\Big|\\
        &\geq C_1\sqrt{\frac{\log\{(h^{-d}\vee n_a)/\eta\}}{n_ah^d}}\Big\} \leq \eta.
    \end{align*}
\end{lemma}

\begin{proof}
    \noindent \textbf{Step 1: upper bound for any $x \in [0,1]^d$.} Note that for any $x\in [0,1]^d$, under \Cref{a_kernel}\ref{a_kernel_bounded}, we have that 
    \begin{align*}
        |K_{h}(X_{a,y}^i -x)| = \bigg|\frac{1}{h^d}K\bigg(\frac{X_{a,y}^i -x}{h}\bigg)\bigg| \leq C_1h^{-d},
    \end{align*}
    and also
    \begin{align*}
        \mathbb{E}\Big\{\frac{1}{n_a}\sum_{y \in \{0,1\}}\sum_{i=1}^{n_{a,y}}K_{h}(X_{a,y}^i -x)\Big\}-\int_u K_{h}(u -x)p_{X|A=a}(u|a)\;\mathrm{d}u=0
    \end{align*}
    by construction. To compute the variance, we observe that for each $y \in \{0,1\}$ and $i \in [n_{a,y}]$,
    \begin{align*}
        \var\big\{K_{h}(X_{a,y}^i -x)\big\} \leq\;& \mathbb{E}\Big\{\frac{1}{h^{2d}} K\Big(\frac{X_{a,y}^i -x}{h}\Big)^2\Big\} \leq \frac{1}{h^{2d}} \int K\Big(\frac{u -x}{h}\Big)^2 p_{X|A=a}(u|a) \;\mathrm{d}u\\
        \leq\;& \frac{C_2}{h^{2d}} \int h^d K(t)^2 \;\mathrm{d}t  \leq C_3h^{-d},
    \end{align*}
    where the third inequality follows from \Cref{a_prob}\ref{a_prob_p_x_a} and the last inequality follows from \Cref{a_kernel}\ref{a_kernel_bounded}. Therefore, by the Bernstein inequality for bounded distributions \citep[e.g.~Theorem 2.8.4 in][]{vershynin2018high}, it holds that for any $0< t <1$,
    \begin{align} \notag
         &\mathbb{P}\Big\{\Big|\frac{1}{n_a}\sum_{y \in \{0,1\}}\sum_{i=1}^{n_{a,y}}K_{h}(X_{a,y}^i -x)-\int_u K_{h}(u -x)p_{X|A=a}(u|a)\;\mathrm{d}u\Big| \geq t\Big\}\\ \label{l_est_p_xa_kernel_sup_eq1}
         \leq \;& 2\exp\Big(-\frac{n_at^2/2}{C_3h^{-d}+ C_1h^{-d}t/3}\Big) \leq C_4 \exp(-n_ah^{d}t^2).
    \end{align}

     \noindent \textbf{Step 2: upper bound on the sup-norm.} 
     Denote 
     \[\mathcal{T}_a(x) = \frac{1}{n_a}\sum_{y \in \{0,1\}}\sum_{i=1}^{n_{a,y}}K_{h}(X_{a,y}^i -x)-\int_u K_{h}(u -x)p_{X|A=a}(u|a)\;\mathrm{d}u.\]
     For any $\nu \in (0,1)$, denote $Q=\mathcal{N}(\nu,\|\cdot\|_2, [0,1]^d)$ the cardinality of an $\nu$-net of $[0,1]^d$ with respect to the $\ell_2$ norm. By standard covering lemma \citep[e.g.~Proposition 4.2.12 in][]{vershynin2018high}, we have that $Q \leq (C_1\sqrt{d}/\nu)^d$. We further let $\{z_1, \ldots, z_Q\}$ be a $\nu$-covering of $[0,1]^d$ and also let $S_j = \{x \in [0,1]^d: \|x-z_j\|_2 \leq \nu\}$, $j \in [Q]$. It then follows that 
     \begin{align*}
         \sup_{x \in [0,1]^d} |\mathcal{T}_a(x)| \leq\;& \max_{j \in [Q]} \Big\{|\mathcal{T}_a(z_j)| + \sup_{x \in S_j}|\mathcal{T}_a(x)- \mathcal{T}_a(z_j)| \Big\}\\
         \leq\;& \max_{j \in [Q]} |\mathcal{T}_a(z_j)| + \max_{j \in [Q]} \sup_{x \in S_j}|\mathcal{T}_a(x)- \mathcal{T}_a(z_j)|\\
         =\;& (I) + (II).
     \end{align*}
     To control $(I)$, by \eqref{l_est_p_xa_kernel_sup_eq1} and a union bound argument, we have that
     \begin{align} \label{l_est_p_xa_kernel_sup_eq2}
         \mathbb{P}\Big\{(I) \geq  \frac{t}{2}\Big\} \leq C_5\exp\Big\{\log(Q)-n_ah^{d}t^2\Big\}.
     \end{align}
     To control $(II)$, note that 
     \begin{align} \notag
         (II) \leq \;& \max_{j \in [Q]} \sup_{x \in S_j} \Big|\frac{1}{n_a}\sum_{y \in \{0,1\}}\sum_{i=1}^{n_{a,y}}K_{h}(X_{a,y}^i -x) -\frac{1}{n_a}\sum_{y \in \{0,1\}}\sum_{i=1}^{n_{a,y}}K_{h}(X_{a,y}^i -z_j)\Big|\\ \notag
         &+ \max_{j \in [Q]} \sup_{x \in S_j}\Big|\int_u K_{h}(u -x)p_{X|A=a}(u|a)\;\mathrm{d}u- \int_u K_{h}(u -z_j)p_{X|A=a}(u|a)\;\mathrm{d}u\Big|\\ \notag
         \leq \;& 2C_{\mathrm{Lip}}h^{-(d+1)} \max_{j \in [Q]} \sup_{x \in S_j} \|x-z_j\|_2\\ \label{l_est_p_xa_kernel_sup_eq3}
         \leq \;& 2C_{\mathrm{Lip}}h^{-(d+1)}\nu = t/2,
     \end{align}
     where the second inequality follows from \Cref{a_kernel}\ref{a_kernel_lipschitz} and the last equality follows by setting $\nu =h^{d+1}t/(4 C_{\textrm{Lip}})$. Therefore, applying a union bound argument to \eqref{l_est_p_xa_kernel_sup_eq2} and \eqref{l_est_p_xa_kernel_sup_eq3}, we have that 
     \begin{align*}
         \mathbb{P}\Big\{ \sup_{x \in [0,1]^d} |\mathcal{T}_a(x)| \geq t\Big\} \leq C_6\exp\Big\{d\log\Big(\frac{\sqrt{d}}{h^{d+1}t}\Big)- n_ah^{d}t^2\Big\}.
     \end{align*}
     The lemma thus follows by taking
     \[t = C_7\sqrt{\frac{\log\{(h^{-d}\vee n_a)/\eta\}}{n_ah^d}}.\]
\end{proof}

\begin{lemma} \label{l_sup_gaussian}
    For any $a \in \{0,1\}$, $\eta \in (0,1/2)$ and $k \in \{1,2\}$, it holds that
    \begin{align*}
        \mathbb{P}\Big\{\sup_{x \in [0,1]^d} |W_{k,a}(x)| \geq C_1\sqrt{\log(1/
        \eta)}\Big\} \leq \eta.
    \end{align*}
\end{lemma}

\begin{proof}
    The proof follows from a similar argument as the one in the proof of Lemma 3 in \citet{auddy2025minimax}. We first control $\mathbb{E}\{\sup_{x \in [0,1]^d} W_{k,a}(x)\}$. By Dudley's theorem \citep[e.g.][]{dudley2016vn}, we have that 
    \begin{align*}
        \mathbb{E}\left\{\sup_{x \in [0,1]^d} W_{k,a}(x)\right\} \leq K(0) + C_1 \int_{0}^1 \sqrt{\log\Big\{\Big(1+\frac{2}{\nu}\Big)^d\Big\}} \; \mathrm{d}\nu \leq C_2 \sqrt{d} = C_3.
    \end{align*}
    Also, by the construction of the process $W_{k,a}$ in \Cref{alg_fair_cdp}, we have that 
    \begin{align*}
        \sup_{x \in [0,1]^d} \mathbb{E}\{W_{k,a}(x)^2\} =  K(0) \leq C_K,
    \end{align*}
    where the inequality follow from \Cref{a_kernel}\ref{a_kernel_bounded}. Therefore, by the Borell--TIS inequality \citep[e.g.~Theorem 2.1.1 in][]{adler2007random}, it holds that
    \begin{align*}
        \mathbb{P}\Big[\Big|\sup_{x \in [0,1]^d} W_{k,a}(x) - \mathbb{E}\big\{\sup_{x \in [0,1]^d} W_{k,a}(x)\big\}\Big|\geq t \Big] \leq 2\exp\Big(-\frac{t^2}{2C_K}\Big).
    \end{align*}
    Therefore, by the triangle inequality, it holds that
    \begin{align*}
        \mathbb{P}\Big\{\Big|\sup_{x \in [0,1]^d} W_{k,a}(x)\Big| \geq C_4\sqrt{\log(1/
        \eta)}\Big\} \leq \eta.
    \end{align*}
    Similarly, we can also prove a result for the process $-W_{k,a}$. Thus, the lemma follows by a union bound argument.
    
\end{proof}

\subsubsection{Upper bound on \texorpdfstring{$|\widetilde{\pi}_a - \pi_a|$}{}}

\begin{lemma} \label{l_pi_a_est}
    Under the same assumptions as in \Cref{t_fair_risk}, it holds for any $\eta \in (0,1/2)$ and $a \in \{0,1\}$ that 
    \begin{align*}
        \mathbb{P}\Big[|\widetilde{\pi}_a - \pi_a| \geq C_1\Big\{\sqrt{\frac{\log(1/\eta)}{n}}+\sqrt{\frac{\log(1/\delta)\log(1/\eta)}{n^2\epsilon^2}}\Big\}\Big] \leq \eta.
    \end{align*}
\end{lemma}

\begin{proof}
    For $a \in \{0, 1\}$, consider the sequence of bounded random variables $\{\indc\{A_i =a\}\}_{i=1}^n$.  We have that 
    \begin{align*}
        \widetilde{\pi}_a =  \frac{n_{a,0}+n_{a,1}}{n}+ w_a = \frac{1}{n}\sum_{i=1}^n \indc\{A_i =a\} + w_a.
    \end{align*}
    By Hoeffding’s inequality for general bounded random variables \citep[e.g.~Theorem 2.2.6 in][]{vershynin2018high}, we have that, for $t_1 > 0$,
    \begin{align*}
        \mathbb{P}\Big(\Big|\frac{1}{n}\sum_{i=1}^n \indc\{A_i =a\} - \pi_a \Big| \geq t_1\Big) \leq C_1\exp(-t_1^2n).
    \end{align*}
    Moreover, by standard Gaussian tail properties \citep[e.g.~Proposition 2.1.2 in][]{vershynin2018high}, we have that for any $t_2 >0$,
   \begin{align*}
       \mathbb{P}(|w_a| \geq t_2)\leq C_2 \exp(-t_2^2/\sigma_1^2),
   \end{align*}
   where $\sigma_1 = 4\sqrt{2\log(5/\delta)}/(n\epsilon)$ is given in \Cref{alg_fair_cdp}. Taking 
   \[t_1 = C_3\sqrt{\frac{\log(1/\eta)}{n}}, \;\; \text{and}\;\; t_2 = C_4\sqrt{\frac{\log(1/\delta)\log(1/\eta)}{n^2\epsilon^2}},\]
   the lemma thus follows by a union bound argument.
\end{proof}

\begin{lemma} \label{l_n_a_n_same_order}
    For any $\min\{n, \breve{n}\} \geq 4C_1^2\log(1/\eta)/C_\pi^2$, it holds for any $\eta \in (0,1/2)$ that 
    \[\mathbb{P}\Big\{C_2 n \leq n_a \leq n, \;\; \text{and} \;\; C_3 \breve{n} \leq \breve{n}_a \leq \breve{n}, \; a\in \{0,1\} \Big\} \geq 1-\eta.\]
\end{lemma}

\begin{proof}
    For $a \in \{0,1\}$, consider the sequence of indicator random variables $\{\indc\{A_i =a\}\}_{i=1}^n$, then we have that $n_a = \sum_{i=1}^n \indc\{A_i =a\}$ and $\mathbb{E}[\sum_{i=1}^n \indc\{A_i =a\}] = \pi_a n$. Consequently, by Hoeffding’s inequality for general bounded random variables \citep[e.g.~Theorem 2.2.6 in][]{vershynin2018high}, it holds for any $t>0$ that
    \begin{align*}
        \mathbb{P}(|n_a - \pi_a n| \geq t) \leq C_1\exp\Big(-\frac{t^2}{n}\Big).
    \end{align*}
    Taking $t = C_2 \sqrt{n\log(1/\eta)}$, we have that with probability at least $1-\eta/4$ that 
    \[n_a \geq \pi_a n - C_2 \sqrt{n\log(1/\eta)} \geq C_\pi n - \frac{C_\pi n}{2} = \frac{C_\pi n}{2},\]
    where the second inequality follows from \Cref{a_prob}\ref{a_prob_pi_a} for any $n \geq 4C_2^2\log(1/\eta)/C_\pi^2$. The proof for the calibration data $\breve{\mathcal{D}}$ are similar. The lemma thus follows by applying a union bound argument.
    
\end{proof}

\section{Theoretical guarantee for Algorithm \ref{alg_fair_fdp}} 
\label{sec_appendix_fdp}
\subsection{Privacy guarantee}
\begin{proposition} \label{thm_fdp_guarantee}
    \Cref{alg_fair_fdp} is $(\bm{\epsilon},\bm{\delta})$-FDP.
\end{proposition}
\begin{proof}
    Most of the proof follows from the same argument as the proof of \Cref{t_cdp_privacy_guarantee}. We only include the difference here.

    By a similar argument as the proof of \Cref{t_cdp_privacy_guarantee}, in Step \textbf{S1}, $(\bm{\epsilon},\bm{\delta})$-FDP is guaranteed because for each site $s \in [S]$, $\{\widetilde{\pi}^s_a,\widetilde{p}^s_{X|A=a},\widetilde{p}^s_{X,Y|A=a}\}_{a\in\{0,1\}}$ together satisfies $(\epsilon,\delta)$-CDP.

    In step \textbf{S2}, namely \Cref{alg_threshold_search_fdp}, $(\bm{\epsilon},\bm{\delta})$-FDP is also guaranteed. Each level $\ell \in [M]$ of the binary tree from server $s$ and sensitive attribute $a$ is given by
    \begin{align*}
        b_{\ell,a}(\mathcal{D}^s) = \Big(\sum_{y\in\{0,1\}}\sum_{i=1}^{\breve{n}_{s,a,y}} \indc\{\tau_1 \leq Z_{s,i,a,y} \leq \tau_{2^{M-\ell}+1}\}, \ldots, \sum_{y\in\{0,1\}}\sum_{i=1}^{\breve{n}_{s,a,y}} \indc\{\tau_{(2^\ell-1)2^{M-\ell}+1} \leq Z_{s,i,a,y} \leq 1\}\Big),
    \end{align*}
    which is a vector of dimension $2^\ell$ containing counts for the bin $\{[\tau_{(k-1)2^{M-\ell}+1}$, $\tau_{k2^{M-\ell}+1}]\}_{k=1}^{2^\ell}$.
    When we append two trees together side by side, then if we change one entry in $\mathcal{D}^s$, only two entries in $ b_{\ell}(\mathcal{D}^s)=(b_{\ell,1}(\mathcal{D}^s),b_{\ell,0}(\mathcal{D}^s))$ will shift by one. Therefore, we have that $\sup_{\mathcal{D}^s \sim \mathcal{D}^{'s}}\|b_\ell(\mathcal{D}^s) -b_\ell(\mathcal{D}^{'s})\|_2^2 \leq 2$. Since this holds for all levels $\ell \in [M]$, by Lemma 30 in \citet{hung2025optimal}, releasing $\{\{N_{s, a, \ell,k}\}_{ \ell=1,k =1}^{M,2^\ell}\}_{a \in\{0,1\}}$ is therefore $(\epsilon_s,\delta_s)$-CDP for each server.
\end{proof}

\subsection{Disparity control}
\begin{proposition} \label{thm_fdp_disparity_control}
   Under the same assumptions as in \Cref{thm_privacy_fair}.\ref{thm_privacy_fair_fdp}, we have that
    \begin{align*}
        \mathbb{P}\big\{|\DD(\widetilde{f}_{\DD, \alpha})| \leq \alpha+ C_1\rho\big\} \geq 1-\eta,
    \end{align*}
    where $\rho$ is given in \eqref{prop_fdp_risk_eq1} and the probability in $\DD(\widetilde{f}_{\DD, \alpha})$ is taken over the test sample conditioning on training data and $\mathbb{P}$ is taken over the training data.
\end{proposition}

\begin{proof}
    By Lemmas \ref{l_fdp_tau_est}, \ref{l_fdp_DD_est_err} and the extra assumption that $\varpi \geq \rho^{1/\gamma}$, we have with probability at least $1-\eta$ that
    \begin{align*}
        |\widetilde{\DD}_{\downarrow}(\widetilde{\tau}) - \DD(\widetilde{\tau})| \leq C_1\rho.
    \end{align*}
    Thus, by triangle inequality and the design in \Cref{alg_threshold_search_fdp}, we have that 
    \begin{align*}
        |\DD(\widetilde{\tau})| \leq |\widetilde{\DD}_{\downarrow}(\widetilde{\tau})| + C_1\rho \leq \alpha+  C_2\rho.
    \end{align*}
    
\end{proof}

\subsection{Excess risk control} \label{sec_app_proof_t_fdp_risk}
\begin{proof}[Proof of \Cref{t_fdp_risk}]
    \Cref{t_fdp_risk} is a direct consequence of \Cref{prop_fdp_risk}. Note that with the choice of $\theta = \rho^{1/\gamma}$, it holds that 
    \[M = \log_2(\theta^{-1}) +1 = C_1 \log\Big(\sum_{s=1}^S \{N_s \wedge N_s^2\epsilon^s\}\Big),\]
    where $C_1 >0$ is a constant depending on $\gamma, d$ and $\beta$. Thus the effect of $M$ in \Cref{prop_fdp_risk} are only up to poly-logarithmic factors.
\end{proof}

\begin{proposition}[Fairness aware excess risk] \label{prop_fdp_risk}
    Suppose Assumptions \ref{a_prob}, \ref{a_kernel} and \ref{a_posterior} hold. Denote
    \begin{align}\notag
         \rho =\;& C_1 \Bigg[\Bigg\{\sqrt{\sum_{s=1}^S \frac{\nu^2_s\log(h^{-d}/\eta)}{n_{s}h^d}} +  \max_{s \in [S]} \frac{\nu_s \log(h^{-d}/\eta)}{n_{s}h^{d}}+ h^\beta+ \sqrt{\sum_{s=1}^S\frac{\nu^2_s\log(1/\delta_s)\log(h^{-1}/\eta)}{n^2_{s}\epsilon^2_s h^{2d}}}\Bigg\}^\gamma\\ \label{prop_fdp_risk_eq1}
        &+\sqrt{\sum_{s=1}^S \frac{\mu_s^2\log(1/\eta)}{\breve{n}_{s}}} + \max_{s \in [S]}\frac{\mu_s\log(1/\eta)}{\breve{n}_{s}}+\sqrt{\sum_{s=1}^S\frac{\mu_s^2M^4\log(1/\delta_s)\log(M/\eta)}{\breve{n}_{s}^2\epsilon_s^2}}\Bigg].
    \end{align}
    and we further assume that $\DD(0) \notin [\alpha - \zeta,\alpha] \cup [-\alpha,-\alpha+\zeta]$, where $\rho \leq \zeta < \alpha$. Then for any $\min\{n, \breve{n}\} \gtrsim \log(S/\eta)$, $\min_{s\in[S]} \breve{n}_{s,a}^2 \epsilon_s^2 \geq C_2M^2 \log(1/\delta_s)\log(S/\eta)$ and $\eta \in (0,1/2)$, it holds with probability at least $1-\eta$ that 
    \begin{align*}
        &d_{\mathrm{fair}}(\widetilde{f}_{\DD,\alpha}, f^*_{\DD,\alpha}) \\
        \lesssim \;& \indc\{\tau^*\neq 0\}\Bigg[\sqrt{\sum_{s=1}^S \frac{\mu_s^2\log(1/\eta)}{\breve{n}_{s}}} + \max_{s \in [S]}\frac{\mu_s\log(1/\eta)}{\breve{n}_{s}}+\sqrt{\sum_{s=1}^S\frac{\mu_s^2M^4\log(1/\delta_s)\log(M/\eta)}{\breve{n}_{s}^2\epsilon_s^2}}\Bigg]^{\frac{1+\gamma}{\gamma}}\\
        & +\Bigg[\sqrt{\sum_{s=1}^S \frac{\nu^2_s\log(h^{-d}/\eta)}{n_{s}h^d}} +  \max_{s \in [S]} \frac{\nu_s\log(h^{-d}/\eta)}{n_{s}h^{d}}+ h^\beta+ \sqrt{\sum_{s=1}^S\frac{\nu^2_s\log(1/\delta_s)\log(h^{-1}/\eta)}{n^2_{s}\epsilon^2_s h^{2d}}}\Bigg]^{1+\gamma}.
    \end{align*}
    
\end{proposition}

\begin{proof}
    Most of the proof follows from a similar argument used in the proof of \Cref{t_fair_risk}. We only include the difference here. To control the fairness-aware excess risk, using a similar argument leading to \eqref{t_fair_risk_eq1}, we have that
    \begin{align*}
        d_{\mathrm{fair}}(\widetilde{f}_{\DD,\alpha}, f^*_{\DD,\alpha}) \leq  C_1 \max_{a \in \{0,1\}} (\|\widetilde{\eta}_a- \eta_a\|_\infty + |\widetilde{T}_a - T_a^*|)^{1+\gamma}.
    \end{align*}
    Consider the following event, 
    \begin{align*}
       \mathcal{E}_1 = \Bigg\{&\|\widetilde{\eta}_a - \eta_a\|_{\infty} \leq C_2\Bigg[\sqrt{\sum_{s=1}^S \frac{\nu^2_s\log(h^{-d}/\eta)}{n_{s}h^d}} +  \max_{s \in [S]} \frac{\nu_s\log(h^{-d}/\eta)}{n_{s}h^{d}}\\
       & \hspace{3.5cm}+ h^\beta+ \sqrt{\sum_{s=1}^S\frac{\nu^2_s\log(1/\delta_s)\log(h^{-1}/\eta)}{n^2_{s}\epsilon^2_s h^{2d}}}\Bigg], a \in \{0,1\}\Bigg\},
   \end{align*}
   and 
   \begin{align*}
       \mathcal{E}_2  = \Big\{|\widetilde{T}_a - T_a^*| \leq C_3 \epsilon_T\Big\},
   \end{align*}
   where 
   \[\epsilon_T =\indc\{\tau^* \neq 0\} \rho^{1/\gamma} + \sqrt{\sum_{s=1}^S\frac{\nu_s^2\log(1/\eta)}{n_s}}+\sqrt{\sum_{s=1}^S\frac{\nu_s^2\log(1/\delta_s)\log(1/\eta)}{n_s^2\epsilon_s^2}}.\]
   The lemma thus follows by Lemmas \ref{l_fdp_eta_est_sup}, \ref{l_fdp_tau_est}, \ref{l_fdp_pi_a_est} and  a similar argument as the one in the proof of \Cref{t_fair_risk}.

\end{proof}

\subsubsection{Upper bound on \texorpdfstring{$\|\widetilde{\eta}_a- \eta_a\|_\infty$}{}}

\begin{lemma} \label{l_fdp_eta_est_sup}
    Under the same assumptions as in \Cref{t_fdp_risk}, it holds for any $\eta \in (0,1/2)$ and $a \in \{0,1\}$ that 
    \begin{align*}
        \mathbb{P}&\Big[\|\widetilde{\eta}_a- \eta_a\|_\infty\\
        &\hspace{-1em}\geq C_1\Bigg[\sqrt{\sum_{s=1}^S \frac{\nu^2_s\log(h^{-d}/\eta)}{n_{s,a}h^d}} +  \max_{s \in [S]} \frac{\nu_s\log(h^{-d}/\eta)}{n_{s,a}h^{d}} + h^\beta+ \sqrt{\sum_{s=1}^S\frac{\nu^2_s\log(1/\delta_s)\log(h^{-1}/\eta)}{n^2_{s,a}\epsilon^2_s h^{2d}}}\Bigg] \Bigg\} \leq \eta.
    \end{align*}
\end{lemma}

\begin{proof}
    By the triangle inequality, we have that 
    \begin{align*}
        \|\widetilde{\eta}_a- \eta_a\|_\infty \leq \;& \Big\|\sum_{s=1}^S  \frac{\nu_s}{\widetilde{p}_{X|A=a}(\cdot|a)}\Big\{\frac{1}{n_{s,a}}\sum_{i=1}^{n_{s,a,1}} K_{h}(X_{a,1}^{s,i} -\cdot)\Big\} - \eta_a(\cdot)\Big\|_{\infty}\\
        & + \Big\|\frac{1}{\widetilde{p}_{X|A=a}(\cdot|a)}\sum_{s=1}^S  \frac{8\nu_s\sqrt{2C_K\log(8/\delta_s)}}{n_{s,a}\epsilon_s h^d}W^s_{2,a}(\cdot)\Big\|_{\infty}\\
        = \;& (I) + (II).
    \end{align*}

    \noindent \textbf{Step 1: Upper bound on $(I)$.} To control $(I)$, note that 
    \begin{align*}
        (I) =\;& \sup_{x \in [0,1]^d}\Big|\sum_{s=1}^S  \frac{\nu_s}{\widetilde{p}_{X|A=a}(x|a)}\Big\{\frac{1}{n_{s,a}}\sum_{i=1}^{n_{s,a,1}} K_{h}(X_{a,1}^{s,i} -x)\Big\} - \frac{p_{X,Y|A=a}(x,1|a)}{p_{X|A=a}(x|a)}\Big|\\
        \leq \;&  \sup_{x \in [0,1]^d}\Big|\frac{1}{\widetilde{p}_{X|A=a}(x|a)}\sum_{s=1}^S  \nu_s \Big\{\frac{1}{n_{s,a}}\sum_{i=1}^{n_{s,a,1}} K_{h}(X_{a,1}^{s,i} -x) -  \int_\ell K_h(\ell-x)p_{X,Y|A=a}(\ell,1|a) \;\mathrm{d}\ell\Big\}\Big|\\
        & + \sup_{x \in [0,1]^d}\Big|\frac{1}{\widetilde{p}_{X|A=a}(x|a)} \Big\{\int_\ell K_h(\ell-x)p_{X,Y|A=a}(\ell,1|a) \;\mathrm{d}\ell - p_{X,Y|A=a}(x,1|a)\Big\}\Big|\\
        & + \sup_{x \in [0,1]^d} \Big|\frac{p_{X,Y|A=a}(x,1|a)}{\widetilde{p}_{X|A=a}(x|a)} - \frac{p_{X,Y|A=a}(x,1|a)}{p_{X|A=a}(x|a)}\Big|\\
        = \;& (I)_1 + (I)_2 + (I)_3.
    \end{align*}
    To control $(I)_1$, we have that 
    \begin{align*}
        &(I)_1  \\
        \leq \;&  \sup_{x \in [0,1]^d} \frac{1}{\widetilde{p}_{X|A=a}(x|a)} \cdot \sup_{x \in [0,1]^d}\Big|\sum_{s=1}^S \sum_{i=1}^{n_{s,a,1}}\frac{\nu_s}{n_{s,a}} K_{h}(X_{a,1}^{s,i} -x) -  \int_\ell K_h(\ell-x)p_{X,Y|A=a}(\ell,1|a) \;\mathrm{d}\ell\Big|.
    \end{align*}
    Therefore, by a union bound argument and \Cref{l_fdp_est_p_y_xa_kenel_sup} and \Cref{coro_fdp_p_xa_const_order}, it holds that 
    \begin{align*}
        \mathbb{P}\Big\{(I)_1 \geq C_1\sqrt{\sum_{s=1}^S \frac{\nu^2_s\log(h^{-d}/\eta)}{n_{s,a}h^d}} + C_2 \max_{s \in [S]} \frac{\nu_s\log(h^{-d}/\eta)}{n_{s,a}h^{d}}\Big\} \leq \frac{\eta}{6}.
    \end{align*}

    To control $(I)_2$, we have that
    \begin{align*}
        (I)_2 \leq \;& \sup_{x \in [0,1]^d} \frac{1}{\widetilde{p}_{X|A=a}(x|a)} \cdot \sup_{x \in [0,1]^d}\Big|\int_\ell K_h(\ell-x)p_{X,Y|A=a}(\ell,1|a) \;\mathrm{d}\ell - p_{X,Y|A=a}(x,1|a)\Big|\\
        \leq \;& C_{\text{adp}}h^\beta \sup_{x \in [0,1]^d} \frac{1}{\widetilde{p}_{X|A=a}(x|a)},
    \end{align*}
    where the last inequality follows from \Cref{a_kernel}\ref{a_kernel_adaptive} and the fact that both $\eta_a$ and $p_{X|A=a}$ are H\"{o}lder over $[0,1]^d$ under Assumptions \ref{a_prob}\ref{a_prob_p_x_a} and \ref{a_posterior}\ref{a_posterior_holder}. Therefore, by \Cref{coro_fdp_p_xa_const_order}, it holds that 
    \begin{align*}
        \mathbb{P}\Big\{(I)_2 \geq C_3h^{\beta}\Big\} \leq \eta/6.
    \end{align*}
    
    To control $(I)_3$, note that 
    \begin{align*}
        (I)_3 \leq\;& \sup_{x \in [0,1]^d} p_{X,Y|A=a}(x,1|a) \cdot \sup_{x \in [0,1]^d} \Big|\frac{1}{\widetilde{p}_{X|A=a}(x|a)}-\frac{1}{p_{X|A=a}(x|a)}\Big|\\
        \leq \;& C_4 \sup_{x \in [0,1]^d} \frac{1}{\widetilde{p}_{X|A=a}(x|a)p_{X|A=a}(x|a)} \cdot  \sup_{x \in [0,1]^d} \big|\widetilde{p}_{X|A=a}(x|a)- p_{X|A=a}(x|a)\big|\\
        \leq \;& \frac{C_4}{C_p} \sup_{x \in [0,1]^d} \frac{1}{\widetilde{p}_{X|A=a}(x|a)} \cdot  \sup_{x \in [0,1]^d} \big|\widetilde{p}_{X|A=a}(x|a)- p_{X|A=a}(x|a)\big|,
    \end{align*}
    where the first inequality follows from the fact that $p_{X,Y|A=a}$ is continuous over $[0,1]^d$, hence is bounded, and the fourth inequality follows from \Cref{a_prob}\ref{a_prob_p_x_a}. Thus, by applying a union bound argument to the events in \Cref{l_fdp_p_x_a_est_sup} and \Cref{coro_fdp_p_xa_const_order}, it holds that 
    \begin{align*}
        &\mathbb{P}\Bigg[(I)_3 \geq  C_5\Bigg\{\sqrt{\sum_{s=1}^S \frac{\nu^2_s\log(h^{-d}/\eta)}{n_{s,a}h^d}} \\
        & \hspace{10em} +  \max_{s \in [S]} \frac{\nu_s\log(h^{-d}/\eta)}{n_{s,a}h^{d}}+\sqrt{\sum_{s=1}^S\frac{\nu^2_s\log(1/\delta_s)\log(h^{-1}/\eta)}{n^2_{s,a}\epsilon^2_s h^{2d}}}\Bigg\}\Bigg] \leq \frac{\eta}{6}.
    \end{align*}

    Thus, by another union bound, we have that 
    \begin{align*}
        &\mathbb{P}\Bigg[(I) \geq  C_6\Bigg\{\sqrt{\sum_{s=1}^S \frac{\nu^2_s\log(h^{-d}/\eta)}{n_{s,a}h^d}} + \max_{s \in [S]} \frac{\nu_s\log(h^{-d}/\eta)}{n_{s,a}h^{d}}\\
        & \hspace{10em}+\sqrt{\sum_{s=1}^S\frac{\nu^2_s\log(1/\delta_s)\log(h^{-1}/\eta)}{n^2_{s,a}\epsilon^2_s h^{2d}}} + h^{\beta}\Bigg\} \Bigg] \leq \frac{\eta}{2}.
    \end{align*}

    \noindent \textbf{Step 2: Upper bound on $(II)$.} Note that 
    \begin{align*}
        (II) \leq \sup_{x \in [0,1]^d} \Big| \frac{1}{\widetilde{p}_{X|A=a}(x|a)}\Big| \cdot \sup_{x \in [0,1]^d}\Big|\sum_{s=1}^S  \frac{8\nu_s\sqrt{2C_K\log(8/\delta_s)}}{n_{s,a}\epsilon_s h^d}W^s_{2,a}(x)\Big|.
    \end{align*}
    Therefore, by applying a union bound argument to \Cref{coro_fdp_p_xa_const_order} and \Cref{l_fdp_sup_gaussian}, it holds that 
    \begin{align*}
        \mathbb{P}\Bigg\{(II) \geq C_7  \sqrt{\sum_{s=1}^S\frac{\nu^2_s\log(1/\delta_s)\log(h^{-1}/\eta)}{n^2_{s,a}\epsilon^2_s h^{2d}}} \Bigg\} \leq \frac{\eta}{2}.
    \end{align*}

    \noindent \textbf{Step 3:} The lemma thus follows by another union bound argument.

\end{proof}

\begin{lemma} \label{l_fdp_est_p_y_xa_kenel_sup}
    Under the same assumptions as in \Cref{t_fdp_risk}, it holds for any $\eta \in (0,1/2)$ and $a \in \{0,1\}$ that 
    \begin{align*}
        \mathbb{P}\Big\{\sup_{x \in [0,1]^d} &\Big|\sum_{s=1}^S \sum_{i=1}^{n_{s,a,1}}\frac{\nu_s}{n_{s,a}} K_{h}(X_{a,1}^{s,i} -x) -  \int_\ell K_h(\ell-x)p_{X,Y|A=a}(\ell,1|a) \;\mathrm{d}\ell\Big|\\
        & \hspace{3cm}\geq C_1\sqrt{\sum_{s=1}^S \frac{\nu^2_s\log(h^{-d}/\eta)}{n_{s,a}h^d}} + C_2 \max_{s \in [S]} \frac{\nu_s\log(h^{-d}/\eta)}{n_{s,a}h^{d}}\Big\} \leq \eta.
    \end{align*}
\end{lemma}

\begin{proof}
    The proof follows from a similar argument to the one in the proofs of Lemmas \ref{l_est_p_yx_a_kernel_sup} and~\ref{l_est_p_xa_kernel_sup}; we only include the difference here. With a slight abuse of notation, we aggregate all data points with $S=s$ and $A^s_i=a$ together, and write
    \begin{align*}
        \sum_{s=1}^S \sum_{i=1}^{n_{s,a,1}}\frac{\nu_s}{n_{s,a}} K_{h}(X_{a,1}^{s,i} -x) = \sum_{s=1}^S \sum_{i=1}^{n_{s,a}}\frac{\nu_s}{n_{s,a}}Y_a^{s,i}K_{h}(X_{a}^{s,i} -x).
    \end{align*}
    Note that for any $x \in [0,1]^d$, under \Cref{a_kernel}\ref{a_kernel_bounded}, we have that 
    \begin{align*}
        \max_{s \in [S]} \max_{i \in [n_{s,a,1}]} \Big|\frac{\nu_s}{n_{s,a}}Y_a^{s,i}K_{h}(X_{a}^{s,i} -x)\Big| \leq  \max_{s \in [S]} \frac{C_K\nu_s}{n_{s,a}h^{d}},
    \end{align*}
    and also
    \begin{align*}
        \mathbb{E}\Big\{\sum_{s=1}^S \sum_{i=1}^{n_{s,a}}\frac{\nu_s}{n_{s,a}}Y_a^{s,i}K_{h}(X_{a}^{s,i} -x)\Big\} =\;& \sum_{s=1}^S \frac{\nu_s}{n_{s,a}}\sum_{i=1}^{n_{s,a}} \sum_{y\in\{0,1\}} \int_\ell y K_h(\ell-x)p_{X,Y|A=a}(\ell,y|a) \;\mathrm{d}\ell\\
        = \;& \sum_{s=1}^S \nu_s\int_\ell K_h(\ell-x)p_{X,Y|A=a}(\ell,1|a) \;\mathrm{d}\ell\\
        = \;& \int_\ell K_h(\ell-x)p_{X,Y|A=a}(\ell,1|a) \;\mathrm{d}\ell.
    \end{align*}
    To control the variance, we have 
    \begin{align*}
        \var\Big\{\sum_{s=1}^S \sum_{i=1}^{n_{s,a}}\frac{\nu_s}{n_{s,a}}Y_a^{s,i}K_{h}(X_{a}^{s,i} -x)\Big\} =\;& \sum_{s=1}^S \sum_{i=1}^{n_{s,a}}\frac{\nu^2_s}{n^2_{s,a}} \var \big\{Y_a^{s,i}K_{h}(X_{a}^{s,i} -x)\big\}\\
         \leq \;& C_1C_K^2 \sum_{s=1}^S \frac{\nu^2_s}{n_{s,a}h^d},
    \end{align*}
    where the first inequality follows from the independence property across and within each server and the second inequality follows from a similar argument leading to \eqref{l_est_p_yx_a_kernel_sup_eq1}. Therefore, by the Bernstein inequality for bounded distributions \citep[e.g.~Theorem 2.8.4 in][]{vershynin2018high}, it holds that for any $t < 1$ that 
    \begin{align*}
        &\mathbb{P}\Big\{\Big|\sum_{s=1}^S \sum_{i=1}^{n_{s,a}}\frac{\nu_s}{n_{s,a}}Y_a^{s,i}K_{h}(X_{a}^{s,i} -x)  -  \int_\ell K_h(\ell-x)p_{X,Y|A=a}(\ell,1|a) \;\mathrm{d}\ell\Big| \geq t\Big\}\\
        \leq \;& C_2\exp\Big\{-t^2 \Big(\sum_{s=1}^S \frac{\nu^2_s}{n_{s,a}h^d} +  \max_{s \in [S]} \frac{\nu_s t}{n_{s,a}h^{d}} \Big)^{-1}\Big\}.
    \end{align*}
    
    The rest of the proof to control the sup-norm is similar to the argument in \textbf{Step 2} in the proof of \Cref{l_est_p_xa_kernel_sup}. With the notation, we can show that 
    \[\mathcal{T}_a(x) = \sum_{s=1}^S \sum_{i=1}^{n_{s,a}}\frac{\nu_s}{n_{s,a}}Y_a^{s,i}K_{h}(X_{a}^{s,i} -x) -\int_\ell K_h(\ell-x)p_{X,Y|A=a}(\ell,1|a) \;\mathrm{d}\ell \]
    satisfies that 
    \begin{align*}
        &\max_{j \in Q} \max_{x \in S_j} |\mathcal{T}_a(x) - \mathcal{T}_a(z_j)|\\
        \leq \;& \sum_{s=1}^S \sum_{i=1}^{n_{s,a}}\frac{\nu_s}{n_{s,a}}\max_{j \in Q} \max_{x \in S_j} \Big|Y_a^{s,i}K_{h}(X_{a}^{s,i} -x) - Y_a^{s,i}K_{h}(X_{a}^{s,i} -z_j)\Big|\\
        & + \max_{j \in Q} \max_{x \in S_j}  \Big|\int_\ell K_h(\ell-x)p_{X,Y|A=a}(\ell,1|a) \;\mathrm{d}\ell- \int_\ell K_h(\ell-z_j)p_{X,Y|A=a}(\ell,1|a) \;\mathrm{d}\ell\Big|\\
        \leq \;&  2 C_{\text{Lip}}h^{-(d+1)}\max_{j \in Q} \max_{x \in S_j} \|x -z_j\|_2.
    \end{align*}
    The rest of the proof to control the sup norm follows from the same argument using the covering lemma, and we omit it here.

\end{proof}

\begin{lemma}\label{l_fdp_sup_gaussian}
    Under the same assumptions as in \Cref{t_fdp_risk}, for any $a \in \{0,1\}$,  $ k \in \{1,2\}$ and $\eta \in (0, 1/2)$, it holds that 
    \begin{align*}
        \mathbb{P}\Bigg\{\sup_{x \in [0,1]^d}\Big|\sum_{s=1}^S \frac{8 \nu_s\sqrt{2C_K\log(8/\delta_s)}}{n_{s,a}\epsilon_s h^d}W^s_{k,a}(x)\Big| \geq C_1\sqrt{\sum_{s=1}^S\frac{\nu^2_s\log(1/\delta_s)\log(h^{-1}/\eta)}{n^2_{s,a}\epsilon^2_s h^{2d}}} \Bigg\} \leq \eta.
    \end{align*}
\end{lemma}

\begin{proof}
    The proof follows from a similar argument to the proof of \Cref{l_sup_gaussian}. Let
    \[Z(x) = \sum_{s=1}^S \frac{8 \nu_s\sqrt{2C_K\log(8/\delta_s)}}{n_{s,a}\epsilon_s h^d}W^s_{k,a}(x).\]
    Note that by the standard property of Gaussian random variables, it holds that $Z$ is a mean-zero Gaussian process, with covariance function
    \begin{align*}
        \cov(Z(\ell), Z(t)) =  C_1\sum_{s=1}^S \frac{\nu^2_s\log(1/\delta_s)}{n^2_{s,a}\epsilon^2_s h^{2d}} K\Big(\frac{\ell-t}{h}\Big),
    \end{align*}
    for all $\ell, t \in [0,1]^d$. Also note that
    \begin{align*}
        \mathbb{E}\big[\{Z(\ell)- Z(t)\}^2\big] = \;&C_2\sum_{s=1}^S\frac{\nu^2_s\log(1/\delta_s)}{n^2_{s,a}\epsilon^2_s h^{2d}} \Big\{K(0) - K\Big(\frac{\ell -t}{h}\Big)\Big\}\\
        \leq \;& C_2\sum_{s=1}^S\frac{C_{\text{Lip}}\nu^2_s\log(1/\delta_s)}{n^2_{s,a}\epsilon^2_s h^{2d}}\frac{\|\ell-t\|_2}{h}\\
        \leq \;& C_3 \sum_{s=1}^S\frac{\nu^2_s\log(1/\delta_s)}{n^2_{s,a}\epsilon^2_s h^{2d}}\frac{\|\ell-t\|_2}{h}: =  \sum_{s=1}^S  \frac{C_2\iota_s^2\|\ell-t\|_2}{h}.
    \end{align*}
    Since $\mathcal{N}([0,1]^d, \|\cdot\|_2, \kappa) \leq \kappa^{-d}$,  we have that 
    \[\mathcal{N}([0,1]^d, d_z, \kappa) \leq \Big(\frac{\sum_s\iota_s^2}{\kappa^2h}\Big)^d,\]
    where $\mathcal{N}$ is the covering number and $d_z$ is defined as $d_z(\ell,t) = \sqrt{ \mathbb{E}\big[\{Z(\ell)- Z(t)\}^2\big]}$.
    By Dudley's theorem \citep[e.g.][]{dudley2016vn}, it holds that 
    \begin{align*}
        \mathbb{E}\big\{\sup_{x \in [0,1]^d} Z(x)\big\} \leq\;& \int_{0}^\infty \sqrt{\log\{\mathcal{N}([0,1]^d, d_z, \kappa)\}} \; \mathrm{d}\kappa\\
        \leq \;& \int_{0}^{\sqrt{C_K\sum_{s=1}^S \iota_s^2}} \sqrt{\log\{\mathcal{N}([0,1]^d, d_z, \kappa)\}} \; \mathrm{d}\kappa\\
        \leq\;& C_4\sqrt{d\sum_{s=1}^S \iota_s^2 \log\Big(h^{-1}\sum_{s=1}^S\iota_s^2\Big)}\\
        \leq \;& C_5 \sqrt{d\sum_{s=1}^S \iota_s^2 \log(h^{-1})},
    \end{align*}
    where the second inequality follows as $\sup_{x \in [0,1]^d} \var(Z(x)) = K(0)\sum_{s=1}^S \iota_s^2$ and the last inequality follows as $\sum_{s=1}^S \iota_s^2 \leq 1$. Moreover, by the Borell--TIS inequality \citep[e.g.~Theorem 2.1.1 in][]{adler2007random}, it holds that for any $t >0$,
    \begin{align*}
        \mathbb{P}\Big\{\Big|\sup_{x \in [0,1]^d} Z(x)-  \mathbb{E}\big\{\sup_{x \in [0,1]^d} Z(x)\big\}\Big| \geq t\Big\} \leq 2\exp\Big\{-t^2\Big(2\sum_{s=1}^S\frac{\nu^2_s\log(1/\delta_s)}{n^2_{s,a}\epsilon^2_s h^{2d}}\Big)^{-1}\Big\}.
    \end{align*}
    Therefore, by the triangle inequality, it holds that 
    \begin{align*}
        \mathbb{P}\Bigg\{\Big|\sup_{x \in [0,1]^d} Z(x)| \geq C_6\sqrt{\sum_{s=1}^S\frac{\nu^2_s\log(1/\delta_s)\log(h^{-1}/\eta)}{n^2_{s,a}\epsilon^2_s h^{2d}}}\Bigg\} \leq \eta.
    \end{align*}
    Similarly, we can also provide the results for the process $-Z(x)$. The lemma thus follows.
\end{proof}

\subsubsection{Upper bound on \texorpdfstring{$|\widetilde{\tau} - \tau^*|$}{}}

\begin{lemma} \label{l_fdp_tau_est}
     Let $\widetilde{\tau}$ be the output of \Cref{alg_threshold_search_fdp} and denote
    \begin{align} \notag
        \rho =\;& C_1 \Bigg[\Bigg\{\sqrt{\sum_{s=1}^S \frac{\nu^2_s\log(h^{-d}/\eta)}{n_{s}h^d}} +  \max_{s \in [S]} \frac{\nu_s \log(h^{-d}/\eta)}{n_{s}h^{d}}+ h^\beta+ \sqrt{\sum_{s=1}^S\frac{\nu^2_s\log(1/\delta_s)\log(h^{-1}/\eta)}{n^2_{s}\epsilon^2_s h^{2d}}}\Bigg\}^\gamma\\ \label{l_fdp_tau_est_eq1}
        &+\sqrt{\sum_{s=1}^S \frac{\mu_s^2\log(1/\eta)}{\breve{n}_{s}}} + \max_{s \in [S]}\frac{\mu_s\log(1/\eta)}{\breve{n}_{s}}+\sqrt{\sum_{s=1}^S\frac{\mu_s^2M^4\log(1/\delta_s)\log(M/\eta)}{\breve{n}_{s}^2\epsilon_s^2}}\Bigg].
    \end{align}
    Suppose Assumptions \ref{a_prob}, \ref{a_kernel} and \ref{a_posterior} hold. We further assume that $\DD(0) \notin [\alpha - \zeta,\alpha] \cup [-\alpha,-\alpha+\zeta]$, where $\rho \leq \zeta < \alpha$. 

    Take $\theta = \rho^{1/\gamma}$, we have that 
    \begin{align*}
        \mathbb{P}\Big(|\widetilde{\tau} - \tau^*| \geq C_2\rho^{1/\gamma}\indc\{\tau^* \neq 0\}\Big) \leq \eta,
    \end{align*}
    whenever $\varpi \geq \rho^{1/\gamma}$. 
\end{lemma}

\begin{proof}
    The proof is a consequence of \Cref{prop_fdp_tau_est}. Denote $\tau^*_{\mathcal{G}} \in \mathcal{G}$ as the smallest point among the closest points to $\tau^*$ in the grid. We will consider three cases below. Denote
    \begin{align*}
        \mathcal{E} =\;& \big\{|\widetilde{\DD}_{\downarrow}(\tau^*_{\mathcal{G}}+ \kappa) - \DD(\tau^*_{\mathcal{G}}+\kappa)| \leq \rho\big\} \cap \big\{|\widetilde{\DD}_{\downarrow}(\tau^*_{\mathcal{G}}-\kappa) - \DD(\tau^*_{\mathcal{G}} -\kappa)| \leq \rho\big\}\\
        & \cap \big\{|\widetilde{\DD}_{\downarrow}(\tau^*) - \DD(\tau^*)| \leq \rho\big\} \cap \{\text{Output of }\Cref{alg_non_increasing} \text{ is valid and non-increasing}\}\\
        & \cap \big\{ (i), (ii), (iii) \text{ in } \Cref{prop_fdp_tau_est}\big\},
    \end{align*}
    where $\kappa >0$ (to be specified) is a scalar multiplicative of $\theta$, i.e.~$\tau^*_{\mathcal{G}}+ \kappa, \tau^*_{\mathcal{G}}- \kappa \in \mathcal{G}$ and $ \tau^*_{\mathcal{G}} - \kappa < \tau^* < \tau^*_{\mathcal{G}} + \kappa$. By Lemmas \ref{l_fdp_DD_est_err}, \ref{l_fdp_non_increas_exist} and \Cref{prop_fdp_tau_est}, we have that $\mathbb{P}(\mathcal{E}^c) \leq \eta$ when $\kappa$ is chosen such that $|\kappa| + \theta \leq \varpi$.

    \noindent \textbf{Case 1: $\tau^* =0$.} In this case, by \Cref{prop_fdp_tau_est}, we have that $\widetilde{\tau} =0$. Hence, $\mathbb{P}(\widetilde{\tau} = \tau^*) \geq \mathbb{P}(\mathcal{E})$.

    \noindent \textbf{Case 2: $\tau^*>0$.} 
    In the case when $\tau^* >0$, we have that $\DD(\tau^*) = \alpha$ and 
    \begin{align*}
        \mathbb{P}(\widetilde{\tau} > \tau^*_{\mathcal{G}} + \kappa) \leq\;& \mathbb{P}\{\widetilde{\DD}_{\downarrow}(\tau^*_{\mathcal{G}} + \kappa) >  \widetilde{\DD}_{\downarrow}(\widetilde{\tau}), \mathcal{E}\} + \mathbb{P}(\mathcal{E}^c)\\
        = \;& \mathbb{P}\{\widetilde{\DD}_{\downarrow}(\tau^*_{\mathcal{G}} + \kappa) - \DD(\tau^*_{\mathcal{G}} + \kappa) > \widetilde{\DD}_{\downarrow}(\widetilde{\tau})- \DD(\tau^*_{\mathcal{G}} + \kappa), \mathcal{E}\} + \mathbb{P}(\mathcal{E}^c)\\
        \leq \;& \mathbb{P}\{\widetilde{\DD}_{\downarrow}(\tau^*_{\mathcal{G}} + \kappa) - \DD(\tau^*_{\mathcal{G}} + \kappa) > \alpha -\rho - \DD(\tau^*_{\mathcal{G}} + \kappa), \mathcal{E}\} + \mathbb{P}(\mathcal{E}^c)\\
        = \;& \mathbb{P}\{\widetilde{\DD}_{\downarrow}(\tau^*_{\mathcal{G}} + \kappa) - \DD(\tau^*_{\mathcal{G}} + \kappa) > \DD(\tau^*) -\rho - \DD(\tau^*_{\mathcal{G}} + \kappa), \mathcal{E}\} + \mathbb{P}(\mathcal{E}^c)\\
        \leq \;& \mathbb{P}(2\rho > (\kappa- \theta)^\gamma) + \mathbb{P}(\mathcal{E}^c) = \mathbb{P}(\mathcal{E}^c),
    \end{align*}
    where the first inequality follows from the construction of \Cref{alg_non_increasing} and $\widetilde{\tau}$ is the one with the minimum absolute value; the second inequality follows from the fact that $\widetilde{\DD}_{\downarrow}(\widetilde{\tau}) \geq \alpha - \rho$ as shown in \Cref{prop_fdp_tau_est} $(iii)$; the third inequality follows by \Cref{a_posterior}\ref{a_posterior_disparity}, $\DD(\tau^*) > \DD(\tau^*_{\mathcal{G}}+ \kappa)$, $|\tau^* - \tau^*_{\mathcal{G}} - \kappa| =  \tau^*_{\mathcal{G}} + \kappa - \tau^* \geq \kappa-\theta$ and $\mathcal{E}$; and the last equality follows by taking $\kappa = (2^{1/\gamma}+1)\rho^{1/\gamma}$. Similarly, by taking $\kappa = 2(2^{1/\gamma}+1)\rho^{1/\gamma}$, we have that 
    \begin{align*}
        \mathbb{P}(\widetilde{\tau} < \tau^*_{\mathcal{G}} -\kappa) \leq \;& \mathbb{P}\{\widetilde{\DD}_{\downarrow}(\tau^*_{\mathcal{G}} - \kappa) \leq \widetilde{\DD}_{\downarrow}(\widetilde{\tau}), \mathcal{E}\} + \mathbb{P}(\mathcal{E}^c)\\
        = \;& \mathbb{P}\{\DD(\tau^*_{\mathcal{G}} - \kappa) -\widetilde{\DD}_{\downarrow}(\tau^*_{\mathcal{G}} - \kappa) \geq \DD(\tau^*_{\mathcal{G}} - \kappa) - \widetilde{\DD}_{\downarrow}(\widetilde{\tau}),\mathcal{E} \} + \mathbb{P}(\mathcal{E}^c)\\
        \leq \;& \mathbb{P}\{\DD(\tau^*_{\mathcal{G}} - \kappa) -\widetilde{\DD}_{\downarrow}(\tau^*_{\mathcal{G}} - \kappa) \geq \DD(\tau^*_{\mathcal{G}} - \kappa) - \alpha - \rho,\mathcal{E} \} + \mathbb{P}(\mathcal{E}^c)\\
        = \;& \mathbb{P}\{\DD(\tau^*_{\mathcal{G}} - \kappa) -\widetilde{\DD}_{\downarrow}(\tau^*_{\mathcal{G}} - \kappa) \geq \DD(\tau^*_{\mathcal{G}} - \kappa) - \DD(\tau^*) - \rho,\mathcal{E} \} + \mathbb{P}(\mathcal{E}^c)\\
        \leq \;& \mathbb{P}(2\rho \geq (\kappa- \theta)^\gamma) + \mathbb{P}(\mathcal{E}^c) = \mathbb{P}(\mathcal{E}^c),
    \end{align*}
    where the second inequality follows from the fact that $\widetilde{\DD}_{\downarrow}(\widetilde{\tau}) \leq \alpha+\rho$.

    \noindent \textbf{Case 3: $\tau^* <0$.} The proof of \textbf{Case 3} is similar to the proof of \textbf{Case 2}. We include the difference here for completeness. In the case when $\tau^*<0$, we have $\DD(\tau^*) = -\alpha$ and by taking $\kappa = 2(2^{1/\gamma}+1)\rho^{1/\gamma}$, we have
    \begin{align*}
        \mathbb{P}(\widetilde{\tau} > \tau^*_{\mathcal{G}} + \kappa) \leq\;& \mathbb{P}\{\widetilde{\DD}_{\downarrow}(\tau^*_{\mathcal{G}} + \kappa) \geq \widetilde{\DD}_{\downarrow}(\widetilde{\tau}),\mathcal{E}\} + \mathbb{P}(\mathcal{E}^c)\\
        \leq \;& \mathbb{P}\{\widetilde{\DD}_{\downarrow}(\tau^*_{\mathcal{G}} + \kappa) - \DD(\tau^*_{\mathcal{G}} + \kappa) \geq \DD(\tau^*)-\DD(\tau^*_{\mathcal{G}} + \kappa) -\rho, \mathcal{E}) + \mathbb{P}(\mathcal{E}^c)\\
        \leq \;& \mathbb{P}\{2\rho \geq \DD(\tau^*)-\DD(\tau^*_{\mathcal{G}} + \kappa) , \mathcal{E}) + \mathbb{P}(\mathcal{E}^c)\\
        \leq \;& \mathbb{P}(2\rho \geq (\kappa- \theta)^\gamma) + \mathbb{P}(\mathcal{E}^c) = \mathbb{P}(\mathcal{E}^c).
    \end{align*}
    Similarly, 
    \begin{align*}
        \mathbb{P}(\widetilde{\tau} < \tau^*_{\mathcal{G}} -\kappa) \leq \;& \mathbb{P}\{\widetilde{\DD}_{\downarrow}(\tau^*_{\mathcal{G}} - \kappa) < \widetilde{\DD}_{\downarrow}(\widetilde{\tau}), \mathcal{E}\} + \mathbb{P}(\mathcal{E}^c)\\
        = \;& \mathbb{P}\{\DD(\tau^*_{\mathcal{G}} - \kappa) -\widetilde{\DD}_{\downarrow}(\tau^*_{\mathcal{G}} - \kappa) > \DD(\tau^*_{\mathcal{G}} - \kappa) - \widetilde{\DD}_{\downarrow}(\widetilde{\tau}),\mathcal{E} \} + \mathbb{P}(\mathcal{E}^c)\\
        \leq \;& \mathbb{P}\{\DD(\tau^*_{\mathcal{G}} - \kappa) -\widetilde{\DD}_{\downarrow}(\tau^*_{\mathcal{G}} - \kappa) > \DD(\tau^*_{\mathcal{G}} - \kappa) - (-\alpha + \rho),\mathcal{E} \} + \mathbb{P}(\mathcal{E}^c)\\
        = \;& \mathbb{P}\{\DD(\tau^*_{\mathcal{G}} - \kappa) -\widetilde{\DD}_{\downarrow}(\tau^*_{\mathcal{G}} - \kappa) > \DD(\tau^*_{\mathcal{G}} - \kappa) - \DD(\tau^*) - \rho,\mathcal{E} \} + \mathbb{P}(\mathcal{E}^c)\\
        \leq \;& \mathbb{P}(2\rho \geq (\kappa- \theta)^\gamma) + \mathbb{P}(\mathcal{E}^c) = \mathbb{P}(\mathcal{E}^c),
    \end{align*}
    where the last inequality follows by choosing $\kappa = (2^{1/\gamma}+1)\rho^{1/\gamma}$.

    Combine three cases together, we have that with probability at least $1-\eta$ that $|\widetilde{\tau} - \tau^*_{\mathcal{G}}| \leq 2(2^{1/\gamma}+1)\rho^{1/\gamma}$. Therefore,  the lemma follows by the triangle inequality with the fact that $|\tau^*_{\mathcal{G}}- \tau^*| \leq \theta \leq \rho^{1/\gamma}$.

\end{proof}

\begin{proposition} \label{prop_fdp_tau_est}
    Let $\widetilde{\tau}$ be the output of \Cref{alg_threshold_search_fdp} and denote
    \begin{align*}\notag
        \rho =\;& C_1\Bigg[\sqrt{\sum_{s=1}^S \frac{\nu^2_sM\log(h^{-d}/\eta)}{n_{s}h^d}} +  \max_{s \in [S]} \frac{\nu_sM\log(h^{-d}/\eta)}{n_{s}h^{d}} + h^\beta+ \sqrt{\sum_{s=1}^S\frac{\nu^2_sM\log(1/\delta_s)\log(h^{-1}/\eta)}{n^2_{s}\epsilon^2_s h^{2d}}}\Bigg]^\gamma\\ 
        &+\sqrt{\sum_{s=1}^S \frac{\mu_s^2M \log(1/\eta)}{\breve{n}_{s}}} + \max_{s \in [S]}\frac{\mu_sM \log(1/\eta)}{\breve{n}_{s}}+\sqrt{\sum_{s=1}^S\frac{\mu_s^2M^4\log(1/\delta_s)\log(M/\eta)}{\breve{n}_{s}^2\epsilon_s^2}}.
    \end{align*}
    Suppose Assumptions \ref{a_prob}, \ref{a_kernel} and \ref{a_posterior} hold. We further assume that $\varpi \geq (\rho)^{1/\gamma}$ and $\DD(0) \notin [\alpha - \zeta,\alpha] \cup [-\alpha,-\alpha+\zeta]$, where $\rho \leq \zeta < \alpha$. Then the following events uniformly hold with probability at least $1-\eta$.
    \begin{itemize}
        \item [(i)] If $\tau^* =0$, then we have $\widetilde{\tau} =0$.

        \item [(ii)] In the case when $\tau^* \neq 0$, by choosing $\theta = C_2\rho^{1/\gamma}$, the solution of the optimisation problem in \Cref{alg_threshold_search_fdp}
        \[ \argmin_{\tau \in \mathcal{G}}\{|\tau|: |\widetilde{\DD}_{\downarrow}(\tau)| \in [\alpha-C_3\rho, \alpha+C_3\rho]\}\]
        exists.

        \item [(iii)] In the case when $\tau^* >0$, it holds that $|\widetilde{\DD}_{\downarrow}(\widetilde{\tau}) - \alpha| \leq C_3\rho$. Also, in the case when $\tau^* <0$, it holds that $|\widetilde{\DD}_{\downarrow}(\widetilde{\tau}) + \alpha| \leq C_3\rho$.
    \end{itemize}
\end{proposition}

\begin{proof}
    By choosing $\theta = C_2\rho^{1/\gamma}$, we can always find $|\kappa| \leq \theta \leq \varpi$ such that the closest points to $\tau^*$ in the grid is $\tau^*_{\mathcal{G}} = \tau^* +\kappa \in \mathcal{G}$. With such $\kappa$, consider the following event
    \begin{align*}
        \mathcal{E} =\;& \big\{|\widetilde{\DD}_{\downarrow}(\tau^*_{\mathcal{G}}) - \DD(\tau^*_{\mathcal{G}})| \leq \rho\big\} \cap \big\{|\widetilde{\DD}_{\downarrow}(\tau^*) - \DD(\tau^*)| \leq \rho\big\}\\
        & \cap \{\text{Output of }\Cref{alg_non_increasing} \text{ is valid and non-increasing}\}.
    \end{align*}
    By Lemmas \ref{l_fdp_DD_est_err} and \ref{l_fdp_non_increas_exist} and a union bound argument, we have that $\mathbb{P}\{\mathcal{E}\} \geq 1-\eta$. 
    
    \noindent \textbf{Proof for $(i)$.} In the case when $\tau^* =0$, we have that $\tau^*_{\mathcal{G}} =0$ and $|\DD(0)| \leq \alpha-\zeta \leq \alpha - \rho$. Conditioning on $\mathcal{E}$, we have with probability at least $1-\eta$ that
    \[|\widetilde{\DD}_{\downarrow}(0)| \leq |\DD(0)| + \rho \leq \alpha.\]
    Therefore, $\textbf{S3}$ in \Cref{alg_threshold_search_fdp} will give us $\widetilde{\tau} =0$.

    \noindent \textbf{Proof for $(ii)$.} To show the solution exists, we will divide the proof into 2 cases.
    \begin{itemize}
        \item In the case when $\tau^* > 0$, we have that $\DD(\tau^*) =\alpha$. Note that by choosing $\theta = (\rho^*)^{1/\gamma}$, we have $|\tau^* - \tau^*_{\mathcal{G}}| \leq \theta \leq \varpi$ and it holds with probability at least $1-\eta$ that
        \begin{align*}
            |\widetilde{\DD}_{\downarrow}(\tau^*_{\mathcal{G}}) -\alpha| \leq\;& |\widetilde{\DD}_{\downarrow}(\tau^*_{\mathcal{G}}) - \DD(\tau^*_{\mathcal{G}})| + |\DD(\tau^*_{\mathcal{G}}) -\DD(\tau^*)|\\
            \leq \;& \rho  + (2C_m)/(2C_\pi)^\gamma|\tau^*_{\mathcal{G}}-\tau^*|^\gamma\\
            \leq \;& \{1+(2C_mC_2^{\gamma})/(2C_\pi)^\gamma\}\rho,
        \end{align*}
        where the first inequality follows from the triangle inequality, the second inequality follows from $\mathcal{E}$ and \Cref{l_DD_shift_upper_bound} and the third inequality follows since $|\tau^*_{\mathcal{G}}-\tau^*| \leq \theta$. 

        \item In the case when $\tau^* < 0$, we have that $\DD(\tau^*) = -\alpha$. Similarly, with probability at least $1-\eta$ that
        \begin{align*}
            |\widetilde{\DD}_{\downarrow}(\tau^*_{\mathcal{G}}) +\alpha| \leq\;& |\widetilde{\DD}_{\downarrow}(\tau^*_{\mathcal{G}}) - \DD(\tau^*_{\mathcal{G}})| + |\DD(\tau^*_{\mathcal{G}}) -\DD(\tau^*)|  \leq C_3 \rho.
        \end{align*}
        Hence $\tau^*_{\mathcal{G}}$ is one solution to the optimisation problem in \textbf{S3}.
    \end{itemize}

    \noindent \textbf{Proof for $(iii)$.} Similar to the proof of $(ii)$, we will show this by considering two cases seperately.
    \begin{itemize}
        \item When $\tau^* >0$, in this case, by the proof of $(ii)$, we have that $\tau^*_{\mathcal{G}} \geq 0$ is one solution satisfying $|\widetilde{\DD}_{\downarrow}(\tau^*_{\mathcal{G}}) -\alpha |\leq C_3\rho$. Since $\widetilde{\tau}$ is the one with the minimum absolute value while satisfying the fairness constraint, it suffices for us to show there does not exists any $\tau' \in \mathcal{G}$, $|\tau'| \leq \tau^*_{\mathcal{G}}$ such that $|\widetilde{\DD}_{\downarrow}(\tau') +\alpha| \leq C_3\rho$. This can be easily verified by the fact that $\{\widetilde{\DD}_{\downarrow}(\tau)\}_{\tau\in \mathcal{G}}$ is a non-increasing sequence of points, hence we must have $\widetilde{\DD}_{\downarrow}(\tau') \geq \widetilde{\DD}_{\downarrow}(\tau^*_{\mathcal{G}}) \geq 0$ whenever $\tau' \leq \tau^*_{\mathcal{G}}$.

        \item When $\tau^* <0$, we have $\tau^*_{\mathcal{G}} \leq 0$ and $|\widetilde{\DD}_{\downarrow}(\tau^*_{\mathcal{G}}) +\alpha |\leq C_3\rho$. Similar to the previous case, using \Cref{l_fdp_non_increas_exist}, it can be easily show that there does not exist any $\tau' \in \mathcal{G}$, $\tau' > \tau^*_{\mathcal{G}}$ such that $|\widetilde{\DD}_{\downarrow}(\tau') -\alpha| \leq C_3\rho$.
    \end{itemize}

\end{proof}

\begin{lemma}  \label{l_fdp_DD_est_err}
    Suppose the same assumptions as in \Cref{thm_privacy_fair} hold. For any $\eta \in [0,1/2]$, denote
    \begin{align} \notag
        \rho=\;& \Bigg[\sqrt{\sum_{s=1}^S \frac{\nu^2_s\log(h^{-d}/\eta)}{n_{s}h^d}} +  \max_{s \in [S]} \frac{\nu_s \log(h^{-d}/\eta)}{n_{s}h^{d}}+ h^\beta+ \sqrt{\sum_{s=1}^S\frac{\nu^2_s\log(1/\delta_s)\log(h^{-1}/\eta)}{n^2_{s}\epsilon^2_s h^{2d}}}\Bigg]^\gamma\\ \notag
        &+\sqrt{\sum_{s=1}^S \frac{\mu_s^2\log(1/\eta)}{\breve{n}_{s}}} + \max_{s \in [S]}\frac{\mu_s\log(1/\eta)}{\breve{n}_{s}}+\sqrt{\sum_{s=1}^S\frac{\mu_s^2M^4\log(1/\delta_s)\log(M/\eta)}{\breve{n}_{s}^2\epsilon_s^2}}.
    \end{align}
    Then we have that
    \begin{align*}
        \mathbb{P}\Big\{\sup_{\tau\in \mathcal{G}: |\tau - \tau^*| \leq \varpi}|\widetilde{\DD}_{\downarrow}(\tau) - \DD(\tau)| \geq C_1\rho\Big\} \leq \eta,
    \end{align*}
    where $\varpi$ is given in \Cref{a_posterior}\ref{a_posterior_margin}.
\end{lemma}

\begin{proof}
    With a slight abuse of notation, in this proof, denote $\mathcal{D} = \{\widetilde{\eta}_0, \widetilde{\eta}_1, \widetilde{\pi}_0, \widetilde{\pi}_1\}$. By the triangle inequality, it holds that 
    \begin{align} \notag
        &\sup_{\tau\in \mathcal{G}: |\tau - \tau^*| \leq \varpi}|\widetilde{\DD}_{\downarrow}(\tau) - \DD(\tau)| \\ \notag
        \leq \;& \sup_{\tau \in \mathcal{G}} |\widetilde{\DD}(\tau) - \mathbb{E}\{\widetilde{\DD}(\tau)|\mathcal{D}\}| + \sup_{\tau\in \mathcal{G}: |\tau - \tau^*| \leq \varpi}|\mathbb{E}\{\widetilde{\DD}(\tau)|\mathcal{D}\} - \DD(\tau)| \\ \notag
        &+ \sup_{\tau\in \mathcal{G}}|\widetilde{\DD}_{\downarrow}(\tau) - \widetilde{\DD}(\tau)|\\ \label{l_fdp_DD_est_err_eq1}
        = \;& (I)+ (II) + (III),
    \end{align}
    where $\mathbb{E}\{\widetilde{\DD}(\tau)|\mathcal{D}\}$ is given by
    \begin{align*} 
        \mathbb{E}\{\widetilde{\DD}(\tau)|\mathcal{D}\} =\;& \mathbb{P}_{X|A=1}\Big\{\widetilde{\eta}_1(X) \geq\frac{1}{2}+\frac{\tau}{2\widetilde{\pi}_1}\Big|\mathcal{D}\Big\} -\mathbb{P}_{X|A=0}\Big\{\widetilde{\eta}_0(X) \geq\frac{1}{2}-\frac{\tau}{2\widetilde{\pi}_0}\Big| \mathcal{D}\Big\},
    \end{align*}
    and $\DD(\tau)$ is
    \begin{align*}
       \DD(\tau) =\;& \mathbb{P}_{X|A=1}\Big\{\eta_1(X) \geq\frac{1}{2}+\frac{\tau}{2\pi_1}\Big\} - \mathbb{P}_{X|A=0}\Big\{\eta_0(X) \geq\frac{1}{2}-\frac{\tau}{2\pi_0}\Big\}.
   \end{align*}

   \noindent \textbf{Step 1: Upper bound on $(I)$.} For any $\tau \in \mathcal{G}$,  
   \begin{align*}
       \widehat{\DD}^s(\tau) = \frac{1}{\breve{n}_{s,1}}\sum_{y\in\{0,1\}}\sum_{i=1}^{\breve{n}_{s,1,y}}\indc\Big\{\widetilde{\eta}_1(\breve{X}^{s,i}_{1,y}) \geq \frac{1}{2}+\frac{\tau}{2\widetilde{\pi}_1}\Big\} -\frac{1}{\breve{n}_{s,0}}\sum_{y\in\{0,1\}}\sum_{i=1}^{\breve{n}_{s,0,y}}\indc\Big\{\widetilde{\eta}_0(\breve{X}^{s,i}_{0,y}) \geq \frac{1}{2}-\frac{\tau}{2\widetilde{\pi}_0}\Big\}.
   \end{align*}
   Then we have that for $\tau_j \in \mathcal{G}$,
    \begin{align*}
        \widetilde{\DD}^s(\tau_j) =\;& \widehat{\DD}^s(\tau_j) + \Big(\frac{\breve{n}_{s,1}}{N_{s,1}}-1\Big)\frac{1}{\breve{n}_{s,1}}\sum_{y\in\{0,1\}}\sum_{i=1}^{\breve{n}_{s,1,y}}\indc\Big\{\widetilde{\eta}_1(\breve{X}^{s,i}_{1,y}) \geq \frac{1}{2}+\frac{\tau_j}{2\widetilde{\pi}_1}\Big\}\\
        & - \Big(\frac{\breve{n}_{s,0}}{N_{s,0}}-1\Big)\frac{1}{\breve{n}_{s,0}}\sum_{y\in\{0,1\}} \sum_{i=1}^{\breve{n}_{s,0,y}}\indc\Big\{\widetilde{\eta}_0(\breve{X}^{s,i}_{0,y}) \geq \frac{1}{2}-\frac{\tau_j}{2\widetilde{\pi}_0}\Big\}\\
        & +  \frac{1}{N_{s,1}}\sum_{\ell=1}^M \sum_{k=1}^{2^\ell}  \breve{w}_{s,1,\ell,k}\indc\Big\{(k-1)2^{M-\ell}+1 \geq j \text{ and } (\lceil k/2\rceil-1)2^{M-\ell+1}+1 <j \Big\}\\
        &-  \frac{1}{N_{s,0}}\sum_{\ell=1}^M \sum_{k=1}^{2^\ell}  \breve{w}_{s,0,\ell,k}\indc\Big\{(k-1)2^{M-\ell}+1 \geq j \text{ and } (\lceil k/2\rceil-1)2^{M-\ell+1}+1 <j \Big\}.
    \end{align*}
    Consequently,
   \begin{align*}
       & \sup_{\tau \in \mathcal{G}}|\widetilde{\DD}(\tau) - \mathbb{E}\{\widetilde{\DD}(\tau)|\mathcal{D}\}|\\
       =\;&\sup_{\tau \in \mathcal{G}} \Big|\sum_{s=1}^S \mu_s\widetilde{\DD}^s(\tau) - \sum_{s=1}^S \mu_s\mathbb{E}\{\widetilde{\DD}(\tau)|\mathcal{D}\}\Big|\\
       \leq \;&  \sup_{\tau \in \mathcal{G}}\Big|\sum_{s=1}^S \mu_s\big\{\widehat{\DD}^s(\tau) -\mathbb{E}\{\widetilde{\DD}(\tau)|\mathcal{D}\big\} \big\}\Big|\\
       &+  2 \sup_{\tau \in \mathcal{G}}\max_{a \in \{0,1\}}\Big|\sum_{s=1}^S \mu_s\Big(\frac{\breve{n}_{s,a}}{N_{s,a}}-1\Big)\frac{1}{\breve{n}_{s,a}}\sum_{y\in\{0,1\}} \sum_{i=1}^{\breve{n}_{s,a,y}}\indc\Big\{\widetilde{\eta}_a(\breve{X}^{s,i}_{a,y}) \geq \frac{1}{2}+\frac{(2a-1)\tau_j}{2\widetilde{\pi}_a}\Big\}\Big| \\
       &+ 2M\sup_{\ell,k,a}\Big|\sum_{s=1}^S\frac{ \mu_s\breve{w}_{s,a,\ell,k}}{N_{s,1}\wedge N_{s,0}}\Big|\\
       =\;& (I)_1 + (I)_2 + (I)_3,
   \end{align*}
   where the second inequality follows from the fact that 
   \[\sum_{\ell=1}^M \sum_{k=1}^{2^\ell}  \indc\{(k-1)2^{M-\ell}+1 \geq j \text{ and } (\lceil k/2\rceil-1)2^{M-\ell+1}+1 <j \} \leq M, \text{ for all } j.\]

   \noindent \textbf{Step 1-1: Upper bound on $(I)_1$}. To control $(I)_1$, denote
   \[\widehat{F}_{s,a}(\tau) = \frac{1}{\breve{n}_{s,a}}\sum_{y\in\{0,1\}}\sum_{i=1}^{\breve{n}_{s,a,y}}\indc\Big\{\widetilde{\eta}_a(\breve{X}^{s,i}_{a,y}) \geq \frac{1}{2}+\frac{(2a-1)\tau}{2\widetilde{\pi}_a}\Big\} = \frac{1}{\breve{n}_{s,a}}\sum_{y\in\{0,1\}}\sum_{i=1}^{\breve{n}_{s,a,y}} f_{\tau,a}(\breve{X}^{s,i}_{a,y}),\]
   and 
   \[F_{s,a}(\tau) = \mathbb{P}_{X|A=a}\Big\{\widetilde{\eta}_a(X) \geq\frac{1}{2}+\frac{(2a-1)\tau}{2\widetilde{\pi}_a}\Big|\mathcal{D}\Big\} = \mathbb{E}\Big\{ f_{\tau,a}(X)\Big|\mathcal{D}, A=a\Big\}.\]
   Then by the triangle inequality, we have that 
   \begin{align*}
       (I)_1 \leq \;&  \sup_{\tau \in \mathcal{G}}\Big|\sum_{s=1}^S\mu_s\big\{\widehat{F}_{s,1}(\tau) - F_{s,1}(\tau)\big\}\Big| + \sup_{\tau \in \mathcal{G}}\Big|\sum_{s=1}^S\mu_s\big\{\widehat{F}_{s,0}(\tau) - F_{s,0}(\tau)\big\}\Big|\\
       \leq \;&\sum_{a\in\{0,1\}} \sup_{\tau \in \mathcal{G}}\Big|\sum_{s=1}^S \sum_{i=1}^{\breve{n}_{s,a,y}}\sum_{y\in\{0,1\}}\frac{\mu_s}{\breve{n}_{s,a}} \big[f_{\tau,a}(\breve{X}^{s,i}_{a,y}) - \mathbb{E}\{f_{\tau,a}(\breve{X}^{s,i}_{a,y})|\mathcal{D},A=a\}\big]\Big|
   \end{align*}
   When $a=1$,  note that 
   \begin{align*}
       \max_{s \in [S]}  \max_{y \in \{0,1\}}\max_{i \in [\breve{n}_{s,1,y}]} \Big|\frac{\mu_s}{n_{s,1}} \big[f_{\tau,1}(\breve{X}^{s,i}_{1,y}) - \mathbb{E}\{f_{\tau,1}(\breve{X}^{s,i}_{1,y})|\mathcal{D},A=1\}\Big| \leq \max_{s \in [S]}\frac{\mu_s}{\breve{n}_{s,1}},
   \end{align*}
   and 
   \begin{align*}
       \sup_{\tau \in \mathcal{G}}\var\Big\{\sum_{s=1}^S\mu_s\big\{\widehat{F}_{s,1}(\tau) - F_{s,1}(\tau)\big\}\Big| \mathcal{D}\Big\} \leq \sum_{s=1}^S \frac{\mu_s^2}{\breve{n}_{s,1}},
   \end{align*}
   where the last inequality follows from the independence between samples within and across $S$ servers conditional on $\mathcal{D}$. Therefore, by the Bennett concentration inequality for empirical processes \citep[e.g.~Theorem 2.3 in][]{bousquet2002bennett}, it holds that 
   \begin{align*}
       \mathbb{P}\Bigg[&\sup_{\tau \in \mathcal{G}}\Big|\sum_{s=1}^S\mu_s\big\{\widehat{F}_{s,1}(\tau) - F_{s,1}(\tau)\big\}\Big| \\
       & \geq \mathbb{E}\Big\{\sup_{\tau \in \mathcal{G}}\big|\sum_{s=1}^S\mu_s\big\{\widehat{F}_{s,1}(\tau) - F_{s,1}(\tau)\big\}\big|\Big|\mathcal{D}\Big\} + \sqrt{2t\sum_{s=1}^S \frac{\mu_s^2}{\breve{n}_{s,1}}} + \frac{t}{3}\max_{s \in [S]}\frac{\mu_s}{\breve{n}_{s,1}}\Big| \mathcal{D}\Bigg] \leq \exp(-t).
   \end{align*}

   To control the expectation term, by the symmetrisation lemma \citep[e.g.~Lemma 2.3.1 in][]{van1996weak}, it holds that
   \begin{align*}
       &\mathbb{E}\Big[\sup_{\tau \in \mathcal{G}}\big|\sum_{s=1}^S\mu_s\big\{\widehat{F}_{s,1}(\tau) - F_{s,1}(\tau)\big\}\big|\Big|\mathcal{D}\Big]\\
       \leq \; &\mathbb{E}\Big[\sup_{\tau: |\tau|\leq 1}\big|\sum_{s=1}^S\mu_s\big\{\widehat{F}_{s,1}(\tau) - F_{s,1}(\tau)\big\}\big|\Big|\mathcal{D}\Big]\\
       \leq \;& 2\mathbb{E}\Big[\sup_{\tau: |\tau|\leq 1}\big|\sum_{s=1}^S \sum_{y\in\{0,1\}}\sum_{i=1}^{\breve{n}_{s,1,y}} \frac{\mu_s}{\breve{n}_{s,1}} \xi_{s,i,y}f_{\tau,1}(\breve{X}^{s,i}_{1,y})\big|\Big|\mathcal{D}\Big],
   \end{align*}
   where $\{\xi_{s,i,y}\}$ is a sequence of i.i.d.~Rademacher random variables. Conditioning on the samples $\{\breve{X}^{s,i}_{1,y}\}$ and pooling the data together, we can sort the values of $\{\widetilde{\eta}_a(\breve{X}^{s,i}_{a,y})\}$ in descending order and rewrite
   \[\sup_{\tau: |\tau|\leq 1}\big|\sum_{s=1}^S \sum_{y\in\{0,1\}}\sum_{i=1}^{\breve{n}_{s,1,y}} \frac{\mu_s}{\breve{n}_{s,1}} \xi_{s,i,y}f_{\tau,1}(\breve{X}^{s,i}_{1,y}) \big|= \sup_{\tau: |\tau|\leq 1}\big|\sum_{j=1}^{N} \frac{\mu_{s(j)}}{\breve{n}_{s(j),1}} \xi_{j}f_{\tau,1}(\breve{X}_j)\big|,\]
   where $N = \sum_{s=1}^S \sum_{y \in \{0,1\}} \breve{n}_{s,1,y}$ and $s(j)$ gives the server label of the rank $j$ value of $\{\widetilde{\eta}_a(\breve{X}^{s,i}_{a,y})\}$. Since $f_{\tau,1}(\cdot)$ is an indicator function to test if the value of $\eta_1(\cdot)$ exceeds the threshold $1/2+\tau/(2\widetilde{\pi}_1)$, finding the $\tau_*$, which gives the sup of the sum, is equivalent to finding the  $r^*$-th value of sorted  $\eta_1(\breve{X}^{s,i}_{1,y})$ such that the sum $|\sum_{j=1}^{r^*} \mu_{s(j)}\xi_{j}/\breve{n}_{s(j),1}|$ is maximised. Thus, we have that 
   \begin{align*}
       \sup_{\tau: |\tau|\leq 1}\big|\sum_{j=1}^{N} \frac{\mu_{s(j)}}{\breve{n}_{s(j),1}} \xi_{j}f_{\tau,1}(\breve{X}_j)\big|  = \max_{r^* \in [N]}\big|\sum_{j=1}^{r^*} \frac{\mu_{s(j)}}{\breve{n}_{s(j),1}} \xi_{j}\big|.
   \end{align*}
   Since $\{S_r = \sum_{j=1}^{r} \mu_{s(j)}\xi_{j}/\breve{n}_{s(j),1}\}$ is a martingale sequence, by Doob's maximal inequality \citep[e.g.~Theorem 4.4.4 in][]{durrett2019probability}, we have that 
   \begin{align*}
       \mathbb{E}_{\xi}\Big[\max_{r^* \in [N]}\Big(\sum_{j=1}^{r^*} \frac{\mu_{s(j)}}{\breve{n}_{s(j),1}} \xi_{j}\Big)^2\Big|\mathcal{D}\Big] \leq 4\mathbb{E}_\xi\Big[\Big(\sum_{j=1}^{N} \frac{\mu_{s(j)}}{\breve{n}_{s(j),1}} \xi_{j}\Big)^2\Big|\mathcal{D}\Big] = 4\sum_{j=1}^{N} \frac{\mu^2_{s(j)}}{\breve{n}^2_{s(j),1}} = 4 \sum_{s=1}^S \frac{\mu_s^2}{\breve{n}_{s,1}}.
   \end{align*}
   Taking another expectation over $\{\widetilde{\eta}_a(\breve{X}^{s,i}_{a,y})\}$ and applying Jensen's inequality lead to 
   \begin{align*}
       \mathbb{E}\Big[\sup_{\tau: |\tau|\leq 1}\big|\sum_{s=1}^S \sum_{y\in\{0,1\}}\sum_{i=1}^{\breve{n}_{s,1,y}} \frac{\mu_s}{\breve{n}_{s,1}} \xi_{s,i,y}f_{\tau,1}(\breve{X}^{s,i}_{1,y}) \big|\Big|\mathcal{D}\Big] = \mathbb{E}\Big[\max_{r^* \in [N]}\big|\sum_{j=1}^{r^*} \frac{\mu_{s(j)}}{\breve{n}_{s(j),1}} \xi_{j}\big|\Big|\mathcal{D}\Big] \leq 2\sqrt{\sum_{s=1}^S \frac{\mu_s^2}{\breve{n}_{s,1}}}.
   \end{align*}

   A similar justification can also be applied to the case when $a=0$. Therefore, by a union bound argument and tower property, we have that 
   \begin{align}\label{l_fdp_DD_est_err_eq2}
       \mathbb{P}\Big[(I)_1 \geq C_1\Big\{\sqrt{\sum_{s=1}^S \frac{\mu_s^2 \log(1/\eta)}{\breve{n}_{s,1}\wedge \breve{n}_{s,0}}} + \max_{s \in [S]}\frac{\mu_s \log(1/\eta)}{\breve{n}_{s,1}\wedge \breve{n}_{s,0}}\Big\}\Big] \leq \frac{\eta}{9}.
   \end{align}

   \noindent \textbf{Step 1-2: Upper bound on $(I)_2$}. Denote
   \[\iota_{a,s,\tau}=\frac{1}{\breve{n}_{s,a}}\sum_{y\in\{0,1\}} \sum_{i=1}^{\breve{n}_{s,a,y}}\indc\Big\{\widetilde{\eta}_a(\breve{X}^{s,i}_{a,y}) \geq \frac{1}{2}+\frac{(2a-1)\tau}{2\widetilde{\pi}_a}\Big\},\]
    and by construction, we have that $\sup_{\tau,a,s}\iota_{a,s,\tau}^2 \leq 1$. Also, by \Cref{l_fdp_N_n_same_order}, we have with probability at least $1-\eta/18$ that 
   \begin{align*}
       (I)_2 =\;& 2 \sup_{\tau \in \mathcal{G}}\max_{a \in \{0,1\}} \Big|\sum_{s=1}^S \frac{\mu_s\iota_{a,s,\tau}}{N_{s,a}}\big(\breve{w}_{s,a,1,1}+ \breve{w}_{s,a,1,2}\big)\Big| \\
       \leq\;& C_2\sup_{\tau \in \mathcal{G}}\max_{a \in \{0,1\}} \Big|\sum_{s=1}^S \frac{\mu_s\iota_{a,s,\tau}}{\breve{n}_{s,a}}\big(\breve{w}_{s,a,1,1}+ \breve{w}_{s,a,1,2}\big)\Big|.
   \end{align*}
    Moreover, using the independence property between $\{\breve{w}\}_{s,a,\ell,k}$ and conditioning on $\{\iota_{a,s,\tau}\}_{s=1}^S$, we the have that 
   \begin{align*}
       \sum_{s=1}^S \frac{\mu_s\iota_{a,s,\tau}}{\breve{n}_{s,a}}\big(\breve{w}_{s,a,1,1}+ \breve{w}_{s,a,1,2}\big) \sim N\Big(0,  \sum_{s=1}^S \frac{8\mu^2_s\iota^2_{a,s}M\log(1/\delta_s)}{\breve{n}^2_{s,a}\epsilon_s^2}\Big).
   \end{align*}
   Consequently, the standard Gaussian tail properties \citep[e.g.~Proposition 2.1.2 in][]{vershynin2018high} gives us
   \begin{align*}
       &\mathbb{P}\Big\{\Big|\sum_{s=1}^S \frac{\mu_s\iota_{a,s,\tau}}{\breve{n}_{s,a}}\big(\breve{w}_{s,a,1,1}+ \breve{w}_{s,a,1,2}\big)\Big| \geq \sqrt{C_3\sum_{s=1}^S \frac{\mu^2_sM\log(1/\delta_s)\log(1/\eta)}{\breve{n}^2_{s,a}\epsilon_s^2}}\Big|\{\iota_{a,s,\tau}\}_{s\in[S]}\Big\}\\ 
       \leq\;& \mathbb{P}\Big\{\Big|\sum_{s=1}^S \frac{\mu_s\iota_{a,s,\tau}}{\breve{n}_{s,a}}\big(\breve{w}_{s,a,1,1}+ \breve{w}_{s,a,1,2}\big)\Big| \geq C_3\sqrt{\sum_{s=1}^S \frac{\mu^2_s\iota^2_{a,s}M\log(1/\delta_s)\log(1/\eta)}{\breve{n}^2_{s,a}\epsilon_s^2}}\Big|\{\iota_{a,s,\tau}\}_{s\in[S]}\Big\} \leq \frac{\eta}{18}.
   \end{align*}
   Thus, by the tower property and a union bound argument, we have that 
   \begin{align} \label{l_fdp_DD_est_err_eq3}
       \mathbb{P}\Bigg[(I)_2 \geq C_{15}\sqrt{\sum_{s=1}^S \frac{\mu^2_sM^2\log(1/\delta_s)\log(1/\eta)}{\{\breve{n}_{s,1} \wedge\breve{n}_{s,0}\}^2 \epsilon_s^2}}\Bigg] \leq \frac{\eta}{9}. 
   \end{align}

   \noindent \textbf{Step 1-3: Upper bound on $(I)_3$.} On the event in \Cref{l_fdp_N_n_same_order}, we have with probability at least $1-\eta/10$ that 
   \begin{align*}
       (I)_3 \leq 2M\sup_{\ell,k,a}\Big|\sum_{s=1}^S\frac{ \mu_s\breve{w}_{s,a,\ell,k}}{\breve{n}_{s,1}\wedge \breve{n}_{s,0}}\Big|.
   \end{align*}
   Note that the standard Gaussian tail properties \citep[e.g.~Proposition 2.1.2 in][]{vershynin2018high} gives
   \[\sum_{s=1}^S\frac{\mu_s\breve{w}_{s,a,\ell,k}}{\breve{n}_{s,1}\wedge \breve{n}_{s,0}} \sim N\Big(0, \sum_{s=1}^S\frac{4\mu_s^2M\log(1/\delta_s)}{\{\breve{n}_{s,1}\wedge \breve{n}_{s,0}\}^2\epsilon_s^2}\Big).\]
   Therefore, by the standard Gaussian tail properties \citep[e.g.~Proposition 2.1.2 in][]{vershynin2018high} and a union bound argument, we have
   \begin{align} \label{l_fdp_DD_est_err_eq4}
       \mathbb{P}\Bigg[(I)_3 \geq C_4\sqrt{\sum_{s=1}^S\frac{\mu_s^2M^4 \log(1/\delta_s)\log(M/\eta)}{\{\breve{n}_{s,1}\wedge \breve{n}_{s,0}\}^2\epsilon_s^2}}\Bigg] \leq \frac{\eta}{9}.
   \end{align}

   \noindent \textbf{Step 2: Upper bound on $(II)$.}  This is very similar to \textbf{Step 2} in the proof of \Cref{l_DD_est_err}. We only include the difference here. Denote $T^*_a =1/2+ \tau^*(2a-1)/(2\pi_a)$ and write $\tau = \tau^* +\kappa$. It holds that 
    \begin{align*}
       &\sup_{\kappa: |\kappa|\leq \varpi}\Big|\mathbb{P}_{X|A=1}\Big\{\widetilde{\eta}_1(X) \geq \frac{1}{2}+\frac{\tau^*+\kappa}{2\widetilde{\pi}_1}\Big|\mathcal{D}\Big\} - \mathbb{P}_{X|A=1}\Big\{\eta_1(X) \geq\frac{1}{2}+\frac{\tau^*+\kappa}{2\pi_1}\Big\}\Big|\\
       = \;& \sup_{\kappa: |\kappa|\leq \varpi}\int \Big|\indc\Big\{\widetilde{\eta}_1(x) \geq \frac{1}{2}+\frac{\tau^*}{2\widetilde{\pi}_1} + \frac{\kappa}{2\widetilde{\pi}_1}\Big\} - \indc\Big\{\eta_1(x) \geq T^*_1 + \frac{\kappa}{2\pi_1}\Big\} \Big|\;\mathbb{d}\mathbb{P}_{X|A=1,\mathcal{D}}(x)\\
       \leq \;& \sup_{\kappa: |\kappa|\leq \varpi} \int \indc\Big\{\Big |\eta_1(x)-T^*_1-\frac{\kappa}{2\pi_1} \Big| \leq \|\widetilde{\eta}_1 - \eta_1\|_{\infty} + \frac{|\tau^* +\kappa|}{2} \cdot \Big|\frac{1}{\widetilde{\pi}_1}-\frac{1}{\pi_1}\Big| \Big\}\;\mathbb{d}\mathbb{P}_{X|A=1,\mathcal{D}}(x)\\
       \leq \;& \sup_{\kappa: |\kappa|\leq \varpi} \int \indc\Big\{\Big |\eta_1(x)-T^*_1-\frac{\kappa}{2\pi_1} \Big| \leq \|\widetilde{\eta}_1 - \eta_1\|_{\infty} + \frac{|\tau^*|+ \varpi}{2} \cdot \Big|\frac{1}{\widetilde{\pi}_1}-\frac{1}{\pi_1}\Big| \Big\}\;\mathbb{d}\mathbb{P}_{X|A=1,\mathcal{D}}(x)\\
       \leq \;& C_m \Big( \|\widetilde{\eta}_1 - \eta_1\|_{\infty} + \frac{|\tau^*| +\varpi}{2} \cdot \Big|\frac{1}{\widetilde{\pi}_1}-\frac{1}{\pi_1}\Big|\Big)^\gamma,
   \end{align*}
   where the last inequality follows from \Cref{a_posterior}\ref{a_posterior_margin}. Consider the following events:
   \begin{align*}
       \mathcal{E}_1 = \Bigg\{&\|\widetilde{\eta}_a - \eta_a\|_{\infty} \leq C_5\Bigg[\sqrt{\sum_{s=1}^S \frac{\nu^2_s\log(h^{-d}/\eta)}{n_{s}h^d}} +  \max_{s \in [S]} \frac{\nu_s\log(h^{-d}/\eta)}{n_{s}h^{d}}\\
       & \hspace{3.5cm}+ h^\beta+ \sqrt{\sum_{s=1}^S\frac{\nu^2_s\log(1/\delta_s)\log(h^{-1}/\eta)}{n^2_{s}\epsilon^2_s h^{2d}}}\Bigg], a \in \{0,1\}\Bigg\},
   \end{align*}
    \[\mathcal{E}_2 = \Bigg\{|\widetilde{\pi}_a - \pi_a| \leq C_6\Bigg\{\sqrt{\sum_{s=1}^S\frac{\nu_s^2\log(1/\eta)}{n_s}}+\sqrt{\sum_{s=1}^S\frac{\nu_s^2\log(1/\delta_s)\log(1/\eta)}{n_s^2\epsilon_s^2}}\Bigg\},\; a\in \{0,1\}\Bigg\},\]
    and 
    \[\mathcal{E}_3 = \{C_7 n_s \leq n_{s,a} \leq n_s\;\; \text{and}\;\; C_{8}\breve{n}_s \leq \breve{n}_{s,a} \leq \breve{n}_{s}, \; a\in \{0,1\}, s \in [S]\}.\]
    By Lemmas \ref{l_fdp_eta_est_sup}, \ref{l_fdp_pi_a_est} and \Cref{coro_fdp_n_a_n_same_order} and a union bound argument, we have that $\mathbb{P}(\mathcal{E}_1 \cap \mathcal{E}_2 \cap \mathcal{E}_3) \geq 1-\eta/3$. Under $\mathcal{E}_2$, it holds from the triangle inequality that 
    \begin{align*}
        \widetilde{\pi}_1 \geq \pi_1 - C_9\Bigg\{\sqrt{\sum_{s=1}^S\frac{\nu_s^2\log(1/\eta)}{n_s}}+\sqrt{\sum_{s=1}^S\frac{\nu_s^2\log(1/\delta_s)\log(1/\eta)}{n_s^2\epsilon_s^2}}\Bigg\} \geq \frac{C_\pi}{2},
    \end{align*}
    where the last inequality follows from \Cref{a_prob}\ref{a_prob_pi_a}. Consequently, we have that 
    \begin{align*}
        \Big|\frac{1}{\widetilde{\pi}_1}-\frac{1}{\pi_1}\Big| = \frac{|\widetilde{\pi}_1 - \pi_1|}{\widetilde{\pi}_1 \cdot \pi_1} \leq \frac{2}{C_\pi^2}|\widetilde{\pi}_1 - \pi_1|.
    \end{align*}
    Thus, we have with probability at least $1-\eta/6$ that 
    \begin{align} \notag
        &\Big( \|\widetilde{\eta}_1 - \eta_1\|_{\infty} + \frac{|\tau^*| +\varpi}{2} \cdot \Big|\frac{1}{\widetilde{\pi}_1}-\frac{1}{\pi_1}\Big|\Big)^\gamma\\ \notag
        \leq \;&C_{10}\Bigg[\sqrt{\sum_{s=1}^S \frac{\nu^2_s\log(h^{-d}/\eta)}{n_{s}h^d}} +  \max_{s \in [S]} \frac{\nu_s\log(h^{-d}/\eta)}{n_{s}h^{d}}+ h^\beta+ \sqrt{\sum_{s=1}^S\frac{\nu^2_s\log(1/\delta_s)\log(h^{-1}/\eta)}{n^2_{s}\epsilon^2_s h^{2d}}}\\ \notag
        &\hspace{1cm} + \sqrt{\sum_{s=1}^S\frac{\nu_s^2\log(1/\eta)}{n_s}}+\sqrt{\sum_{s=1}^S\frac{\nu_s^2\log(1/\delta_s)\log(1/\eta)}{n_s^2\epsilon_s^2}}\Bigg]^\gamma\\ \label{l_fdp_DD_est_err_eq5}
        \leq \;& C_{11}\Bigg[\sqrt{\sum_{s=1}^S \frac{\nu^2_s\log(h^{-d}/\eta)}{n_{s}h^d}} +  \max_{s \in [S]} \frac{\nu_s\log(h^{-d}/\eta)}{n_{s}h^{d}}+ h^\beta+ \sqrt{\sum_{s=1}^S\frac{\nu^2_s\log(1/\delta_s)\log(h^{-1}/\eta)}{n^2_{s,1}\epsilon^2_s h^{2d}}}\Bigg]^\gamma,
    \end{align}
    where the first inequality follows from \Cref{l_range_tau} and the last inequality follows under $\mathcal{E}_3$. Similarly, we have that 
    \begin{align*}
        &\sup_{\kappa: |\kappa|\leq \varpi}\Big|\mathbb{P}_{X|A=0}\Big\{\widetilde{\eta}_0(X) \geq \frac{1}{2}-\frac{\tau^*+\kappa}{2\widetilde{\pi}_0}\Big|\mathcal{D}\Big\} - \mathbb{P}_{X|A=1}\Big\{\eta_0(X) \geq\frac{1}{2}-\frac{\tau^*+\kappa}{2\pi_0}\Big\}\Big|\\
         \leq \;& \sup_{\kappa: |\kappa|\leq \varpi}\int \indc\Big\{\Big |\eta_0(x)-T^*_0 +\frac{\kappa}{2\pi_0} \Big| \leq \|\widetilde{\eta}_0 - \eta_0\|_{\infty} + \frac{|\tau^* +\kappa|}{2} \cdot \Big|\frac{1}{\widetilde{\pi}_0}-\frac{1}{\pi_0}\Big| \Big\}\;\mathbb{d}\mathbb{P}_{X|A=0,\mathcal{D}}(x)\\
       \leq \;& C_m \Big( \|\widetilde{\eta}_0 - \eta_0\|_{\infty} + \frac{|\tau^*| +\varpi}{2} \cdot \Big|\frac{1}{\widetilde{\pi}_0}-\frac{1}{\pi_0}\Big|\Big)^\gamma.
    \end{align*}
    By a similar argument leading to \eqref{l_fdp_DD_est_err_eq5}, we have that with probability at least $1-\eta/6$, we have that 
    \begin{align*}
        &\sup_{\kappa: |\kappa|\leq \varpi}\Big|\mathbb{P}_{X|A=0}\Big\{\widetilde{\eta}_0(X) \geq \frac{1}{2}-\frac{\tau^*+\kappa}{2\widetilde{\pi}_0}\Big|\mathcal{D}\Big\} - \mathbb{P}_{X|A=1}\Big\{\eta_0(X) \geq\frac{1}{2}-\frac{\tau^*+\kappa}{2\pi_0}\Big\}\Big| \\
        \leq \;& C_{12}\Bigg[\sqrt{\sum_{s=1}^S \frac{\nu^2_s\log(h^{-d}/\eta)}{n_{s}h^d}} +  \max_{s \in [S]} \frac{\nu_s\log(h^{-d}/\eta)}{n_{s}h^{d}}+ h^\beta+ \sqrt{\sum_{s=1}^S\frac{\nu^2_s\log(1/\delta_s)\log(h^{-1}/\eta)}{n^2_{s}\epsilon^2_s h^{2d}}}\Bigg]^\gamma.
    \end{align*}
    By a union bound argument, we have with probability at least $1-\eta/3$ that 
    \begin{align}  \notag
        &(II) \\ \label{l_fdp_DD_est_err_eq6}
        \leq\;& C_{13}\Bigg[\sqrt{\sum_{s=1}^S \frac{\nu^2_s\log(h^{-d}/\eta)}{n_{s}h^d}} +  \max_{s \in [S]} \frac{\nu_s\log(h^{-d}/\eta)}{n_{s}h^{d}}+ h^\beta+ \sqrt{\sum_{s=1}^S\frac{\nu^2_s\log(1/\delta_s)\log(h^{-1}/\eta)}{n^2_{s}\epsilon^2_s h^{2d}}}\Bigg]^\gamma.
    \end{align}
    
    \noindent \textbf{Step 3: Upper bound on $(III)$.} By \Cref{l_fdp_non_increas_exist}, we have with probability at least $1 - \eta/3$ that 
    \begin{align} \label{l_fdp_DD_est_err_eq7}
        (III) \leq \omega = C_{14}\sqrt{\sum_{s=1}^S\frac{\mu_s^2M^4\log(1/\delta_s)\log(M/\eta)}{\breve{n}_s^2\epsilon_s^2}}
    \end{align}

    \noindent \textbf{Step 4:} The lemma thus follows by applying a union bound argument, \Cref{coro_fdp_n_a_n_same_order}, and substituting the results in \eqref{l_fdp_DD_est_err_eq2}, \eqref{l_fdp_DD_est_err_eq3}, \eqref{l_fdp_DD_est_err_eq4}, \eqref{l_fdp_DD_est_err_eq6} and \eqref{l_fdp_DD_est_err_eq7} into \eqref{l_fdp_DD_est_err_eq1}.

\end{proof}

\begin{lemma} \label{l_fdp_non_increas_exist}
    Denote $\{f_i\}_{i \in [2^M+1]}$ the output of \Cref{alg_non_increasing} with input sequence $\{\widetilde{\DD}(\tau_i)\}_{i \in [2^M+1]}$. It holds with probability at least $1-\eta$ that \Cref{alg_non_increasing} will output a non-increasing sequence $f_1, \ldots, f_{2^M+1}$ satisfying $\sup_{i \in [2^M+1]}|f_i - \widetilde{\DD}(\tau_i)| \leq \omega$.
\end{lemma}

\begin{proof}
    By construction, if \Cref{alg_non_increasing} does not output \textbf{NULL}, then the lemma follows. Therefore, in the rest of the proof, it suffices for us to show that for all $i \in [2^M+1]$, with large probability, $f_i = \min\{f_{i-1}, \widetilde{\DD}(\tau_i) +\omega\} \geq \widetilde{\DD}(\tau_i) - \omega$. We will show this by an induction argument. 

    Denote
    \begin{align*} 
        \mathcal{E}= \Bigg\{\widetilde{\DD}(\tau_1) \geq \widetilde{\DD}(\tau_2) - \omega, \text{ for all } \tau_1, \tau_2 \in \mathcal{G}, \tau_1 \leq \tau_2 \Bigg\}.
    \end{align*}
    By \Cref{l_DD_non_increas_diff}, we have that $\mathbb{P}(\mathcal{E}) \geq 1-\eta$. The rest of the proof is constructed conditioning on $\mathcal{E}$ happening.
    
    Consider another sequence $\{h_i\}_{i\in[2^M+1]}$, where $h_i = \sup_{j \geq i}\widetilde{\DD}(\tau_j) - \omega$. By construction, we have that $\{h_i\}$ is non-increasing. Moreover, since for all $ j\geq i$, we have that $\widetilde{\DD}(\tau_i) \geq \widetilde{\DD}(\tau_j) - \omega$. It holds that 
    \begin{align} \label{l_fdp_non_increas_exist_eq1}
        \widetilde{\DD}(\tau_i) \geq \sup_{j \geq i} \; \widetilde{\DD}(\tau_j) - \omega = h_i \geq \widetilde{\DD}(\tau_i) - \omega.
    \end{align}
    
    With the newly defined sequence $\{h_i\}$, we prove the lemma by induction. 

    For the base case when $i =1$, we have that $f_1 = \widetilde{\DD}(\tau_1) + \omega \geq \widetilde{\DD}(\tau_1) - \omega$. Also, $f_1 \geq \widetilde{\DD}(\tau_1) \geq h_1$ by construction.

    Next, suppose that $\min\{f_{l-1}, \widetilde{\DD}(\tau_l) +\omega\} \geq \widetilde{\DD}(\tau_l)-\omega$ and $f_l \geq h_l$ for $ l \in [2^M+1] \backslash \{1\}$. We want to show $\min\{f_{l}, \widetilde{\DD}(\tau_{l+1}) +\omega\} \geq \widetilde{\DD}(\tau_{l+1})-\omega$ and $f_{l+1} \geq h_{l+1}$.

    By construction, we have that $h_{l+1} \leq \widetilde{\DD}(\tau_{\ell+1}) \leq \widetilde{\DD}(\tau_{\ell+1}) + \omega$. Also, since $\{h_l\}$ is non-increasing, we have that $h_{\ell+1} \leq h_{\ell} \leq f_{\ell}$. Combine the two bounds together, we have that $h_{\ell+1} \leq \min\{f_{\ell}, \widetilde{\DD}(\tau_{\ell+1}) + \omega\} = f_{\ell+1}$. In addition, since $h_{\ell+1}\geq \widetilde{\DD}(\tau_{\ell+1}) - \omega$ by \eqref{l_fdp_non_increas_exist_eq1}. We have shown that $\widetilde{\DD}(\tau_{\ell+1}) - \omega \leq h_{\ell+1} \leq f_{\ell+1}$. Thus, the induction step holds and the condition $f_i < \widetilde{\DD}(\tau_i) - \omega$ is never satisfied under $\mathcal{E}$.

\end{proof}

\begin{lemma} \label{l_DD_non_increas_diff}
    Suppose the same assumptions as in \Cref{t_fdp_risk} hold. For any $\eta \in (0,1/2)$, it holds that
    \begin{align*}
        \mathbb{P}\Bigg[\widetilde{\DD}(\tau_1) \geq \widetilde{\DD}(\tau_2) - C_1\sqrt{\sum_{s=1}^S\frac{\mu_s^2M^4\log(1/\delta_s)\log(M/\eta)}{\breve{n}_s^2\epsilon_s^2}}, \text{ for all } \tau_1, \tau_2 \in \mathcal{G}, \tau_1 \leq \tau_2 \Bigg] \geq 1-\eta.
    \end{align*}
\end{lemma}

\begin{proof}
    Note we can rewrite
    \begin{align*}
        \widetilde{\DD}(\tau) = \;&\widehat{\DD}(\tau
        ) + \varepsilon_0(\tau),
    \end{align*}
    where for any $\tau = \tau_j \in \mathcal{G}$,
    \begin{align*}
       &\varepsilon_0(\tau) \\
       = \;& \sum_{s=1}^S \mu_s\Bigg[\Big(\frac{\breve{n}_{s,1}}{N_{s,1}}-1\Big)\frac{1}{\breve{n}_{s,1}}\sum_{y\in\{0,1\}}\sum_{i=1}^{\breve{n}_{s,1,y}}\indc\Big\{\widetilde{\eta}_1(\breve{X}^{s,i}_{1,y}) \geq \frac{1}{2}+\frac{\tau}{2\widetilde{\pi}_1}\Big\}\\
        & \hspace{1cm}- \Big(\frac{\breve{n}_{s,0}}{N_{s,0}}-1\Big)\frac{1}{\breve{n}_{s,0}}\sum_{y\in\{0,1\}} \sum_{i=1}^{\breve{n}_{s,0,y}}\indc\Big\{\widetilde{\eta}_0(\breve{X}^{s,i}_{0,y}) \geq \frac{1}{2}-\frac{\tau}{2\widetilde{\pi}_0}\Big\}\\
        &\hspace{1cm}+\frac{1}{N_{s,1}}\sum_{\ell=1}^M \sum_{k=1}^{2^\ell}  \breve{w}_{s,1,\ell,k}\indc\Big\{(k-1)2^{M-\ell}+1 \geq j \text{ and } (\lceil k/2\rceil-1)2^{M-\ell+1}+1 <j \Big\}\\
        & \hspace{1cm}-  \frac{1}{N_{s,0}}\sum_{\ell=1}^M \sum_{k=1}^{2^\ell}  \breve{w}_{s,0,\ell,k}\indc\Big\{(k-1)2^{M-\ell}+1 \geq j \text{ and } (\lceil k/2\rceil-1)2^{M-\ell+1}+1 <j \Big\}\Bigg],
    \end{align*}
    and $\widehat{\DD}(\cdot)$ is defined in \Cref{coro_fdp_DD_non_increase}. By a similar argument to control $(I)_2$ and $(I)_3$ in \textbf{Step~1-3} in the proof of \Cref{l_fdp_DD_est_err}, we have with probability at least $1-\eta$ that 
        \begin{align*}
            \sup_{\tau \in \mathcal{G}}|\varepsilon_0(\tau)| \leq C_1\sqrt{\sum_{s=1}^S\frac{\mu_s^2M^4\log(1/\delta_s)\log(M/\eta)}{\breve{n}_s^2\epsilon_s^2}}.
        \end{align*}
        Consequently, the lemma follows by the fact that $\widehat{\DD}$ is non-increasing as shown in \Cref{coro_fdp_DD_non_increase} and a union bound argument over all possible choices of $\tau_1, \tau_2 \in \mathcal{G}$.
\end{proof}

\begin{corollary} \label{coro_fdp_DD_non_increase}
    For any $\tau \in \mathcal{G}$, denote $\widehat{\DD}(\tau) = \sum_{s=1}^S \nu_s\widehat{\DD}^s(\tau)$, where 
    \[\widehat{\DD}^s(\tau) = \frac{1}{N_{s,1}}\sum_{y\in\{0,1\}}\sum_{i=1}^{\breve{n}_{s,1,y}}\indc\Big\{\widetilde{\eta}_1(\breve{X}^{s,i}_{1,y}) \geq \frac{1}{2}+\frac{\tau}{2\widetilde{\pi}_1}\Big\} -\frac{1}{N_{s,0}}\sum_{y\in\{0,1\}}\sum_{i=1}^{\breve{n}_{s,0,y}}\indc\Big\{\widetilde{\eta}_0(\breve{X}^{s,i}_{0,y}) \geq \frac{1}{2}-\frac{\tau}{2\widetilde{\pi}_0}\Big\}.\]
    Then we have that $\widehat{\DD}(\cdot)$ is a non-increasing function. 
\end{corollary}

\begin{proof}
    \Cref{coro_fdp_DD_non_increase} follows from a similar argument as \Cref{l_DD_non_increase}. We omit the proof here.
\end{proof}

\begin{lemma} \label{l_fdp_N_n_same_order}
    If we additionally assume that $\min_{s\in[S]} \breve{n}_{s,a}^2 \epsilon_s^2 \geq C_1M \log(1/\delta_s)\log(S/\eta)$, then
    \begin{align*}
        \mathbb{P}\Big\{C_1\breve{n}_{s,a} \leq N_{s,a} \leq C_2\breve{n}_{s,a}, \text{ for all } s\in [S], a\in \{0,1\}\Big\} \geq 1-\eta.
    \end{align*}
\end{lemma}
\begin{proof}
    By construction in \Cref{alg_threshold_search_fdp}, since for any $s,i,a,y$, we have $\widetilde{\pi}_a \in [0,1], \;\; \widetilde{\eta} \in [0,1]$, hence $Z_{s,i,a,y} \in [-1,1]$, and we have that for any $a \in \{0,1\}$,
   \begin{align*}
       N_{s,a} = \breve{n}_{s,a} + \breve{w}_{s,a,1,1}+ \breve{w}_{s,a,1,2}.
   \end{align*}
    By standard Gaussian tail properties \citep[e.g.~Proposition 2.1.2 in][]{vershynin2018high} and a union bound argument, we have that  with probability at least $1-\eta$ that for all $s \in [S]$ and $a \in \{0,1\}$,
    \begin{align*}
        | N_{s,a}-\breve{n}_{s,a}| \leq C_1 \sqrt{\frac{M\log(1/\delta_s)\log(S/\eta)}{\epsilon_s^2}}.
    \end{align*}
    Hence, by using a similar argument as the one used in the proof of \Cref{l_n_a_n_same_order}, the lemma holds whenever $\min_{s\in[S]} \breve{n}_{s,a}^2 \epsilon_s^2 \geq C_2M \log(1/\delta_s)\log(S/\eta)$.
\end{proof}

\begin{lemma} \label{l_DD_shift_upper_bound}
    Under \Cref{a_posterior}\ref{a_posterior_margin}, we have for any small $\rho$ in the neighbourhood of $0$, 
    \[|\DD(\tau^*+\rho) - \DD(\tau^*)| \leq (2C_m)/(2C_\pi)|\rho|^\gamma.\]
\end{lemma}

\begin{proof}
    By definition, we have that for any $\rho >0$, we have
    \begin{align*}
        0< \;&\DD(\tau^*) - \DD(\tau^*+\rho) \\
        = \;& \mathbb{P}\Big\{\frac{1}{2}+\frac{\tau^*}{2\pi_1} \leq \eta_1(X) < \frac{1}{2}+\frac{\tau^*+\rho}{2\pi_1}\Big\} +  \mathbb{P}\Big\{\frac{1}{2}-\frac{\tau^*+\rho}{2\pi_0} \leq \eta_0(X) < \frac{1}{2}-\frac{\tau^*}{2\pi_0}\Big\}\\
        = \;& \mathbb{P}\Big\{0 \leq \eta_1(X)-T_1^* < \frac{\rho}{2\pi_1}\Big\} + \mathbb{P}\Big\{-\frac{\rho}{2\pi_0}\leq \eta_0(X) - T_0^* < 0\Big\}\\
        \leq \;& \mathbb{P}\Big\{|\eta_1(X)-T_1^*| < \frac{\rho}{2\pi_1}\Big\} + \mathbb{P}\Big\{|\eta_0(X) - T_0^* |< \frac{\rho}{2\pi_0}\Big\} \leq \frac{2C_m}{(2C_\pi)^{\gamma}}\rho^\gamma,
    \end{align*}
    where the last inequality follows from \Cref{a_posterior}\ref{a_posterior_margin}. Similarly, when $\rho <0$, we have that 
    \begin{align*}
        0< \;& \DD(\tau^*+\rho) - \DD(\tau^*)\\
        =\;& \mathbb{P}\Big\{\frac{\rho}{2\pi_1} \leq \eta_1(X)-T_1^* < 0\Big\} + \mathbb{P}\Big\{0 \leq \eta_0(X) - T_0^* < -\frac{\rho}{2\pi_0}\Big\}\\
        \leq \;& \mathbb{P}\Big\{|\eta_1(X)-T_1^*| < -\frac{\rho}{2\pi_1}\Big\} + \mathbb{P}\Big\{|\eta_0(X) - T_0^*| < -\frac{\rho}{2\pi_0}\Big\} \leq  \frac{2C_m}{(2C_\pi)^{\gamma}}(-\rho)^\gamma.
    \end{align*}
    Therefore, we have that for any small $\rho$ in the neighbourhood of $0$, 
    \[|\DD(\tau^*+\rho) - \DD(\tau^*)| \leq C_2|\rho|^\gamma.\]
\end{proof}

\subsubsection{Upper bound on \texorpdfstring{$\|\widetilde{p}_{X|A=a} -p_{X|A=a} \|_\infty$}{}}

\begin{lemma} \label{l_fdp_p_x_a_est_sup}
    Under the same assumptions as in \Cref{t_fdp_risk}, it holds for any $\eta \in (0,1/2)$ and $a \in \{0,1\}$ that 
    \begin{align*}
        \mathbb{P}&\Bigg\{\|\widetilde{p}_{X|A=a} -p_{X|A=a} \|_\infty \\
        &\hspace{-1em}\geq C_1\Bigg[\sqrt{\sum_{s=1}^S \frac{\nu^2_s\log(h^{-d}/\eta)}{n_{s,a}h^d}} +  \max_{s \in [S]} \frac{\nu_s\log(h^{-d}/\eta)}{n_{s,a}h^{d}} + h^\beta+ \sqrt{\sum_{s=1}^S\frac{\nu^2_s\log(1/\delta_s)\log(h^{-1}/\eta)}{n^2_{s,a}\epsilon^2_s h^{2d}}}\Bigg] \Bigg\} \leq \eta.
    \end{align*}
\end{lemma}

\begin{proof}
    The proof is similar to the argument used in the proof of \Cref{l_p_x_a_est_sup}. By the triangle inequality, we have that 
    \begin{align*}
        &\|\widetilde{p}_{X|A=a} -p_{X|A=a} \|_\infty \\
        \leq \;& \sup_{x\in[0,1]^d} \Big|\sum_{s=1}^S\frac{\nu_s}{n_{s,a}}\sum_{y \in \{0,1\}}\sum_{i=1}^{n_{s,a,y}}K_{h}(X_{a,y}^{s,i} -x) - p_{X|A=a}(x|a)\Big|\\
        & + \sup_{x\in[0,1]^d} \Big|\sum_{s=1}^S \frac{8 \nu_s\sqrt{2C_K\log(8/\delta_s)}}{n_{s,a}\epsilon_s h^d}W^s_{1,a}(x)\Big|\\
        \leq \;& \sup_{x\in[0,1]^d}\Big|\sum_{s=1}^S\frac{\nu_s}{n_{s,a}}\sum_{y \in \{0,1\}}\sum_{i=1}^{n_{s,a,y}}K_{h}(X_{a,y}^{s,i} -x) - \int_x K_h(u-1)p_{X|A=a}(u|a) \; \mathrm{d}u \Big|\\
        & + \sup_{x\in[0,1]^d}\Big|\int_x K_h(u-1)p_{X|A=a}(u|a) \; \mathrm{d}u - p_{X|A=a}(x|a)\Big|\\
        & + \sup_{x\in[0,1]^d} \Big|\sum_{s=1}^S \frac{8 \nu_s\sqrt{2C_K\log(8/\delta_s)}}{n_{s,a}\epsilon_s h^d}W^s_{1,a}(x)\Big|\\
        = \;& (I) + (II) + (III).
    \end{align*}

    To control $(I)$, by \Cref{coro_fdp_est_p_x_a_kernel_sup}, we have that 
    \begin{align*}
        \mathbb{P}\Big\{(I) \geq C_1\sqrt{\sum_{s=1}^S \frac{\nu^2_s\log(h^{-d}/\eta)}{n_{s,a}h^d}} + C_2 \max_{s \in [S]} \frac{\nu_s\log(h^{-d}/\eta)}{n_{s,a}h^{d}}\Big\} \leq \frac{\eta}{2}.
    \end{align*}

    To control $(II)$, by Assumptions \Cref{a_prob}\ref{a_prob_p_x_a} and \ref{a_kernel}\ref{a_kernel_adaptive}, we have that 
    \begin{align*}
        (II) \leq C_{\text{adp}}h^{\beta}.
    \end{align*}

    To control $(III)$, by \Cref{l_fdp_sup_gaussian}, we have that 
    \begin{align*}
        \mathbb{P}\Big\{(III) \geq C_3\sqrt{\sum_{s=1}^S\frac{\nu^2_s\log(1/\delta_s)\log(h^{-1}/\eta)}{n^2_{s,a}\epsilon^2_s h^{2d}}}\Big\} \leq \frac{\eta}{2}.
    \end{align*}

    The lemma thus follows from a union bound argument.

\end{proof}

\begin{corollary}  \label{coro_fdp_p_xa_const_order}
     Under the same assumption of \Cref{t_fdp_risk}, it holds for any $\eta \in (0,1/2)$ and $a \in \{0,1\}$ that 
    \begin{align*}
        \mathbb{P}\Big\{C_1 \leq \inf_{x \in [0,1]^d} \widetilde{p}_{X|A=a}(x) \leq \sup_{x \in [0,1]^d} \widetilde{p}_{X|A=a}(x) \leq C_2 \Big\} \geq 1- \eta.
    \end{align*}
\end{corollary}

\begin{proof}
    The proof follows from \Cref{l_fdp_p_x_a_est_sup} and a similar argument used in the proof of \Cref{l_p_xa_const_order}. We omit the proof here.
\end{proof}

\begin{corollary} \label{coro_fdp_est_p_x_a_kernel_sup}
     Under the same assumption of \Cref{t_fdp_risk}, it holds for any $\eta \in (0,1/2)$ and $a \in \{0,1\}$ that 
    \begin{align*}
        \mathbb{P}\Big\{\sup_{x\in[0,1]^d}\Big|\sum_{s=1}^S&\frac{\nu_s}{n_{s,a}}\sum_{y \in \{0,1\}}\sum_{i=1}^{n_{s,a,y}}K_{h}(X_{a,y}^{s,i} -x) - \int_x K_h(u-1)p_{X|A=a}(u|a) \; \mathrm{d}u \Big| \\
        & \hspace{3cm} \geq C_1\sqrt{\sum_{s=1}^S \frac{\nu^2_s\log(h^{-d}/\eta)}{n_{s,a}h^d}} + C_2 \max_{s \in [S]} \frac{\nu_s\log(h^{-d}/\eta)}{n_{s,a}h^{d}}\Big\} \leq \eta.
    \end{align*}
\end{corollary}
\begin{proof}
    The proof follows from similar arguments to the proofs of Lemmas  \ref{l_est_p_xa_kernel_sup} and \ref{l_fdp_est_p_y_xa_kenel_sup}. We omit it here.
\end{proof}

\subsubsection{Upper bound on \texorpdfstring{$|\widetilde{\pi}_a - \pi_a|$}{}}

\begin{lemma} \label{l_fdp_pi_a_est}
     Under the same assumptions as in \Cref{t_fdp_risk}, it holds for any $\eta \in (0,1/2)$ and $a \in \{0,1\}$ that 
    \begin{align*}
        \mathbb{P}\Bigg[|\widetilde{\pi}_a - \pi_a| \geq C_1\Bigg\{\sqrt{\sum_{s=1}^S\frac{\nu_s^2\log(1/\eta)}{n_s}}+\sqrt{\sum_{s=1}^S\frac{\nu_s^2\log(1/\delta_s)\log(1/\eta)}{n_s^2\epsilon_s^2}}\Bigg\}\Bigg] \leq \eta.
    \end{align*}
\end{lemma}

\begin{proof}
    Consider the sequence of bounded random variables $\indc\{A^s_{i} =a\}_{i=1}^{N_a}$,  and we have that 
    \begin{align*}
        \widetilde{\pi}_a = \sum_{s=1}^S \nu_s \Big(\frac{n_{s,a,0}+n_{s,a,1}}{n_s}+ w^s_a\Big) = \sum_{s=1}^S\sum_{i=1}^{n_s}\frac{\nu_s\indc\{A^s_{i} =a\}}{n_s} + \sum_{s=1}^S \nu_s w^s_a.
    \end{align*}
    By general Hoeffding’s inequality \citep[e.g.~Theorem 2.6.2 in][]{vershynin2018high}, we have that for any $t_1 >0$,
    \begin{align*}
        \mathbb{P}\Big(\Big|\sum_{s=1}^S\sum_{i=1}^{n_s}\frac{\nu_s\indc\{A^s_{i} =a\}}{n_s} - \pi_a \Big| \geq t_1\Big) \leq C_1\exp\Bigg(-\frac{t_1^2}{\sum_{s=1}^S\frac{\nu_s^2}{n_s}}\Bigg).
    \end{align*}
    Moreover, by standard properties of Gaussian random variables, we have that 
    \begin{align*}
        \sum_{s=1}^S \nu_s w^s_a \sim N\Big(0, \sum_{s=1}^S \nu_s^2\sigma_s^2 \Big),
    \end{align*}
    where $\sigma= 4\sqrt{2\log(5/\delta_s)}/(n_s\epsilon_s)$. Thus standard Gaussian tail properties \citep[e.g.~Proposition 2.1.2 in][]{vershynin2018high} gives for any $t_2 >0$,
   \begin{align*}
       \mathbb{P}\Big(\Big|\sum_{s=1}^S \nu_s w^s_a\Big| \geq t_2\Big)\leq C_2 \exp\Bigg(-\frac{t_2^2}{\sum_{s=1}^S\frac{\nu_s^2\log(1/\delta_s)}{n_s^2\epsilon_s^2} }\Bigg).
   \end{align*}
    Taking 
   \[t_1 = C_3\sqrt{\sum_{s=1}^S\frac{\nu_s^2\log(1/\eta)}{n_s}}, \;\; \text{and}\;\; t_2 = C_4\sqrt{\sum_{s=1}^S\frac{\nu_s^2\log(1/\delta_s)\log(1/\eta)}{n_s^2\epsilon_s^2}},\]
   the lemma thus follows by a union bound argument.
\end{proof}

\begin{corollary} \label{coro_fdp_n_a_n_same_order}
    For any $\min\{n_s, \breve{n}_s\} \geq 4C_1^2\log(S/\eta)/C_\pi^2$, it holds that 
    \[\mathbb{P}\Big\{C_2 n_s \leq n_{s,a} \leq n_{s}, \;\; \text{and} \;\; C_3 \breve{n}_s \leq \breve{n}_{s,a} \leq \breve{n}_s, \; a\in \{0,1\}, s \in [S] \Big\} \geq 1-\eta.\]
\end{corollary}
\begin{proof}
    The proof follows from \Cref{l_fdp_pi_a_est} and a similar argument used in the proof of \Cref{l_n_a_n_same_order}. We omit it here.
\end{proof}

\section{Additional background} \label{sec_appendix_additional_background}
We collect some background results that are used throughout the paper.
\subsection{Additional background related to DP}
\begin{lemma}[Gaussian mechanism for univariate output, Theorem 3.22 in \citealp{dwork2014algorithmic}] \label{l_GM_univariate}
    Let $f$ be a function $f: \mathcal{D} \rightarrow \mathbb{R}$ such that $\Delta(f) = \sup_{D\sim D'} |f(D)-f(D')|$ is finite. Then for any $\epsilon, \delta >0$, the mechanism $M(D) = f(D) + \sqrt{2\log(1.25/\delta)}\Delta(f)Z/\epsilon $ is $(\epsilon,\delta)$-CDP, where $Z \sim N(0,1)$.
\end{lemma}

\begin{lemma}[Gaussian mechanism for functional output, Corollary 9 in \citealp{hall2013differential}] \label{l_GM_func}
    Let $G$ be a Gaussian process with mean zero and covariance function $K$. We further denote $\mathcal{K}$ the RKHS space associated with kernel $K$ with RKHS norm $\|\cdot\|_{\mathcal{K}}$. Then for any function $f: \mathcal{D} \rightarrow \mathcal{K}$, the point-wise release of 
    \begin{align*}
       f_D(\cdot) + \frac{\sqrt{2\log(2/\delta)}\Delta_{\mathcal{K}}(f)}{\epsilon}G(\cdot)
    \end{align*}
    is $(\epsilon,\delta)$-CDP, where $\Delta_{\mathcal{K}}(f) = \sup_{D \sim D'} \|f_D - f_{D'}\|_{\mathcal{K}}$ and $f_D(\cdot)$ is a shorthand notion for $f(D)(\cdot)$.
\end{lemma}

\begin{lemma}[Composition of DP, e.g.~Theorem 3.16 in \citet{dwork2014algorithmic}] \label{l_dp_composition}
    For $\epsilon_1, \epsilon_2 >0$ and $\delta_1, \delta_2 \geq 0$. Suppose that algorithms $M_1$ and $M_2$ are $(\epsilon_1, \delta_1)$-DP and $(\epsilon_2, \delta_2)$-DP respectively. It then holds that the composition $M_1\circ M_2$ is $(\epsilon_1+ \epsilon_2, \delta_1 + \delta_2)$-DP.
\end{lemma}

\subsection{Additional background related to Fairness}
\begin{definition}[Definitions 3.2 and 3.3 in \citealp{zeng2024bayes}] \label{def_bilinear_disparity}
We call a disparity measure $\D : \mathcal{F} \to [0,1]$ linear if for all $\mathbb{P}$, there is a weighting function 
    $w_{\D,\mathbb{P}} : [0,1]^d \times \{0,1\} \to \mathbb{R}$ such that for all $f \in \mathcal{F}$,
\[
\D(f) = \sum_{a\in\{0,1\}} \int_{[0,1]^d} f(x,a)\, w_{\D}(x,a) \, \mathrm{d}\mathbb{P}_{X,A}(x,a).
\]
We further call the disparity measure D bilinear if for all $\mathbb{P}$, there is $s_{\D,a}$ and $b_{\D,a}$ depending on $a \in \{0,1\}$ such that for all $x \in [0,1]^d$, $w_{\D}(x,a) = s_{\D,a} \cdot \eta_a(x) + b_{\D,a}$.
\end{definition}

\begin{proposition}[Proposition 3.4 in \citealp{zeng2024bayes}]
The disparity measure DD in \Cref{def_disparity_measure} is bilinear with weighting functions defined for all $x,a$ by $w_{\DD}(x,a) = \frac{(2a-1)}{\pi_a}$. Thus we have $s_{\DD,a} =0$ and $b_{\DD,a} = (2a-1)/\pi_a$.
\end{proposition}

\begin{theorem}[Theorem 4.2 in \citealp{zeng2024bayes}] \label{thm_fair_bayes_opt}
Assume that $\eta_a(X)$ is a continuous random variable for $a \in \{0, 1\}$. For any $\tau \in \mathbb{R}$ and a given bilinear disparity measure $\D$ in \Cref{def_bilinear_disparity}, denote the classifier
\begin{align*}
    g_{\D, \tau}(x,a) = \begin{cases} \vspace{0.5em}
            1, &(2-\tau  s_{\D,a})\eta_a(x) \geq 1+\tau  b_{\D,a},\\ 
            0, &(2-\tau  s_{\D,a})\eta_a(x) < 1+\tau b_{\D,a},
        \end{cases}
    \end{align*}
    and let $\mathrm{D}(\tau) = \mathrm{D}(g_{\mathrm{D}, \tau})$. Then, for any $\alpha \geq 0$, the $\alpha$-fair Bayes optimal classifier is $f_{\mathrm{D}, \alpha}^* = g_{\mathrm{D}, \tau_{\mathrm{D}, \alpha}^*}$, where 
    \begin{align*} 
        \tau_{\mathrm{D}, \alpha}^* = \argmin_{\tau \in \mathbb{R}} \big\{ |\tau|: \, |\mathrm{D}(\tau)| \leq \alpha\big\}.
    \end{align*}
\end{theorem}

\begin{proposition}[Proposition 4.1 in \citealp{zeng2024bayes}]\label{prop_D_and_R}
	Recall that $D(\tau) = D(g_{D, \tau})$, where $g_{D, \tau}$ is defined in \Cref{thm_fair_bayes_opt}. Then, under the assumptions in \Cref{thm_fair_bayes_opt}, the following properties hold.
	\begin{enumerate}
		\item[(i)] The disparity $D(\tau)$ is continuous and non-increasing.
		\item[(ii)]  The misclassification $R(g_{D, \tau})$ is non-increasing with respect to $\tau$ on $(-\infty, 0)$ and non-decreasing with respect to $\tau$ on $(0, +\infty)$.
	\end{enumerate}
\end{proposition}

\end{document}